\documentclass[letterpaper, 10 pt, conference]{ieeetran}
\pdfoutput=1

\usepackage{xspace}
\usepackage[T1]{fontenc}

\usepackage{multicol}
\usepackage{hyperref}
\hypersetup{bookmarksopen,bookmarksnumbered,
pdfpagemode=UseOutlines,
colorlinks=true,
linkcolor=blue,
anchorcolor=blue,
citecolor=blue,
filecolor=blue,
menucolor=blue,
urlcolor=blue
}

\usepackage{amsmath, amssymb, amsfonts}
\usepackage{ntheorem}
\usepackage{bm}
\usepackage{flushend}
\usepackage{graphicx}
\usepackage{mathtools}
\usepackage[numbers]{natbib}
\usepackage[font=footnotesize]{caption} 
\usepackage[font=footnotesize]{subcaption}
\usepackage{mathabx}
\usepackage{verbatim}
\usepackage{algorithm}
\usepackage[noend]{algpseudocode}

\usepackage{multirow}
\usepackage{booktabs}
\usepackage[dvipsnames]{xcolor}

\usepackage{tikz}

\newcommand{\xxnote}[3]{}
\ifx\hidenotes\undefined
 \usepackage{color}
 \renewcommand{\xxnote}[3]{\color{#2}{#1: #3}}
\fi

\newcommand{\eref}[1]{Equation~\ref{#1}} 
\newcommand{\sref}[1]{Section~\ref{#1}} 
\newcommand{\figref}[1]{Figure~\ref{#1}} 

\newcommand{\calG}{\mathcal{G}}
\newcommand{\calR}{\mathcal{R}}
\newcommand{\calM}{\mathcal{M}}
\newcommand{\calX}{\mathcal{X}}
\newcommand{\calQ}{\mathcal{Q}}
\newcommand{\calS}{\mathcal{S}}

\newcommand{\R}{\mathbb{R}}

\newcommand{\N}{\mathbb{N}}

\newcommand{\Cfree}{\calX_{\rm free}}
\newcommand{\Cobs}{\calX_{\rm obs}}
\newcommand{\Cs}{C-space }

\newcommand{\cbest}[1]{c_{\mathrm{best}}^{#1}}
\newcommand{\cmin}[0]{c_{\mathrm{min}}}
\newcommand{\eps}[1]{\varepsilon_{#1}}
\newcommand{\imax}[0]{i_{\mathrm{max}}}
\newcommand{\vol}[1]{\mu\left(#1\right)}
\newcommand{\const}[0]{\Gamma}
\newcommand{\spacebest}[0]{\mathcal{X}_{\cbest{i}}}

\newtheorem{theorem}{Theorem}

\newtheorem{proposition}[theorem]{Proposition}

\renewcommand{\algorithmicrequire}{\textbf{Input:}}

\definecolor{myblue}{RGB}{158,202,225}
\definecolor{myred}{RGB}{252,146,114}
\definecolor{mygreen}{RGB}{161,217,155}
\definecolor{mypurple}{RGB}{190,174,212}

\title{Anytime Motion Planning on Large Dense Roadmaps with Expensive Edge Evaluations}

\author{\IEEEauthorblockN{Shushman Choudhury}
\IEEEauthorblockA{Department of Computer Science\\
Stanford University\\
shushman@stanford.edu}
\and
\IEEEauthorblockN{Oren Salzman \\ Sanjiban Choudhury\\ Christopher M. Dellin}
\IEEEauthorblockA{Robotics Institute\\
Carnegie Mellon University}
\and
\IEEEauthorblockN{Siddhartha S. Srinivasa}
\IEEEauthorblockA{Paul G. Allen School of \\ Computer Science and Engineering\\
University of Washington}}

\begin{document}

\maketitle
\thispagestyle{empty}
\pagestyle{empty}

\begin{abstract}
We propose an algorithmic framework for efficient anytime motion
planning on large dense geometric roadmaps, in domains where collision
checks and therefore edge evaluations are computationally expensive.
A large dense roadmap (graph) can typically ensure the existence of 
high quality solutions for most motion-planning problems, but the size of the 
roadmap, particularly in high-dimensional spaces, makes existing
search-based planning algorithms computationally expensive.
We deal with the challenges of expensive search and collision checking
in two ways.
First, we frame the problem of anytime motion planning on
roadmaps as searching for the shortest path over a sequence of 
subgraphs of the entire roadmap graph, generated by some densification
strategy. This lets us achieve bounded sub-optimality 
with bounded worst-case planning effort. Second, for searching each subgraph, 
we develop an anytime planning algorithm which uses a belief model to compute
the collision probability of unknown configurations and searches for paths that are 
Pareto-optimal in path length and collision probability. 
This algorithm is efficient with respect to collision checks as it searches
for successively shorter paths. We theoretically
analyze both our ideas and evaluate them individually on high-dimensional motion-planning 
problems. Finally, we apply both of these ideas together in our algorithmic 
framework for anytime motion planning, and show that it outperforms $\text{BIT}^{*}$ on high-dimensional hypercube problems.
\end{abstract}

\section{Introduction}
\label{sec:intro}

We consider the problem of finding successively shorter feasible paths (solutions) on a motion-planning roadmap
$\calG$, with $N$ vertices geometrically embedded in some configuration space (\Cs).
Specifically, we are interested in problem settings with two properties - 
the roadmap is large ($N >> 1$) and dense ($|E| = \Theta(N^2)$), and evaluating
if an edge of the graph is collision-free or not is computationally expensive.

\subsection{Motivation}
\label{sec:intro-motiv}

Our problem is motivated by previous work on sampling-based motion-planning algorithms
that pre-construct a \emph{fixed} roadmap~\cite{kavraki1996probabilistic,bohlin2000path}
to efficiently approximate the structure of the \Cs.
We want a geometric motion planning algorithm to obtain \emph{high-quality paths 
over a wide range of motion planning problems}. To find any feasible path 
over a diverse set of problems, we would want \emph{large} roadmaps, with
enough vertices (samples) to cover the configuration space sufficiently. 
To obtain the best quality paths possible with those vertices, we would want 
\emph{dense} roadmaps, with an edge between almost every pair of vertices. Note that our roadmap formulation departs from the
$\text{PRM}^{*}$ algorithm which chooses a radius of connectivity of $O(\text{log } n)$
to achieve asymptotic optimality~\citep{karaman2010incremental}. We consider
longer edges that capture as much \Cs connectivity as possible. Our experiments
at the end of \sref{sec:experiments} will support this formulation.

Collision checks are a significant computational bottleneck for sampling-based 
motion planning~\citep{sanchez2002delaying}. This is particularly true for manipulation
planning with articulated robots where collision checks involve geometric intersection tests 
on large, complex meshes. In roadmaps, this bottleneck is manifested in the evaluation
of edges, which have multiple embedded configurations to test for collision.

For large dense roadmaps with expensive edge evaluations, any shortest-path search algorithm 
must perform $|E| = \Theta(N^2)$ edge operations
to obtain its first (and only) solution. This is impractical for real-time applications.
Therefore, we consider the \emph{anytime motion planning} problem, where we find some initial
feasible solution and refine it as time permits, using the quality of the current path
(based on our objective function) as a bound for future paths. This 
accommodates a time budget for planning where the planning module returns the best solution
found when the available time elapses.

\begin{figure*}
	\centering
	\captionsetup[subfigure]{justification=centering}
	\begin{subfigure}[b]{0.99\columnwidth}
		\centering
		\includegraphics[width=0.45\textwidth]{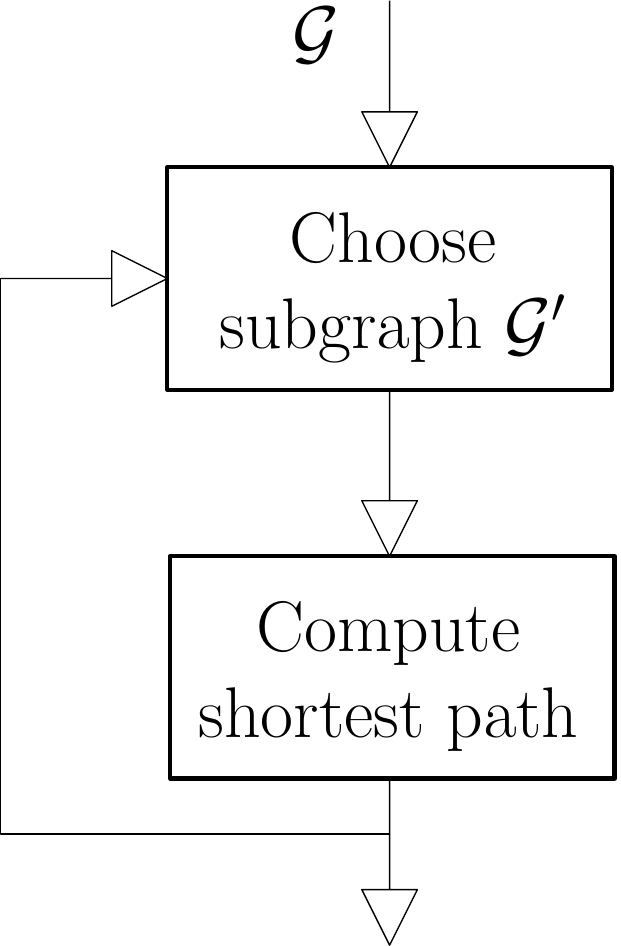}
		\caption{}
		\label{fig:intro-flow}
	\end{subfigure}
	\begin{subfigure}[b]{0.99\columnwidth}
		\centering
		\includegraphics[width=\textwidth]{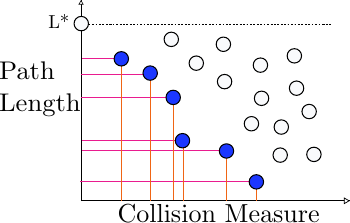}
		\caption{}
		\label{fig:intro-belief}
	\end{subfigure}
	\caption{
		The two key ideas to our approach : (\subref{fig:intro-flow}) Densification generates a sequence of increasingly dense subgraphs of
		the entire roadmap graph, which are searched for their shortest paths
		(\subref{fig:intro-belief}) We reason about the tradeoff between path length and collision measure (related to collision probability)
		for candidate paths. We search for paths that are Pareto-optimal in
		the two quantities and select them for lazy evaluation.
	}
	\label{fig:intro-ideas}
	
\end{figure*}

\subsection{Key Ideas}
\label{sec:intro-ideas}

Our approach is based on two key ideas, depicted in \figref{fig:intro-ideas}. Each idea addresses one of the characteristics
of the problem setting we are interested in. 

First, we approach the problem of anytime motion planning on roadmaps by providing
graph-search-based planning algorithms with a sequence of subgraphs 
of the roadmap, generated using some \emph{densification strategy}~\citep{choudhury2017densification}.
At each iteration, we run a shortest-path algorithm on the current subgraph to obtain an increasingly 
tighter approximation of the true shortest path.

Existing approaches to anytime planning on graphs~\citep{likhachev2004ara,likhachev2005anytime,BSHG11} 
modify the objective function, thereby sacrificing 
optimality for quicker solutions. Some of these approaches guarantee bounded sub-optimality.
Searching with a modified objective function, however, has worst-case complexity
$\Theta(N^2)$, which is unacceptable for large dense roadmaps in real-time applications.
Also, there is no formal guarantee that these approaches will decrease search time
and they may still search all edges of a given graph~\citep{WR12}. Therefore, we explicitly 
solve sub-problems of the entire problem, and gradually increase the size of each sub-problem 
until we solve the original problem. On a roadmap, this naturally translates to searching over
a sequence of increasing subgraphs of the entire roadmap.

Second, we try to
minimize the number of collision checks while searching for the shortest feasible path in any
roadmap. Performing a collision check provides exact information but is computationally expensive.
Therefore, we use a model of the \Cs to estimate the probability of unevaluated configurations to 
be free or in collision.  We ensure that updating and querying the model is inexpensive relative to
a collision check.

Previous works have approached this by using the probability of collision as a heuristic~\citep{nielsen2000two},
optimistically searching for the shortest path with lazy collision checking~\citep{bohlin2000path},
and using collision probabilities learned from previous instances to filter out unlikely configurations
~\citep{pan2013faster}. They lack, however, a way to connect the two problems of finding a feasible 
path quickly and finding the shortest feasible path in the roadmap. We do this by considering both
path length and collision probability in our search objective functions, gradually
prioritizing path length to search for successively shorter paths.

\subsection{Contributions}
\label{sec:intro-contrib}

The following are our major contributions in this paper:

\begin{itemize}
	\item We motivate the use of a large, dense roadmap for motion planning. In particular 
	we propose using the same roadmap constructed for a particular robot's configuration 
	space over a wide range of problems that the robot will be required to solve. This lets 
	us take advantage of the structure embedded in the roadmap, as well as various 
	preprocessing benefits of using the same roadmap repeatedly (\sref{sec:problem}).

	\item We present several densification strategies to generate the sequence of subgraphs
	of the entire roadmap. For the specific case where the samples are generated from 
	a low-dispersion deterministic Halton sequence (\sref{sec:background-dispersion}), we analyse the tradeoff 
	between effort and bounded sub-optimality (\sref{sec:densification}).

	\item We present an anytime planning algorithm (for a reasonably sized roadmap) called Pareto
	Optimal Motion Planner (POMP). It uses a model to maintain a belief over configuration space 
	collision probabilities and efficiently searches for paths 
	that are Pareto-optimal in collision probability and path length, eventually computing the shortest
	feasible path on the roadmap. (\sref{sec:cspacebelief}).

\end{itemize}

\noindent
Our key ideas are useful and well-motivated in their own right, and we extensively evaluate them individually, 
thereby demonstrating favourable
performance against contemporary motion planning algorithms (\sref{sec:experiments}).
Densification, irrespective of the  underlying shortest-path search algorithm, is an efficient 
way to organize the search over a large dense roadmap. POMP is an efficient algorithm in domains
with expensive collision-checks, irrespective of whether it is used stand-alone or in a larger
anytime planning context.

This paper is an extended version of previous works that discussed densification~\citep{choudhury2017densification}
and POMP~\citep{choudhury2016pareto}.
We propose using POMP as
the underlying search algorithm in our roadmap densification framework.
POMP is particularly well-suited 
for this: it is efficient with respect to collision checks, it computes
the shortest feasible path on the roadmap it searches, and it builds up an increasingly
accurate model of configuration space with each batch, which aids searches
in future batches. Densification imposes anytime behaviour and 
POMP is an anytime algorithm, therefore our overall framework
has two-level anytime behaviour. We implement this framework efficiently and outperform the $\text{BIT}^{*}$
algorithm~\citep{gammell2014batch} on a range of high-dimensional planning problems. Finally, we conclude in
\sref{sec:conclusion} by discussing the limitations of our framework and 
questions for future research.


\section{Background}
\label{sec:background}

Our algorithms and analyses incorporate ideas from several topics. We are planning with 
sampling-based roadmaps and we our analysis considers the case when samples are generated from a low-dispersion deterministic
sequence. We are interested in efficiently searching these roadmaps and in their finite-time
properties. We also use a configuration space belief model to estimate collision probabilities.
In this section, we briefly outline related work in each of these areas.

\subsection{Sampling-based motion planning}
\label{sec:background-sampling}

Sampling-based planning approaches build a roadmap (graph) in the configuration space, where vertices are configurations and edges are local paths connecting configurations. 
A path is found by searching this roadmap while checking if the vertices and edges are collision free.
Initial algorithms such as  PRM~\citep{kavraki1996probabilistic} and RRT~\citep{LK99} were concerned with finding \emph{a feasible} solution.
Recently, there has been growing interest in finding high-quality solutions.
Variants of the PRM and RRT algorithms, called $\text{PRM}^{*}$ and $\text{RRT}^{*}$~\citep{karaman2010incremental}
were proved to converge to the optimal solution asymptotically.

However, the running times of these algorithms are often significantly higher than their non-optimal counterparts.
Thus, subsequent algorithms have been suggested to increase the rate of convergence to high-quality solutions.
They use different approaches such as 
lazy computation~\citep{bohlin2000path,janson2015fast,dellin2016unifying},
informed sampling~\citep{GSB14},
pruning vertices~\citep{gammell2014batch},
relaxing optimality~\citep{SH16}, exploiting local information~\citep{choudhury2016regionally}
and lifelong planning with heuristics~\citep{koenig2004lifelong}.

\subsection{Dispersion}
\label{sec:background-dispersion}

The \emph{dispersion} of $\calS$, a sequence of $n$ points, is defined as 
\begin{equation}
\label{eq:dispersion-def}
D_n(\calS) = \sup_{x \in \calX} \min_{s \in \calS} \rho(x, s).
\end{equation}
It can be thought of as the radius of the largest empty ball (by some distance metric $\rho$)
that can be drawn around any point in the space $\calX$ without intersecting any point
in $\calS$. A lower dispersion implies a better \emph{coverage} of the space by the points in $\calS$.
When $\calX$ is the $d$-dimensional 
Euclidean space and $\rho$ is the Euclidean distance, deterministic sequences  with dispersion of order $O(n^{-1/d})$ exist.
A simple example is a set of points lying on grid or a lattice.

Other low-dispersion deterministic sequences exist which also  have low \emph{discrepancy}, i.e. they appear to be random.
Specifically, the discrepancy of a set of points is the deviation of the set from the uniform random distribution.
The corresponding mathematical definition of discrepancy is
\begin{equation}
\label{eq:discrepancy-def}
\mathbb{D}_n(\calS,\calR) = \sup_{R \in \calR} \{ \left| \frac{\left|\calS \cap R \right|}{n} - \frac{\mu(R)}{\mu(\calX)} \right| \},
\end{equation}
where $\calR$ is a collection of subsets of $\calX$ called the \emph{range space}. It is typically taken to be the 
set of all axis-aligned rectangular subsets. The $\mu$ operator refers to the \emph{Lebesgue Measure} or
the generalized volume of the operand.

One such example is the \emph{Halton sequence}~\citep{H60}. We will use them extensively for our analysis because
they have been studied in the context of deterministic motion planning~\citep{JIP15,BLOY01}. 
Halton sequences are constructed by taking $d$ prime numbers, called generators, one for each dimension.
Each generator $g$ induces a sequence, called a Van der Corput sequence.
The $k$'th element of the Halton sequence is then constructed by juxtaposing the $k$'th element of each of the $d$ Van der Corput sequences.
For Halton sequences, tight bounds on dispersion exist. Specifically, $D_n(\calS) \leq p_d \cdot n^{-1/d}$ where $p_d \approx d \ \text{log } d$ is the $d^{th}$ prime number.
Subsequently in this paper, we will use $D_n$ (and not~$D_n(\calS)$) to denote the dispersion of the first $n$ points of~$\calS$.

Previous work bounds the length of the shortest path computed over an $r$-disk roadmap constructed using a low-dispersion deterministic sequence~\citep[Thm2]{JIP15}.
Specifically, given start (source) and goal (target) vertices, consider $\Gamma$, the set of all paths connecting them which 
have $\delta$-clearance for some $\delta$. A path has clearance $\delta$ if every point on the path is at a distance of at least
$\delta$ away from every obstacle. Set $\delta_{\text{max}}$ to be the maximal clearance over all such $\delta$.
If $\delta_{\text{max}} > 0$, then
for all $0<\delta \leq \delta_{\text{max}}$ 
set $c^*(\delta)$ to be the cost of the shortest path in $\Gamma$ with $\delta$-clearance.
Let $c(\ell,r)$ be the length of the path returned by a shortest-path algorithm on $\calG(\ell,r)$ with $\calS(\ell)$ having dispersion $D_\ell$.
For $2D_\ell < r < \delta$, we have that
\begin{equation}
\label{eq:dispersion_suboptimality}
c(\ell,r)
\leq 
\left( 
	1  + \frac{2D_\ell}{r - 2D_\ell}
\right) \cdot c^*(\delta).
\end{equation}

Notably, for $n$ random i.i.d. points, 
the dispersion is $O\left( (\log n / n )^{1/d}\right)$~\citep{deheuvels1983strong,N92}, which is strictly larger than
that of deterministic low-dispersion sequences.

For domains other than the unit hypercube, the insights from the analysis will generally hold. However, the dispersion bounds may become far more complicated depending on the domain, and the distance metric would need to be scaled accordingly. The quantitative bounds may then be difficult to deduce analytically.

\begin{figure*}[th]
	\centering
	\captionsetup[subfigure]{justification=centering}
	\begin{subfigure}[b]{0.66\columnwidth}
		\centering
		\includegraphics[width=0.9\textwidth]{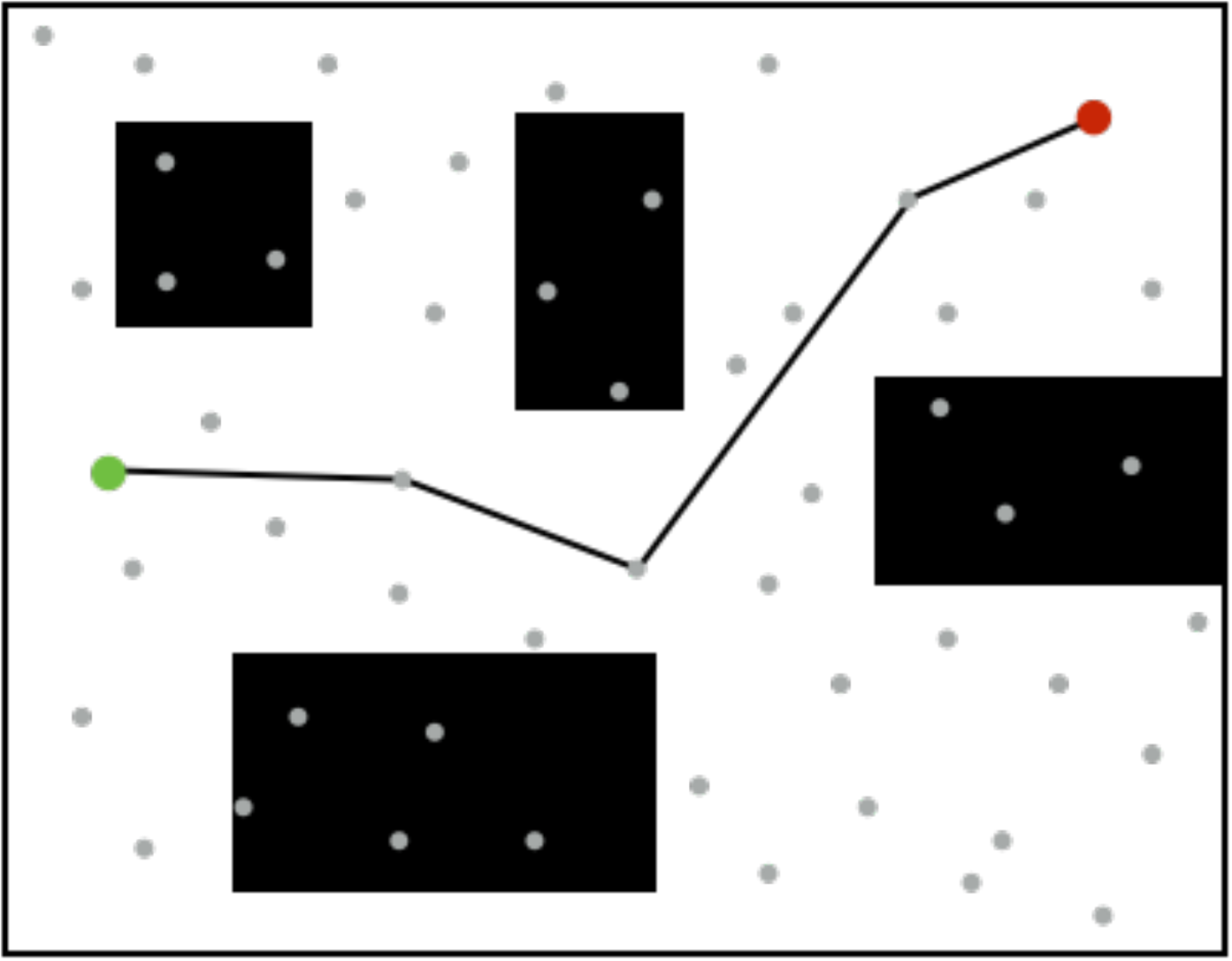}
		\caption{}
		\label{fig:approach-motiv1}
	\end{subfigure}
	\begin{subfigure}[b]{0.66\columnwidth}
		\centering
		\includegraphics[width=0.9\textwidth]{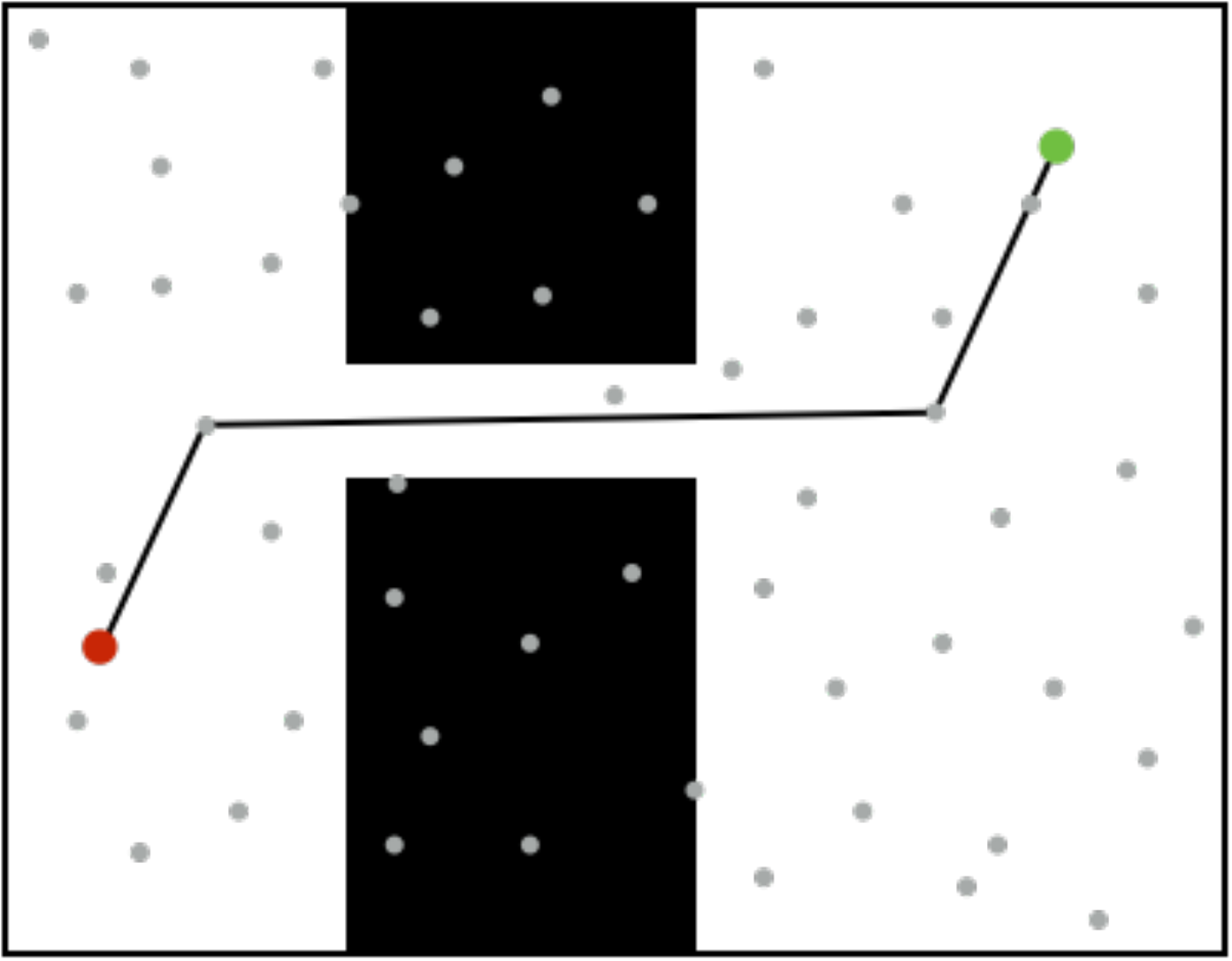}
		\caption{}
		\label{fig:approach-motiv2}
	\end{subfigure}
	\begin{subfigure}[b]{0.66\columnwidth}
		\centering
		\includegraphics[width=0.9\textwidth]{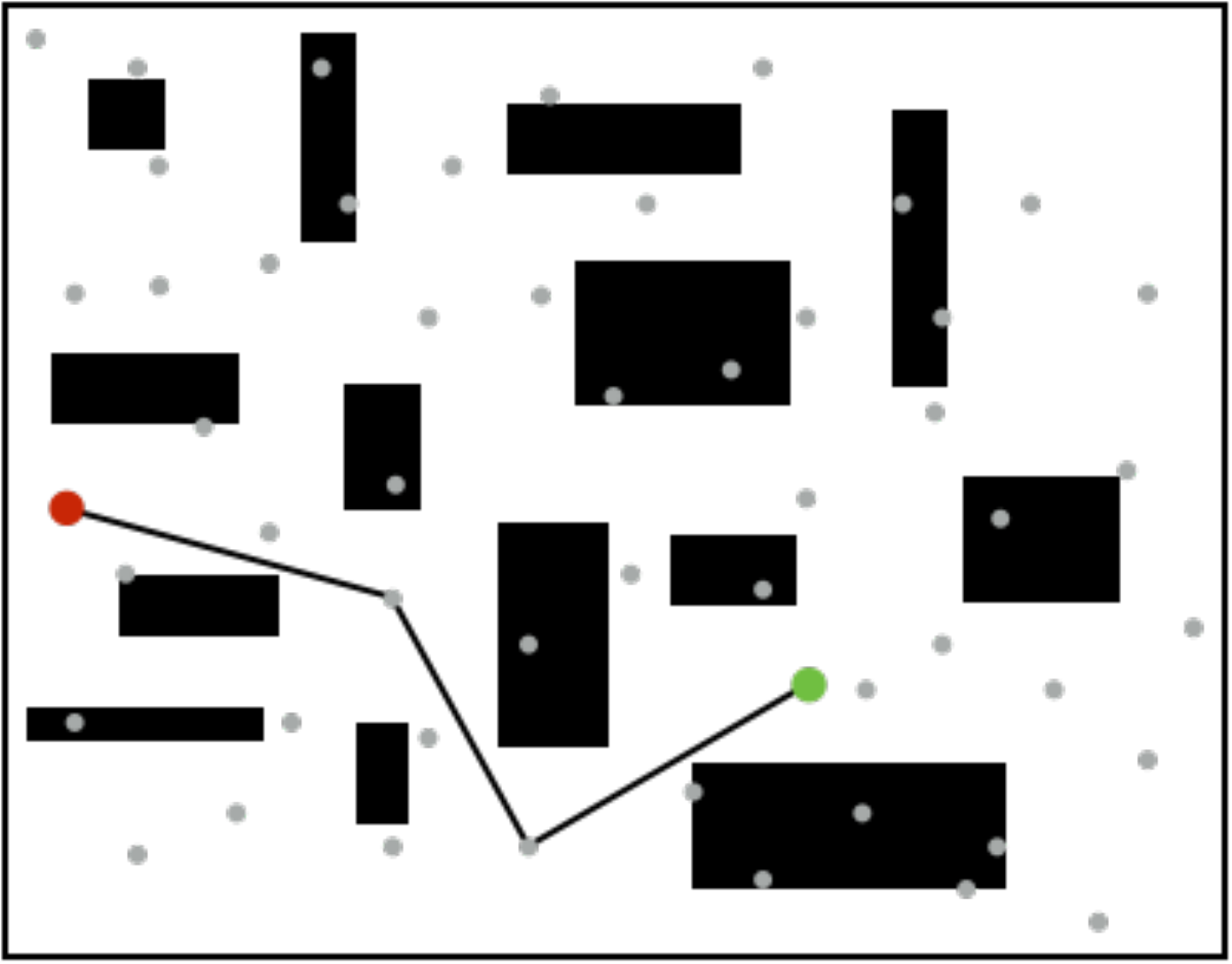}
		\caption{}
		\label{fig:approach-motiv3}
	\end{subfigure}
	\caption
	{
		Consider the range of toy problems shown here : (\protect\subref{fig:approach-motiv1}) A few obstacles with large gaps between them (\protect\subref{fig:approach-motiv2}) A single narrow passage 
		(\protect\subref{fig:approach-motiv3}) Several obstacles with narrow gaps between them.
		A large dense roadmap is able to find 
		feasible paths of good quality over such a diverse set of problems.
	}
	\label{fig:approach-motiv}
\end{figure*}

\subsection{Finite-time properties of sampling-based algorithms}
\label{sec:background-finite}

Extensive analysis has been done on \emph{asymptotic} properties of sampling-based algorithms,
i.e. properties such as connectivity and optimality, when the number of samples tends to infinity~\citep{KKL98, KF11}.
We are interested in bounding the quality of a solution obtained using a \emph{fixed} roadmap for a finite number of samples.
When the samples are generated from a \emph{deterministic low-dispersion} sequence, a closed-form solution 
bounds the quality of the best solution in an $r$-disk roadmap~\citep[Thm2]{JIP15}. The bound is a function of $r$, the number of vertices $n$ and the dispersion of the set of points used.
Similar bounds have been provided~\citep{DMB15} when randomly sampled i.i.d points are used.
Specifically, these bounds are for a PRM whose roadmap is an $r$-disk graph for a \emph{specific} radius $r = c \cdot \left(\log n / n\right)^{1/d}$ where $n$ is the number of points, $d$ is the dimension and $c$ is some constant.
There is also a bound on the probability that the quality of the solution will be larger than a given threshold.

\subsection{Efficient roadmap search algorithms}
\label{sec:background-efficient}

We are interested in roadmap search algorithms that attempt to reduce the amount of 
computationally expensive edge expansions performed.
This is typically done using heuristics such 
as in A*~\citep{hart1968formal},
Iterative Deepening A*~\citep{K85} and
Lazy Weighted A*~\citep{CPL14}.
Some of these algorithms, such as Lifelong Planning A*~\citep{koenig2004lifelong} allow recomputing the shortest path in an efficient manner when the graph undergoes changes.
\emph{Anytime} variants of A* such as
Anytime Repairing A*~\citep{likhachev2004ara}
and
Anytime Nonparametric A*~\citep{BSHG11}
efficiently run a succession of A* searches, each with an inflated heuristic.
A variant of Anytime D*~\citep{van2006anytime} is useful for efficiently recomputing 
paths in dynamic environments.
This potentially obtains a fast approximation and refines its quality as time permits
but in the worst-case these algorithms may have to search the entire graph~\citep{WR12}.
For a unifying formalism of such algorithms relevant to explicit roadmaps, in settings where edge evaluations are expensive, and for additional references, see~\citep{dellin2016unifying}.
Intelligently pre-constructing the roadmap helps achieve efficiency in high-dimensional C-spaces~\citep{SSH16b,SSH16}.
Another line of work has been to use roadmap spanners to produce sparse subgraphs that guarantee asymptotic
near-optimality and probabilistic completeness~\citep{marble2013asymptotically,dobson2014sparse}.

\subsection{Planning with Configuration Space Belief Models}
\label{sec:background-belief}

The probability of collision of a path is derived from an approximate model of the configuration space of the robot. Since we explicitly seek to minimize collision checks, we build up an incremental model using data from previous collision tests, instead of sampling several, potentially irrelevant configurations apriori. This idea has been studied in similar contexts~\citep{burns2003information,burns2005sampling}. Furthermore, the evolving probabilistic model can be used to guide future searches towards likely free regions. Previous work has analyzed and utilized this exploration-exploitation paradigm for faster motion planning \citep{rickert2008balancing,knepper2012real,pan2013faster,arslan2015dynamic}.


\section{Problem Definition}
\label{sec:problem}

Let~$\calX$ denote a $d$-dimensional C-space, $\Cfree$ the collision-free portion of $\calX$, 	$\Cobs = \calX \setminus \Cfree$ 
its complement
and
let $\zeta: \calX \times \calX \rightarrow \R$ be some distance metric.
For simplicity, we assume that $\calX = [0,1]^d$ and that $\zeta$ is the Euclidean norm.
Let  $\calS = \{ s_1, \ldots, s_N\}$ be a sequence of points 
where $s_i \in \calX$
for some $N \in \N$ 
and denote by $\calS(\ell)$ the first $\ell$ elements of $\calS$. 
We define the $r$-disk graph $\calG(\ell, r) = (V_{\ell},E_{\ell,r})$ 
where
$V_{\ell} = \calS(\ell)$, 
$E_{\ell,r} = \{(u,v) \ | \ u,v \in V_\ell \text{ and } \zeta(u,v) \leq r \}$
and each edge $(u,v)$ is of length $\zeta(u,v)$.
See~\citep{KF11, SSH16c} for various properties of such graphs in the context of motion planning.
Our definition assumes that $\calG$ is embedded in $\calX$. 
Set~$\calG = \calG(N, \sqrt{d})$, namely, the complete\footnote{Using a radius of $\sqrt{d}$ ensures that every two points will be connected due to the assumption that $\calX = [0,1]^d$ and that $\zeta$ is Euclidean.} graph defined over $\calS$. 
In a complete graph there is an edge between every pair of vertices.

To save computation, we do not evaluate the roadmap apriori, 
so we do not know if any of its vertices or edges are in $\Cfree$ or $\Cobs$. 
This setup is similar to that of LazyPRM \citep{bohlin2000path}.
We propose using the same roadmap (with online evaluations for each specific environment)
over a range of problem instances. Even though we cannot do apriori evaluation
of the roadmap for each new environment, using a fixed roadmap structure allows us to 
do some environment-agnostic preprocessing
that other classes of approaches like tree-growing~\citep{kuffner2000rrt} or trajectory
optimization~\citep{ratliff2009chomp} cannot. 
For instance, we can pre-compute all the nearest neighbors for each of the vertices
in the roadmap, and filter out any configurations in self-collision.
Previous work has shown that both the computation of nearest neighbours~\citep{kleinbort2016collision}
and the detection of self-collision~\citep{srinivasa2016system} are expensive
components of motion planning algorithms.

As we have motivated earlier (\sref{sec:intro}), large dense roadmaps
can achieve high-quality solutions over a wide range of problems. See \figref{fig:approach-motiv}
for an illustration of this with some $\mathbb{R}^{2}$ problems.
For ease of analysis we assume that the roadmap is complete, but our densification strategies and analysis can be extended to \emph{dense} roadmaps that are not complete.

A query $\calQ$ is a scenario with start and goal configurations. 
Let the start configuration be $s_1$ and the goal be $s_2$. 
The obstacles induce a mapping $\calM: \calX \rightarrow \{\Cfree,\Cobs\}$ called a \emph{collision detector} which checks if a configuration or edge is collision-free or not.
Typically, edges are checked by densely sampling along the edge (at some minimum resolution), and performing expensive collision checks for each sampled configuration, hence the term \emph{expensive edge evaluation}. 
A feasible path is denoted by $\gamma: [0,1] \rightarrow \Cfree$ where $\gamma[0] = s_1$ and $\gamma[1] = s_2$.
For the graph $\calG$, a path is \emph{feasible} if every included edge is in $\Cfree$.
Slightly abusing this notation, 
set $\gamma(\calG(\ell, r))$ to be the shortest collision-free path from $s_1$ to $s_2$ that can be computed in $\calG(\ell, r)$,
its clearance as $\delta(\calG(\ell, r))$ and denote by
$\gamma^* = \gamma(\calG)$ and
$\delta^* = \delta(\calG)$
the shortest path and its clearance that can be computed in $\calG$, respectively.

The anytime motion planning problem calls for finding a sequence of successively shorter
feasible paths  $\gamma_0$, $\gamma_1 \ldots$ in $\calG$,  converging
to $\gamma^{*}$.
We assume that $N = |\calS|$ is sufficiently large, and the roadmap covers the space well enough, so that for any reasonable set of obstacles, there are multiple feasible paths to be obtained between start and goal.


\section{Densification}
\label{sec:densification}

We now discuss in detail our first key idea: densification strategies for generating
subgraphs of the complete roadmap $\calG$~\citep{choudhury2017densification}. We focus on $r$-disk subgraphs of $\calG$, 
i.e. graphs defined by a specific set of vertices where any two vertices $u$ and $v$ are connected if their 
mutual distance $\zeta(u,v)$ is at most $r$. This induces a space of subgraphs (\figref{fig:ve_batching}) defined 
by the number of vertices and the connection radius (which, in turn, defines the number of edges).

We start by discussing the space of subgraphs, its boundaries and regions. 
Subsequently, we introduce two densification strategies---edge batching, which adds batches
of edges via an increasing radius of connectivity, and vertex batching, which adds batches
of vertices and searches over each complete subgraph.
These two have complementary behaviours with respect to problem difficulty, which motivates
a third strategy that is more robust to problem difficulty, which we call hybrid batching.

To evaluate densification (\sref{sec:experiments-densification}) we use an incremental path-planning algorithm that allows us to efficiently recompute shortest paths.
However, any alternative shortest-path algorithm may be used. 
We emphasize again that we focus on the meta-algorithm of choosing which subgraphs to search.
The details of the actual implementation for experiments are provided subsequently.

\begin{figure}
    \centering
    \includegraphics[width=0.43\textwidth]{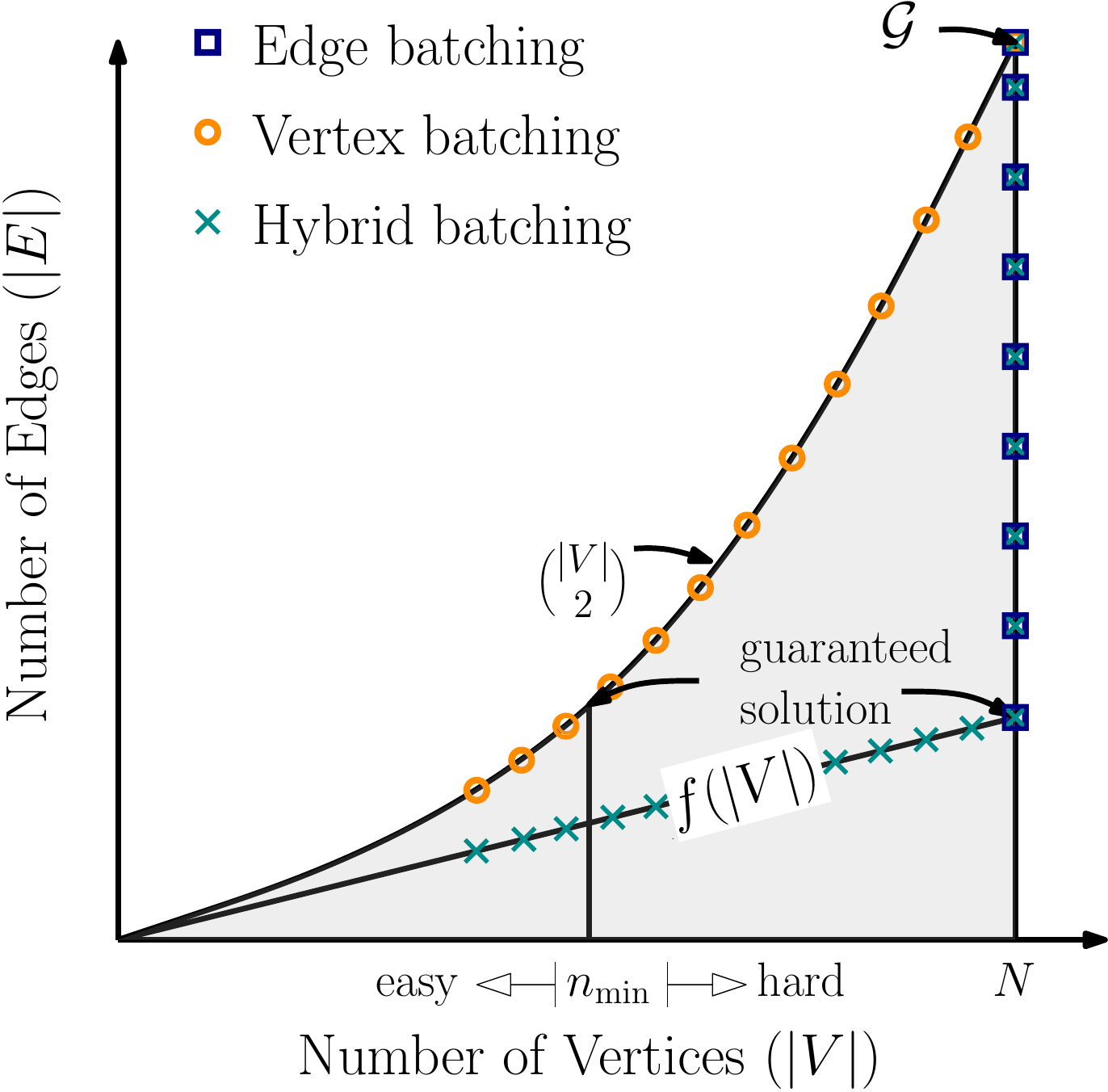}
    \caption{Our meta algorithm provides shortest-path planning algorithms with a sequence of subgraphs.
    To do so we consider densification strategies for traversing the space of $r$-disk subgraphs
    of the roadmap $\calG$.
    The $x$-axis represents the number of vertices of the subgraph and the $y$-axis the number of edges (induced by the radius $r$). A particular subgraph is defined by a point in this space.
    \emph{Edge batching} searches with all vertices and adds edges according to an increasing radius of connectivity. 
    \emph{Vertex batching} searches over complete subgraphs induced by progressively larger subsets of vertices.
    \emph{Hybrid batching} uses the minimal connection radius $f(|V|)$ to ensure connectivity until it reaches~$|V| = N$ and then behaves like edge batching.}
    \label{fig:ve_batching}
\end{figure}

\subsection{The space of subgraphs}
\label{sec:densification-space}


We depict the set of possible graphs~$\calG(\ell, r)$ for all choices of  $0 < \ell \leq N$ and $0 < r \leq \sqrt{d}$ in
\figref{fig:ve_batching}.
Since $\calS$ is a sequence of points or samples, the vertices of $\calG(\ell, r)$ represent a unique subsequence, i.e.
the first $\ell$ points of $\calS$.
Therefore every point in this space is a unique subgraph.
\figref{fig:ve_batching} shows $|E_{\ell, r}|$ as a function of $|V_\ell|$.
We discuss it in detail to motivate our approach for solving the problem of anytime planning on large dense roadmaps and the specific sequence of subgraphs we use.
First, consider the curves that define the boundaries of all possible graphs:
The vertical line $|V| = N$ corresponds to subgraphs defined over the entire set of vertices, where batches of edges are added as $r$ increases.
The parabolic arc $|E| = |V|\cdot(|V|-1)/2$, corresponds to complete subgraphs defined over increasingly larger sets of vertices.

We wish to approximate the shortest path $\gamma^*$ which has some minimal clearance $\delta^*$.
To ensure that a path that approximates $\gamma^*$ is found, 
the graph should meet two conditions:
(i)~A minimum number of vertices $n_{\text{min}}$ to ensure sufficient coverage of the \Cs. The exact value of~$n_{\text{min}}$ will be a function of the dispersion $D_{n_{\text{min}}}$ of the first $n_{\text{min}}$ points in the sequence $\calS$ and the clearance $\delta^*$.
(ii)~A minimal connection radius $r_0$ to ensure that the graph is connected.
Its value will depend on the sequence $\calS$ (and not on~$\delta^*$).

In \figref{fig:ve_starvation}, requirement (i) induces a vertical line at $|V| = n_{\text{min}}$.
Any point to the left of this line corresponds to a graph with too few vertices to guarantee that a solution will be found. We call this the \emph{vertex-starvation} region.
Requirement~(ii) induces a curve $f(|V|)$ such that any point below this curve corresponds to a graph which may be disconnected
due to having too few edges. 
We call this the \emph{edge-starvation} region.
The exact form of the curve depends on the sequence~$\calS$ that is used. 
Any point outside the starvation regions represents a graph~$\calG(\ell,r)$ such that the length of $\gamma(\calG(\ell,r))$ may be bounded.
Specific bounds are discussed in our subsequent analysis.

\begin{figure}
\centering
    \includegraphics[width=0.43\textwidth]{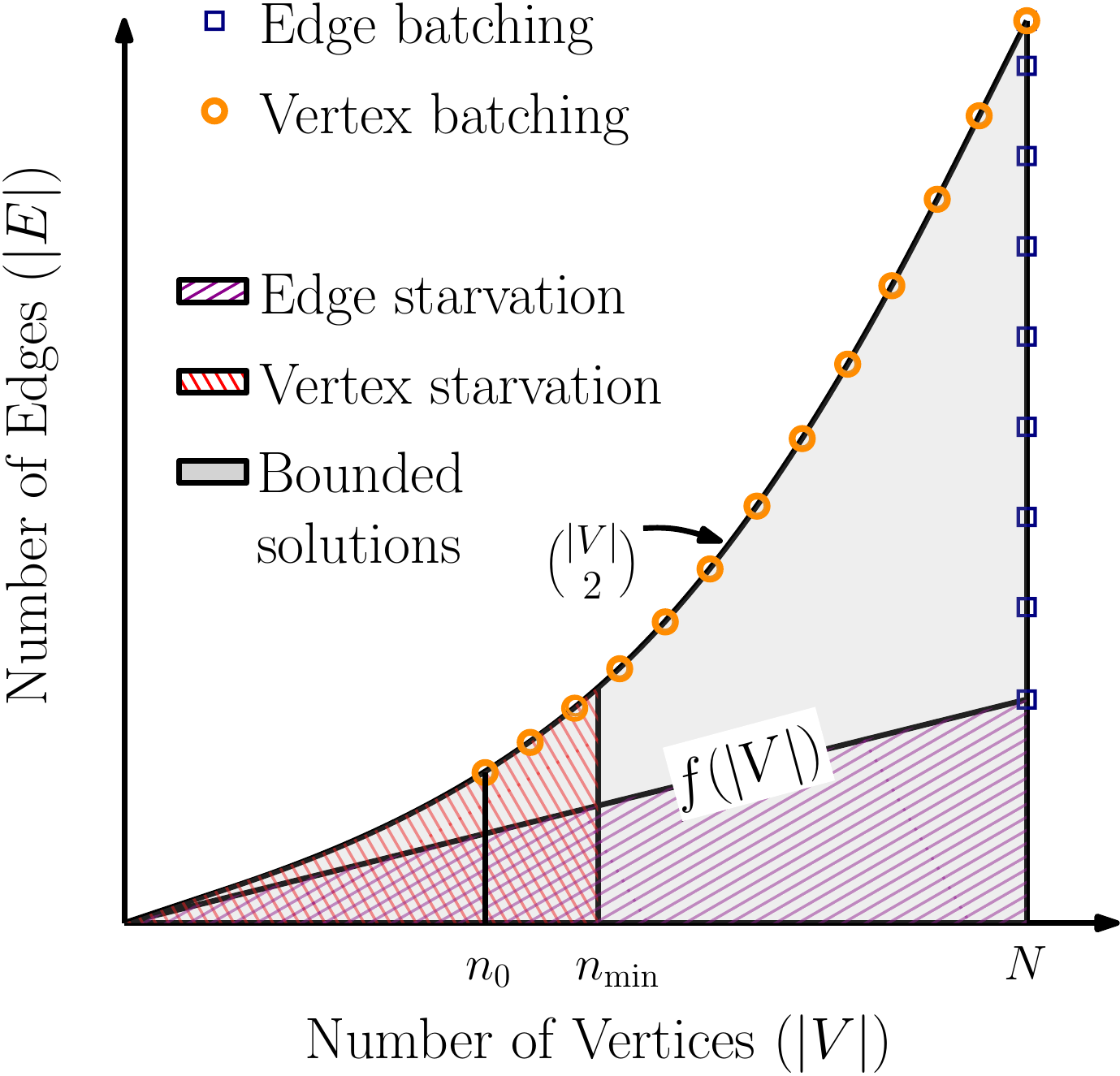}
    \caption{\emph{Vertex Starvation} happens in the region
    with too few vertices to ensure a solution, even for a fully connected subgraph. \emph{Edge Starvation}
    happens in the region where the radius $r$ is too low to guarantee connectivity.}
    \label{fig:ve_starvation}
\end{figure}

\subsection{Densification Strategies}
\label{sec:densification-strategies}

Our goal is to search for increasingly dense subgraphs of~$\calG$.
This corresponds to a sequence of points $\calG(n_i,r_i)$ on the space of subgraphs (\figref{fig:ve_starvation}) 
that ends at the complete roadmap $\calG$ at the upper right corner of the space.
If no feasible path exists in the subgraph, we move on to the next subgraph in the sequence, which is more likely to have a feasible path.
We discuss three general densification strategies for traversing the space of subgraphs.

\subsubsection{Edge Batching}
\hfill\\
In edge batching, all subgraphs have the complete set of vertices $\calS$ and 
the edges are incrementally added
via an increasing connection radius.
Specifically, $\forall i~n_i = N$ and $r_{i} = \eta_{e} r_{i-1}$ where
$\eta_{e} > 1$ and $r_{0}$ is some small initial radius. 
Here, we choose
$r_0 = O(f(N))$, where $f$ is the edge-starvation boundary
curve defined previously. It defines the minimal radius to ensure connectivity
(in the asymptotic case) using $r$-disk graphs.
Specifically, 
$f(N) = O\left(N^{-1/d}\right)$ for low-dispersion deterministic sequences
and 
$f(N) = O\left( \left(\log N / N\right)^{-1/d}\right)$ for random i.i.d sequences.
As shown in \figref{fig:ve_starvation}, this induces a sequence of points along the vertical line at $|V| = N$ starting from $|E| = O(N^2 r_0^d)$ and ending at $|E| = \Theta(N^2)$.

\begin{figure}
        \centering
    \includegraphics[width=0.43\textwidth]{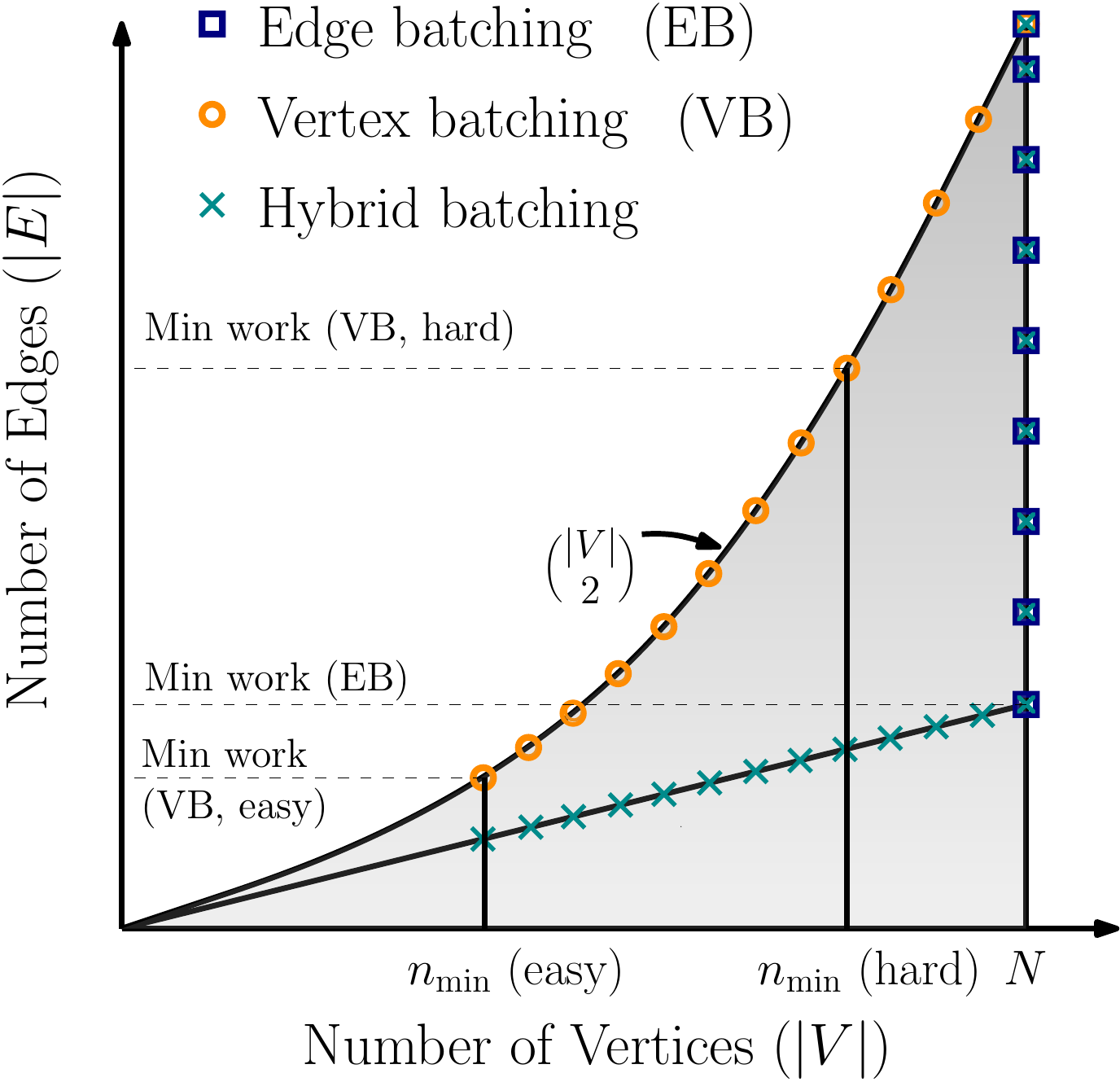}
    \caption[The effect of hardness on densification strategies]{\ A visualization of the work required by our densification strategies as a function of the problem's hardness. Here work is measured as the number of edges evaluated.
    This is visualized using the gradient shading where light gray (dark gray) depicts a small (large) amount of work.
    Assuming $N > n_{\min}$, the amount of work required by edge batching remains the same regardless of problem difficulty.
    For vertex batching the amount of work required depends on the hardness of the problem.
    Hybrid batching does less work than edge batching for easy problems and less work than vertex batching
    for hard problems.
    Here we visualize an easy and a hard problem using~$n_{\min}$ (easy) and $n_{\min}$ (hard), respectively.}
    \label{fig:ve_comparitive}
\end{figure}

\begin{figure*}[t]
\captionsetup[subfigure]{justification=centering}
    \begin{subfigure}[b]{0.324\columnwidth}
    \centering
        \frame{\includegraphics[width=0.9\textwidth]{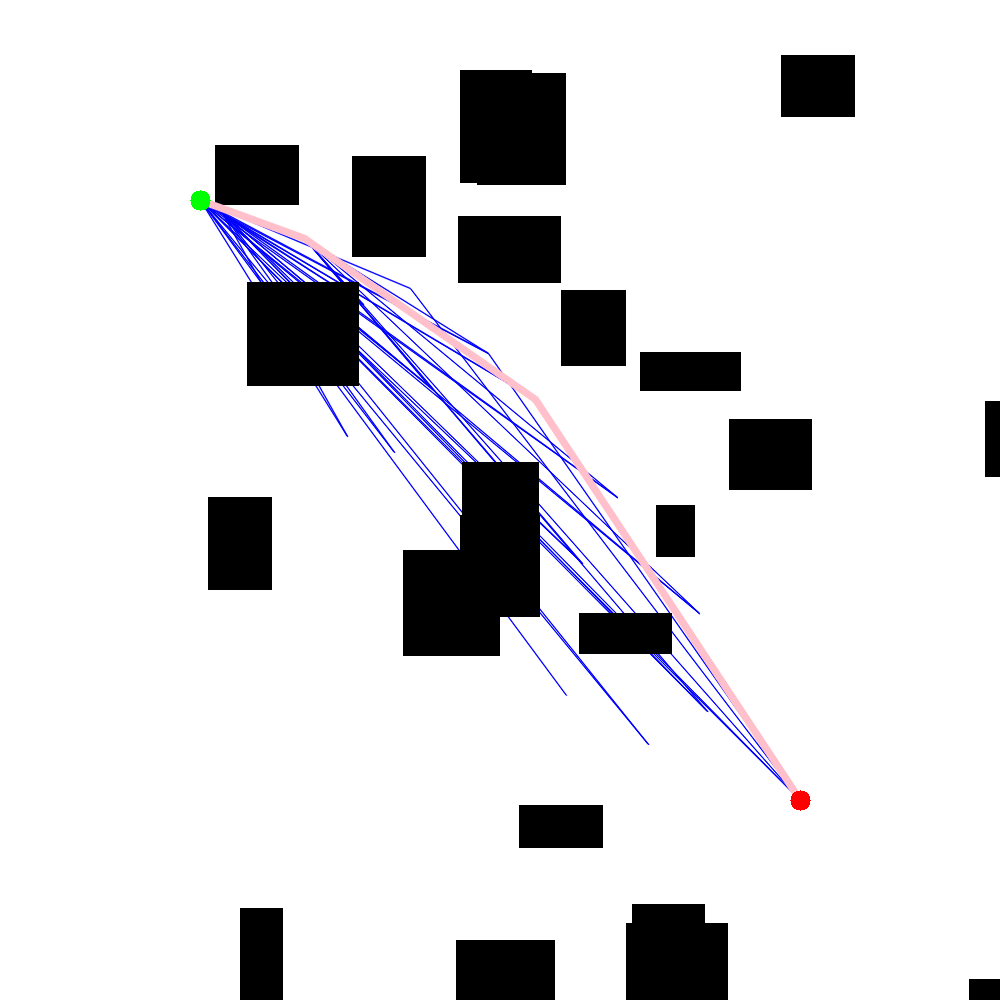}}
        \caption{40 checks}
        \label{fig:viz2d_vertex_easy1}
    \end{subfigure}
    \begin{subfigure}[b]{0.324\columnwidth}
    \centering
        \frame{\includegraphics[width=0.9\textwidth]{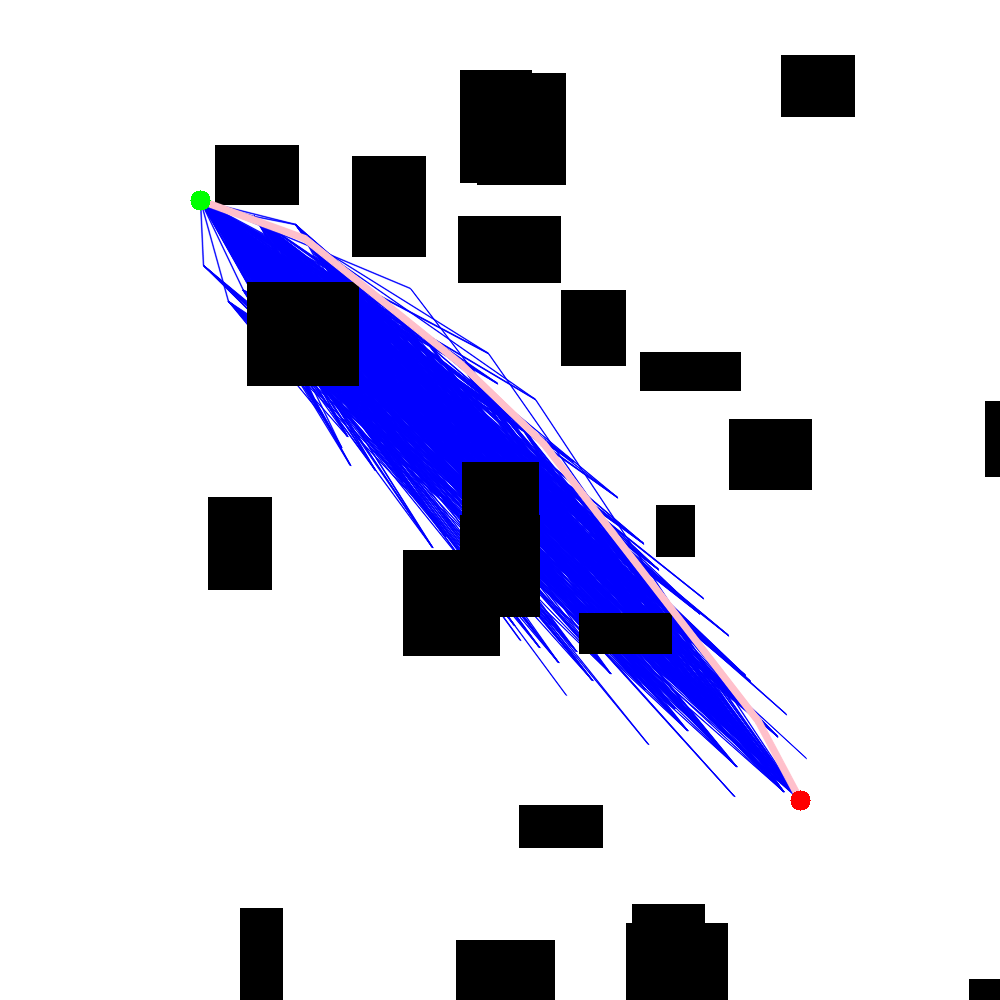}}
        \caption{953 checks}
        \label{fig:viz2d_vertex_easy2}
    \end{subfigure}
    \begin{subfigure}[b]{0.324\columnwidth}
    \centering
        \frame{\includegraphics[width=0.9\textwidth]{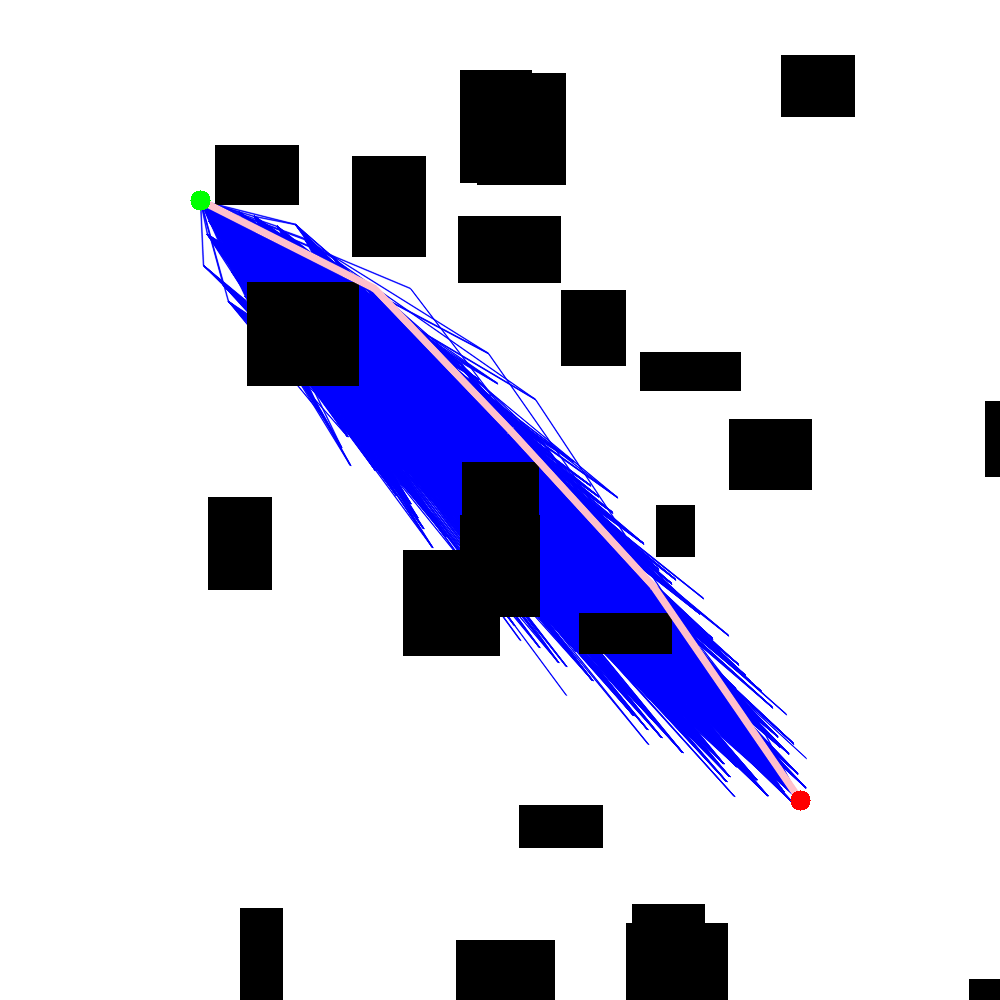}}
        \caption{6,310 checks}
        \label{fig:viz2d_vertex_easy3}
    \end{subfigure}
    \begin{subfigure}[b]{0.324\columnwidth}
    \centering
        \frame{\includegraphics[width=0.9\textwidth]{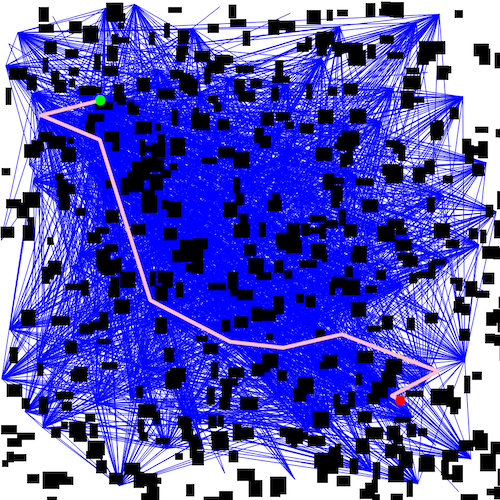}}
        \caption{2,573 checks}
        \label{fig:viz2d_vertex_hard1}
    \end{subfigure}
    \begin{subfigure}[b]{0.324\columnwidth}
    \centering
        \frame{\includegraphics[width=0.9\textwidth]{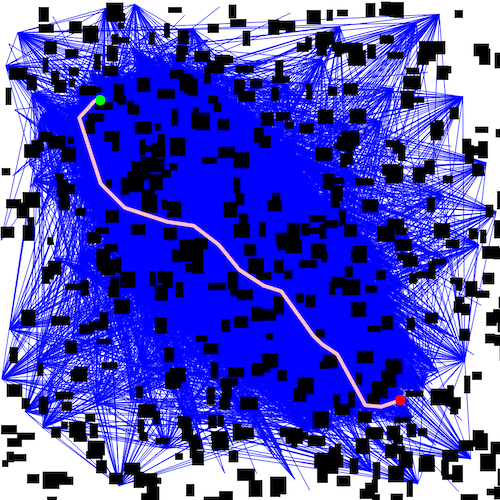}}
        \caption{61,506 checks}
        \label{fig:viz2d_vertex_hard2}
    \end{subfigure}
    \begin{subfigure}[b]{0.324\columnwidth}
    \centering
        \frame{\includegraphics[width=0.9\textwidth]{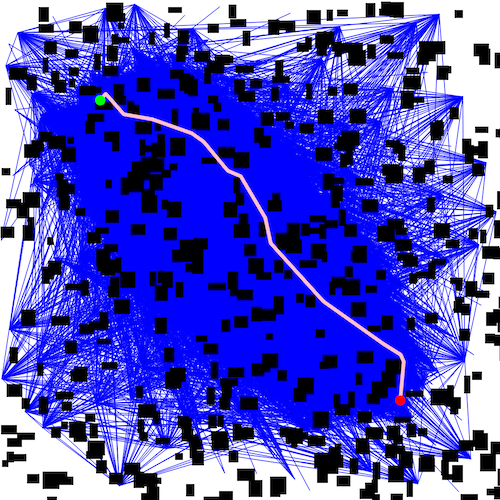}}
        \caption{164,504 checks}
        \label{fig:viz2d_vertex_hard3}
    \end{subfigure}

    \begin{subfigure}[b]{0.324\columnwidth}
    \centering
        \frame{\includegraphics[width=0.9\textwidth]{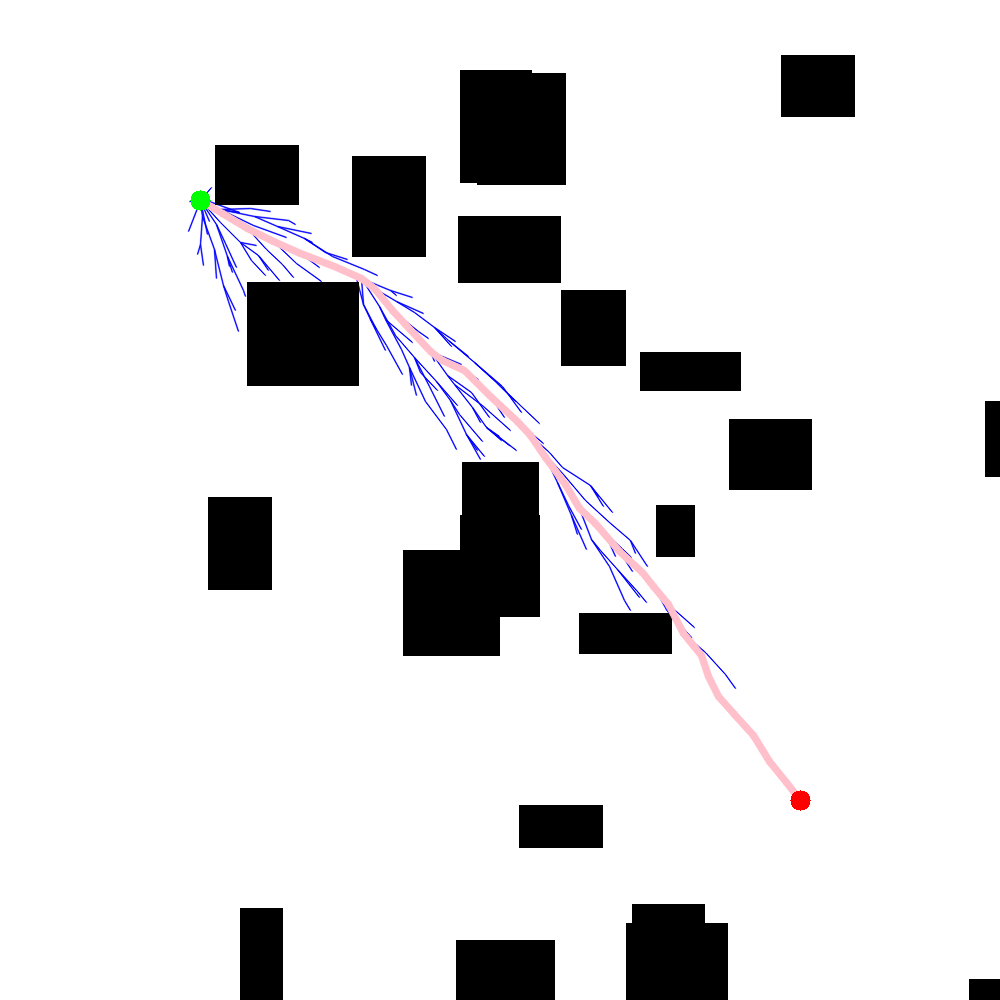}}
        \caption{206 checks}
        \label{fig:viz2d_edge_easy1}
    \end{subfigure}
    \begin{subfigure}[b]{0.324\columnwidth}
    \centering
        \frame{\includegraphics[width=0.9\textwidth]{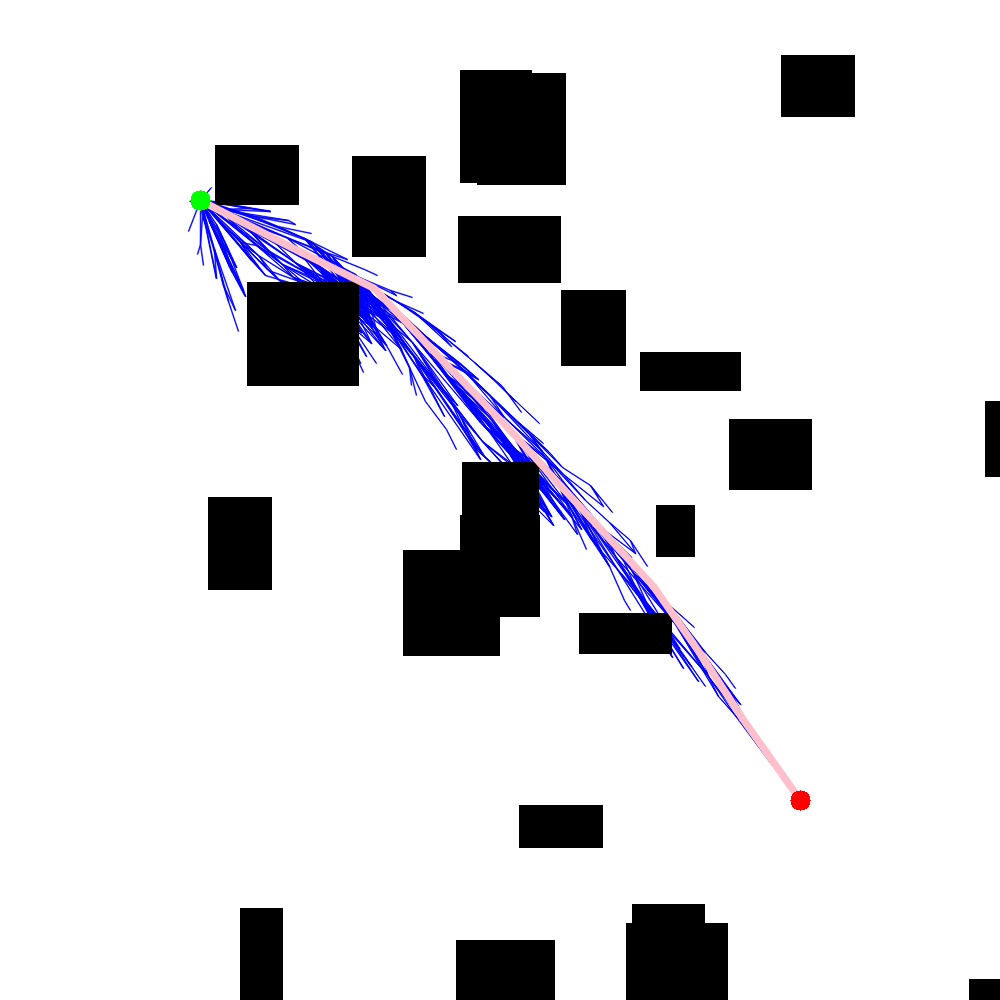}}
        \caption{676 checks}
        \label{fig:viz2d_edge_easy2}
    \end{subfigure}
    \begin{subfigure}[b]{0.324\columnwidth}
    \centering
        \frame{\includegraphics[width=0.9\textwidth]{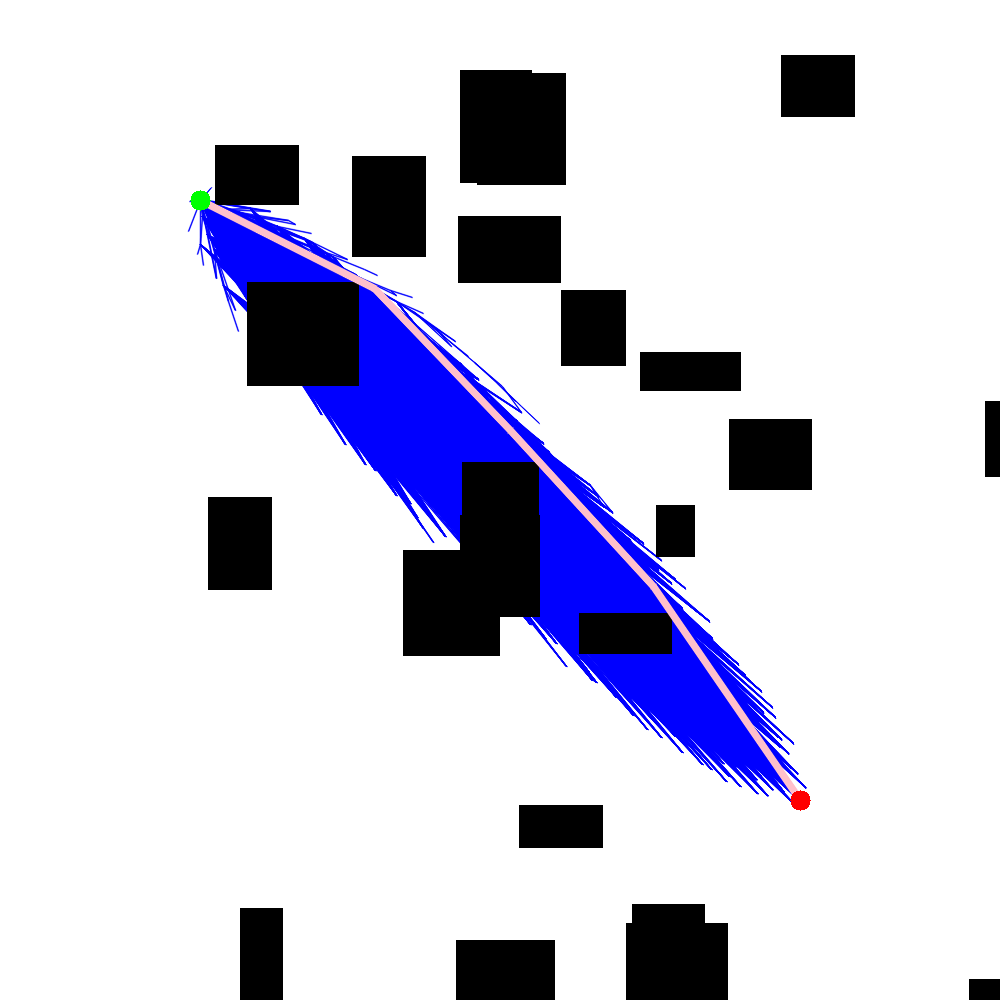}}
        \caption{15,099 checks}
        \label{fig:viz2d_edge_easy3}
    \end{subfigure}
    \begin{subfigure}[b]{0.324\columnwidth}
    \centering
        \frame{\includegraphics[width=0.9\textwidth]{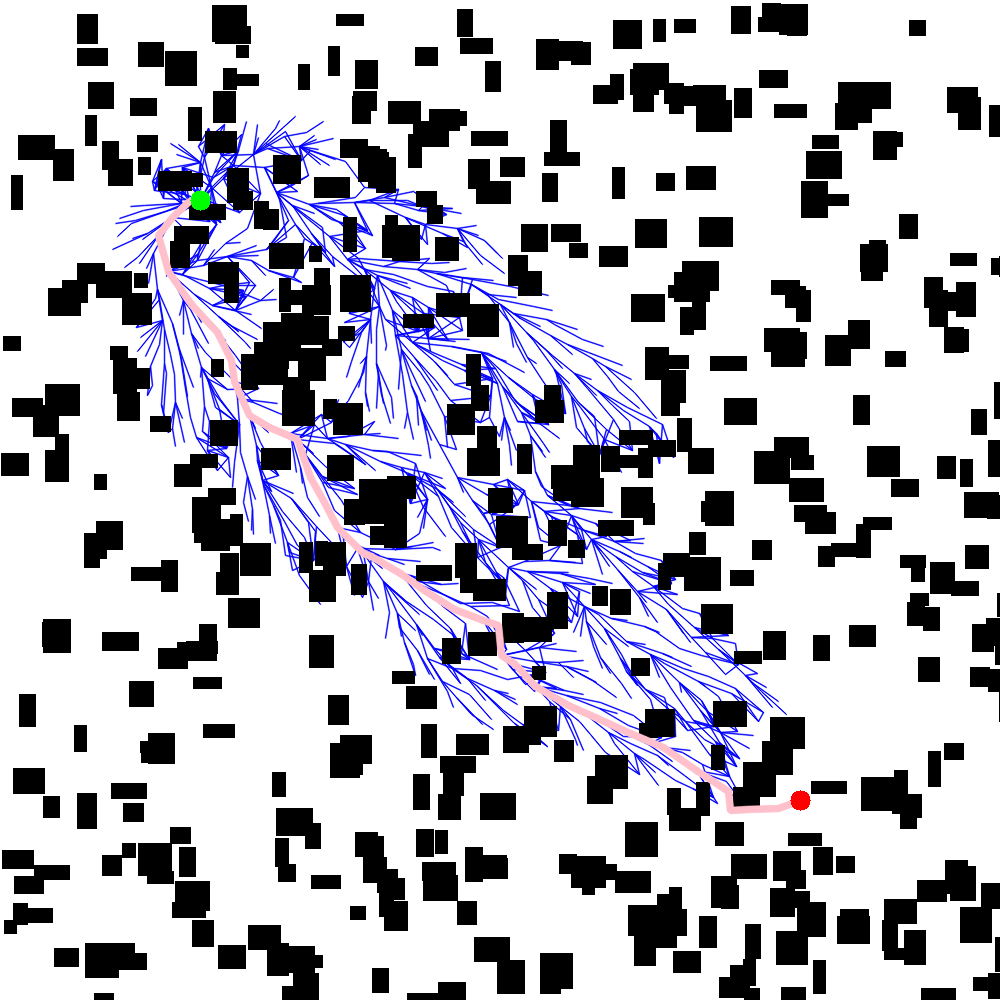}}
        \caption{1390 checks}
        \label{fig:viz2d_edge_hard1}
    \end{subfigure}
    \begin{subfigure}[b]{0.324\columnwidth}
    \centering
        \frame{\includegraphics[width=0.9\textwidth]{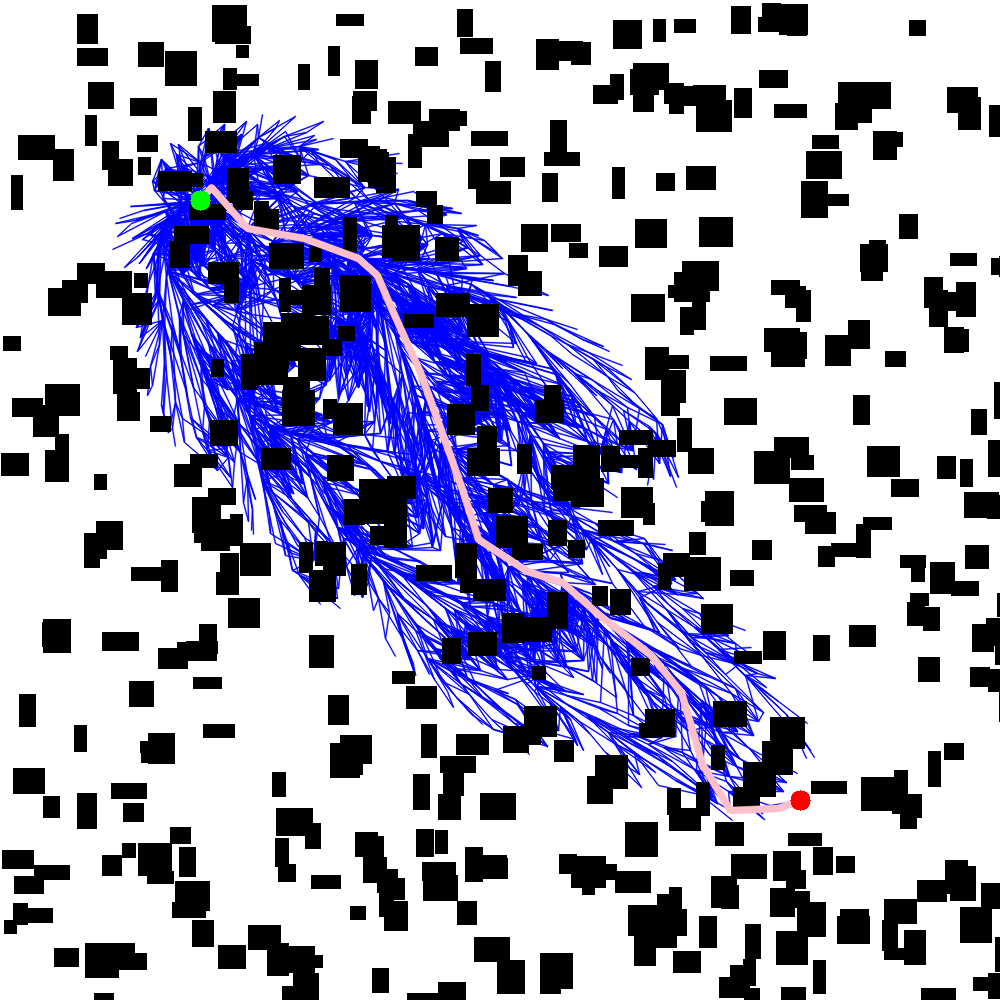}}
        \caption{4,687 checks}
        \label{fig:viz2d_edge_hard2}
    \end{subfigure}
    \begin{subfigure}[b]{0.324\columnwidth}
    \centering
        \frame{\includegraphics[width=0.9\textwidth]{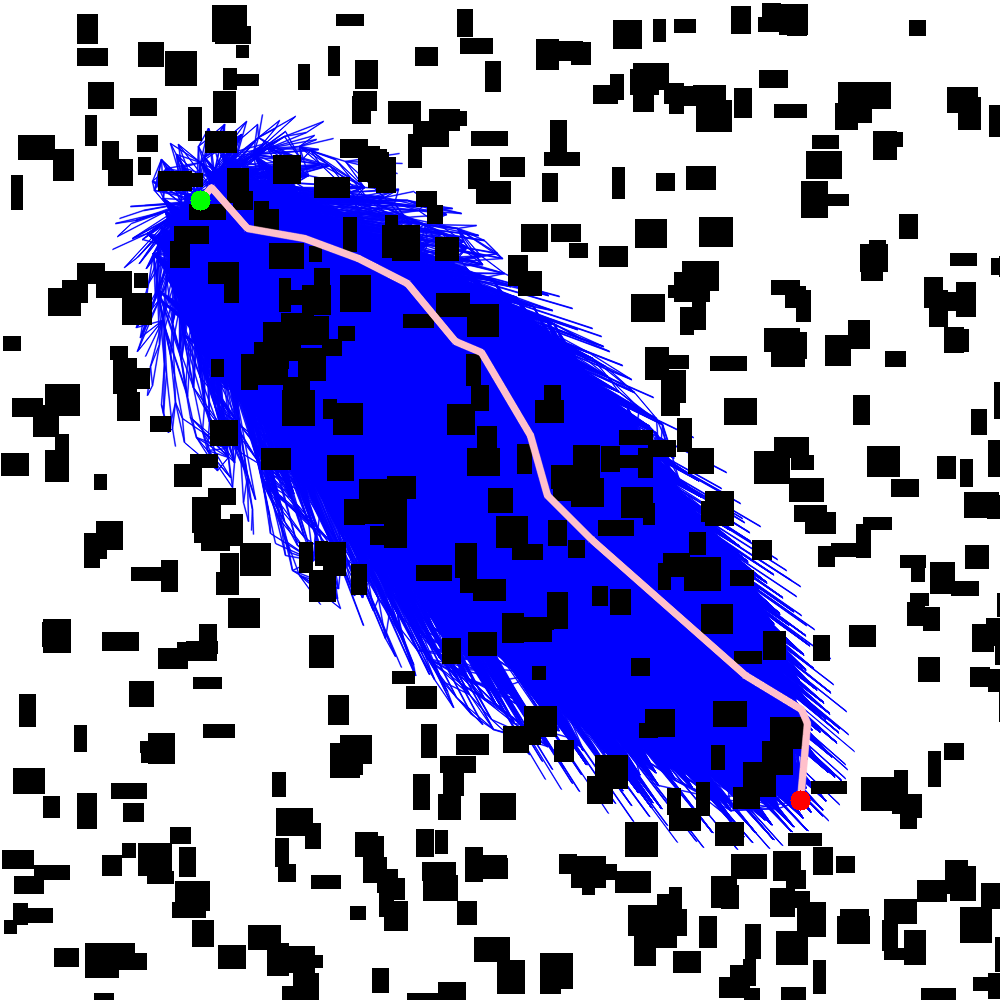}}
        \caption{78,546 checks}
        \label{fig:viz2d_edge_hard3}
    \end{subfigure}
    \caption{
    Visualizations of vertex batching (upper row) and edge batching (lower row) on easy (left pane)
    and hard (right pane)
    $\R^{2}$ problems respectively. The same set of samples $\calS$ is used in each case. For easy problems, vertex batching finds the 
    first solution quickly with a sparse set of initial samples. Additional heuristics hereafter help it converge to the optimum with fewer edge evaluations than edge batching. The harder problem has $10\times$ more obstacles and lower average obstacle gaps. Therefore, both vertex and edge batching require
    more edge evaluations for finding feasible solutions and the shortest path on it. In particular, vertex batching requires multiple iterations to find its first solution,
    while edge batching still does so on its first search, albeit with more collision checks than for the easy problem. Thus edge batching finds solutions faster than vertex batching on hard problems. Note that the coverage of collision checks appears similar at the end due to resolution
    limits for visualization.
    }
    \label{fig:viz2d}
\end{figure*}

\subsubsection{Vertex Batching}
\hfill\\
In this variant, all subgraphs are complete graphs defined over increasing subsets of the complete set of vertices $\calS$.
Specifically $\forall i~r_i = \sqrt{d}$,
$n_i = \eta_{v} n_{i-1}$ where $\eta_{v} > 1$ and the base term $n_0$ is some small number of vertices. Because we have no priors about the obstacle density or distribution, the chosen $n_0$ is a constant and does not vary due to $N$ or due to the volume of $\Cobs$.
As shown in~\figref{fig:ve_starvation}, this induces a sequence of points along the parabolic arc $|E| = |V|\cdot (|V|-1)/2$
starting from $|V| = n_0$ and ending at $|V| = N$. The vertices are chosen in the same
order with which they are generated by~$\calS$. So, $\calG_0$ has the first $n_0$ samples of
$\calS$, and so on.

\vspace{\baselineskip}
\noindent
The relative performance of these two
strategies depends on the hardness of the problem. 
We use the clearance of the shortest path, $\delta^{*}$, as a proxy for problem hardness.
This, in turn, defines $n_{\min}$ which bounds the vertex-starvation region.
Specifically, we say that a problem is easy (hard) when
$\delta^* \approx \sqrt{d}$ 
($\delta^* \approx \Omega(D_n(\calS))$).
For easy problems, with larger gaps between obstacles,
where  $\delta^{*} \approx O(\sqrt{d})$, 
vertex batching can find a solution quickly with fewer samples and long edges, thereby restricting the work required
for future searches.
In contrast, assuming $N > n_{\min}$, edge batching will find a solution on the first iteration but 
the time to do so may be far greater than for vertex batching because the number of samples
is so large.
For hard problems
where  $\delta^{*} \approx O(D_n)$,
vertex batching may require multiple iterations until the number of samples it uses is large enough and it is out of the vertex-starvation region.
Each of these searches would exhaust the complete subgraph before terminating. 
This cumulative effort would exceed that required by edge batching for the same problem, as the latter can find a feasible
(albeit sub-optimal) path on the first search.
A visual depiction of this intuition is given in \figref{fig:ve_comparitive}.
Since we are focused on problems
with expensive edge evaluations, we treat the work due to edge evaluations
as a reasonable approximation of the total work done by the search.
An empirical example of this is shown in \figref{fig:viz2d}.

\subsubsection{Hybrid Batching}
\hfill\\
Vertex and edge batching exhibit generally complementary properties for 
problems with varying difficulty.
Yet, when a query~$\calQ$ is given, the hardness of the problem is not known a-priori.
Therefore, we propose a hybrid approach that exhibits favourable properties regardless of the hardness of the problem.

This can be visualized on the space of subgraphs as sampling along the curve $f(|V|)$ from $|V| = n_0$ until $f(|V|)$ intersects $|V| = N$ and then sampling along the vertical line~$|V| = N$. 
See~\figref{fig:ve_batching} and~\figref{fig:ve_comparitive}.
As we shall see in our experiments, hybrid batching typically performs comparably (in terms of anytime planning performance) to vertex batching on easy problems and to edge batching on hard problems.

If the problem is easy, then hybrid batching
finds a feasible solution early on, typically when the number of vertices
is similar to that needed by vertex batching for a feasible solution. Thus, the work
would be far less than that for edge batching.
On the other hand, if the problem is hard, then hybrid batching 
would have to get much closer to the~$|V| = N$ line before the dispersion becomes low enough to find a solution. However, it would
not involve as much work as for vertex batching, because the radius decreases as the number of vertices increases, unlike vertex batching which uses $r_i = \sqrt{d}$ for every iteration~$i$. So on hard problems, it does far less work
than vertex batching.

\subsection{Analysis for Halton Sequence}
\label{sec:densification-analysis}

In this section we consider the space of subgraphs and the densification strategies that we introduced
for the specific case that $\calS$ is a Halton sequence (\sref{sec:background-dispersion}).
We start by describing the boundaries of the starvation regions.
We then simulate the bound on the quality of the solution obtained as a function of the work done for each of our strategies.

Since we are considering $[0,1]^{d}$, the unit hypercube, $\delta_{\text{max}} \leq \sqrt{d}$.
We use \eref{eq:dispersion_suboptimality} 
to first obtain bounds on the vertex  and edge starvation regions, and subsequently
analyze the tradeoff between worst-case work and solution quality for vertex and edge batching.

\subsubsection{Starvation Region Bounds}
\hfill\\
To bound the vertex starvation region we compute the~$n_{\text{min}}$ after which bounded sub-optimality can be guaranteed
to find the first solution. 
Note that $\delta^*$ is the clearance of the shortest path $\gamma*$ in $\calG$ connecting $s_1$ and $s_2$,
that~$p_d$ denotes the $d^{th}$ prime, and that $D_N \leq p_d / N^{1/d}$ for Halton sequences.
For~\eref{eq:dispersion_suboptimality} to hold we require that $2D_{n_{\text{min}}} < \delta^{*}$.
Thus,

\begin{equation}
\label{eq:vertex_starvation}
2D_{n_{\text{min}}} < \delta^{*} 
\Rightarrow
2\frac{p_d}{n_{\text{min}}^{1/d}} < \delta^{*}
\Rightarrow
n_{\text{min}} > \left( \frac{2p_d}{\delta^*}\right)^d.
\end{equation}
As the problem becomes harder (as $\delta^*$ decreases),
$n_{\text{min}}$ and 
the vertex-starvation region grows.

We now show that for Halton sequences, the edge-starvation region has a linear boundary, i.e. $f(|V|) = O(|V|)$.
Using \eref{eq:dispersion_suboptimality} the minimal radius $r_{\min}(|V|)$ required for a graph with $|V|$ vertices is 
\begin{equation}
\label{eq:edge_starvation}
r_{\min}(|V|) > 2D_{|V|} 
\Rightarrow 
r_{\min}(|V|) > \frac{2 p_d}{(|V|)^{1/d}}.
\end{equation}
For any $r$-disk graph~$\calG(\ell, r)$, the number of edges is $|E_{\ell, r}| = O\left(\ell^2 \cdot r^{d}\right)$.
In our case,
\begin{equation}
\label{eq:f_v_line}
f(|V|) 
= O\left(|V|^2 \cdot r_{\min}^{d}(|V|)\right)
= O(|V|).
\end{equation}

\begin{figure}
	\centering
    \captionsetup[subfigure]{justification=centering}
	\begin{subfigure}[b]{0.49\columnwidth}
	\centering
		\includegraphics[width=\textwidth]{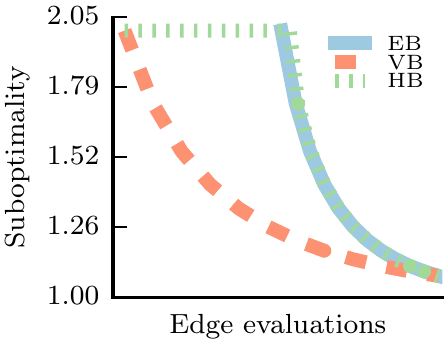}
		\caption{  Easy problem}
		\label{fig:ratio_easy}
	\end{subfigure}
	\begin{subfigure}[b]{0.49\columnwidth}
	\centering
		\includegraphics[width=\textwidth]{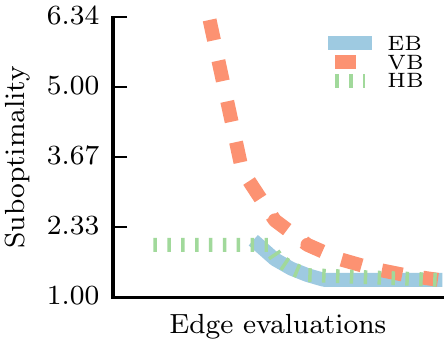}
		\caption{  Hard problem}
		\label{fig:ratio_hard}
	\end{subfigure}

	\caption{ A simulation of the work-suboptimality tradeoff for vertex, edge  and hybrid batching. 
	The number of samples and dimensionality are $n = 10^6$ and $d=4$ respectively. 
	The easy and hard problems have
	$\delta^{*} = \sqrt{d} /2$ and $\delta^{*} = 5D_{n}$, respectively.	
	The plot is produced by sampling points along the curves $|V| = n$ and $|E| = |V| \cdot (|V|-1|)/2$ and using the respective values in \eref{eq:dispersion_suboptimality}.
	Note that the $x$-axis is in log-scale.
}
	\label{fig:ratio}
\end{figure}

\subsubsection{Effort-to-Quality Ratio}
\hfill\\
We now compare our densification strategies in terms of their worst-case anytime performance. Specifically, we plot the cumulative amount of worst-case work as subgraphs are searched, measured by the maximum number of edges that may be evaluated, as a function of the bound on the quality of the solution that may be obtained using \eref{eq:dispersion_suboptimality}.
We fix a specific setting (namely $d$ and $n$) and simulate the work done and the suboptimality using the necessary formulae.
This is done for an easy and a hard problem. See~\figref{fig:ratio}. 

Indeed, this simulation coincides with our discussion on the properties of all batching strategies with respect to the problem difficulty.
Vertex batching outperforms edge batching on easy problems and vice versa. Hybrid batching lies somewhere in between the two approaches
with the specifics depending on problem difficulty.

\subsection{Implementation}

Our analysis so far has been independent of any parameters for the 
densification or any other implementation decisions. Here we outline the specifics 
of how we implement the densification strategies for evaluating them.

\subsubsection{Densification Parameters}
\hfill\\
We choose the parameters for each densification strategy such 
that the number of batches is $O(\text{log}_2N)$.
\\
\\
EDGE BATCHING
\\
\indent
We set $\eta_e = 2^{1/d}$ . 
Recall that for $r$-disk graphs, the average degree of vertices is~$n \cdot r_{i}^{d}$, therefore this average degree (and hence the number of edges) is doubled after each iteration. 
We set $r_0 = 3\cdot N^{-1/d}$.
\\
\\
VERTEX BATCHING
\\
\indent
We set the initial number of vertices~$n_0$ to be $100$,
irrespective of the roadmap size and problem setting,
and set $\eta_v = 2$.
After each batch
we double the number of vertices.
\\
\\
HYBRID BATCHING
\\
\indent
The parameters are derived from those of vertex and edge batching.
We begin with $n_0 = 100$, and in each batch we increase the vertices
by a factor of~$\eta_v = 2$. For these searches, i.e. in the region where $n_i < N$,
we use $r_i = 3 \cdot n_i^{-1/d}$. This ensures the same radius
at $n_i = N$ as for edge batching. Subsequently, we increase the radius as
$r_i = \eta_e r_{i-1}$, where $\eta_e = 2^{1/d}$. 
\\

These parameters let us bound the worst-case complexity of the total work (measured in edge evaluations)
by the batching strategies to be of the same order
as naively searching the complete roadmap, for which 
this work is ${N\choose2} = O(N^2)$.
\\
For vertex batching, the number of
vertices at a given iteration is~$n_i = O(2^i)$ and the number of edges is~$|E_i| = n_i^2 / 2 = O(2^{2i})$.
The worst-case complexity for vertex batching is
\begin{equation}
\label{eq:vb_worstcase_work}
\sum\limits_{i=0}^{\text{log}_2N} |E_i| 
= 
\sum\limits_{i=0}^{\text{log}_2N} 
O(2^{2i})
=
O(2 ^ {2\text{ log}_2N}) = O(N^2).
\end{equation}

For edge batching, $|E_{i}| = O(N \cdot N \cdot r_i^{d}) = O(N \cdot 2^{i})$. The worst-case
complexity for edge batching is
\begin{equation}
\label{eq:eb_worstcase_work}
\sum\limits_{i=0}^{\text{log}_2N} |E_i| 
=
\sum\limits_{i=0}^{\text{log}_2N} O(N \cdot 2^{i})
 = N \sum\limits_{i=0}^{\text{log}_2N} O(2^i)
= O(N^2).
\end{equation}

For hybrid batching, the complexity is analysed in two phases - the first phase where both vertices and edges are added,
and the second phase where only edges are added.
In the first phase, we have~$n_i = O(2^i)$ and
$|E_{i}| = O(n_i \cdot n_i \cdot r_i^{d}) = O(n_i)$. 
In the second phase, we have $n_i = N$ and 
$|E_{i}| = O(N \cdot N \cdot r_i^{d}) = O(N \cdot 2^{i})$. Thus the worst-case work 
complexity for hybrid batching is
\begin{multline}
\label{eq:hb_worstcase_work}
\sum\limits_{i=0}^{2 \ \text{log}_2N} |E_i| 
=
\sum\limits_{i=0}^{\text{log}_2N}  O(2^{i}) + 
\sum\limits_{i=0}^{\text{log}_2N} O(N \cdot 2^{i})
\\ = O(N) + O(N^2) = O(N^2).
\end{multline}
\hfill\\

Alternatively, if we set the number of batches to be $O(\frac{N}{k})$, so that each batch is of size $k$,
then the worst-case work for the batching strategies would
have higher complexity. For instance, for vertex batching in this setting,
$n_i = i \cdot k$ and $|E_i| = O(i^2 k^2)$. Therefore the total 
worst-case work is
\begin{multline}
\label{eq:vb_worstcase_work_poly}
\sum\limits_{i=0}^{\frac{N}{k}} |E_i| 
=
\sum\limits_{i=0}^{\frac{N}{k}} O(i^2 k^2)
=
k^2 \sum\limits_{i=0}^{\frac{N}{k}} O(i^2)
\\ =
k^2 O( (\frac{N}{k})^3)
= O(N^3).
\end{multline}
Similar results can be shown for edge and hybrid batching in this setting.

\subsubsection{Optimizations}
\hfill\\
In all the cases shown above, the worst-case work for any batching strategy
is still larger than searching $\calG$ directly.
Thus, we consider numerous optimizations to make the strategies more efficient in practice.
\\
\\
SEARCH TECHNIQUES

Each subgraph is searched using Lazy $\text{A}^{*}$~\citep{CPL14} with 
incremental rewiring as in $\text{LPA}^{*}$~\citep{koenig2004lifelong}. For details,
see the search algorithm used for a single batch of $\text{BIT}^{*}$~\citep{gammell2014batch}.
This lazy edge-centric variant of $\text{A}^{*}$ has been shown to outperform other shortest-path search techniques for problems
with expensive edge evaluations~\citep{dellin2016unifying}.
\\
\\
CACHING COLLISION CHECKS

Each time the collision-detector~$\calM$ is called on an edge, we store the 
ID of the edge along with the result using a hashing data structure. 
Subsequent calls for that specific edge are simply lookups in the hashing data structure which incur negligible running time.
Thus,~$\calM$ is called for each edge at most once.
\\
\\
SAMPLE PRUNING AND REJECTION

For anytime algorithms, once an initial solution is obtained, subsequent searches should be focused
on the subset of states that could potentially improve the solution. When the space $\mathcal{X}$ is
Euclidean, this so-called ``informed subset'' can be described 
by a prolate hyperspheroid~\citep{GSB14}. 
For our densification strategies, we prune away all existing vertices
(for all batching), and reject the newer vertices (for vertex and hybrid batching),
that fall outside the informed subset. 

Doing successive prunings due to intermediate solutions significantly reduces the average-case complexity
of future searches~\citep{gammell2014batch}, despite the extra time required to do so, 
which we account for in our benchmarking. Note that for Vertex and Hybrid Batching, 
which begin with only a few samples, samples in successive batches that are outside the 
current informed set can just be rejected. This is cheaper than pruning, which is 
required  for Edge Batching. Across all test cases, we noticed poorer performance when pruning was omitted.

In the presence of obstacles, the extent to which the complexity is reduced due to 
pruning is difficult to obtain analytically. In the assumption of free space, however, we can derive an 
interesting result for Edge Batching
and also Hybrid Batching, which motivates using this heuristic.

\begin{figure*}[th]
\captionsetup[subfigure]{justification=centering}
    \begin{subfigure}[b]{0.66\columnwidth}
    \centering
        \includegraphics[trim={600 300 300 150},clip,width=0.8\columnwidth]{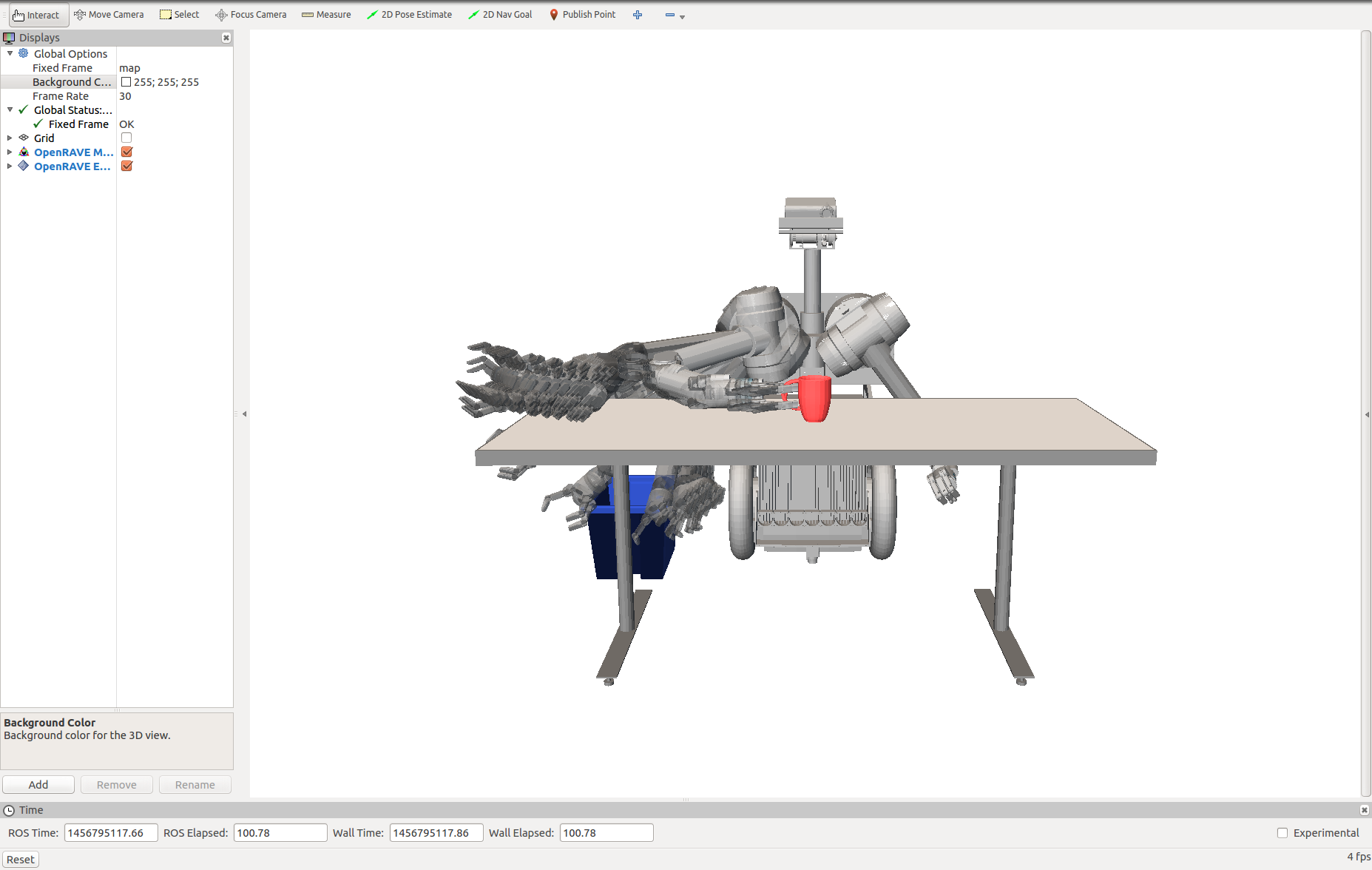}
        \caption{}
        \label{fig:anytimereel1}
    \end{subfigure}
    \begin{subfigure}[b]{0.66\columnwidth}
    \centering
        \includegraphics[trim={600 300 300 150},clip,width=0.8\columnwidth]{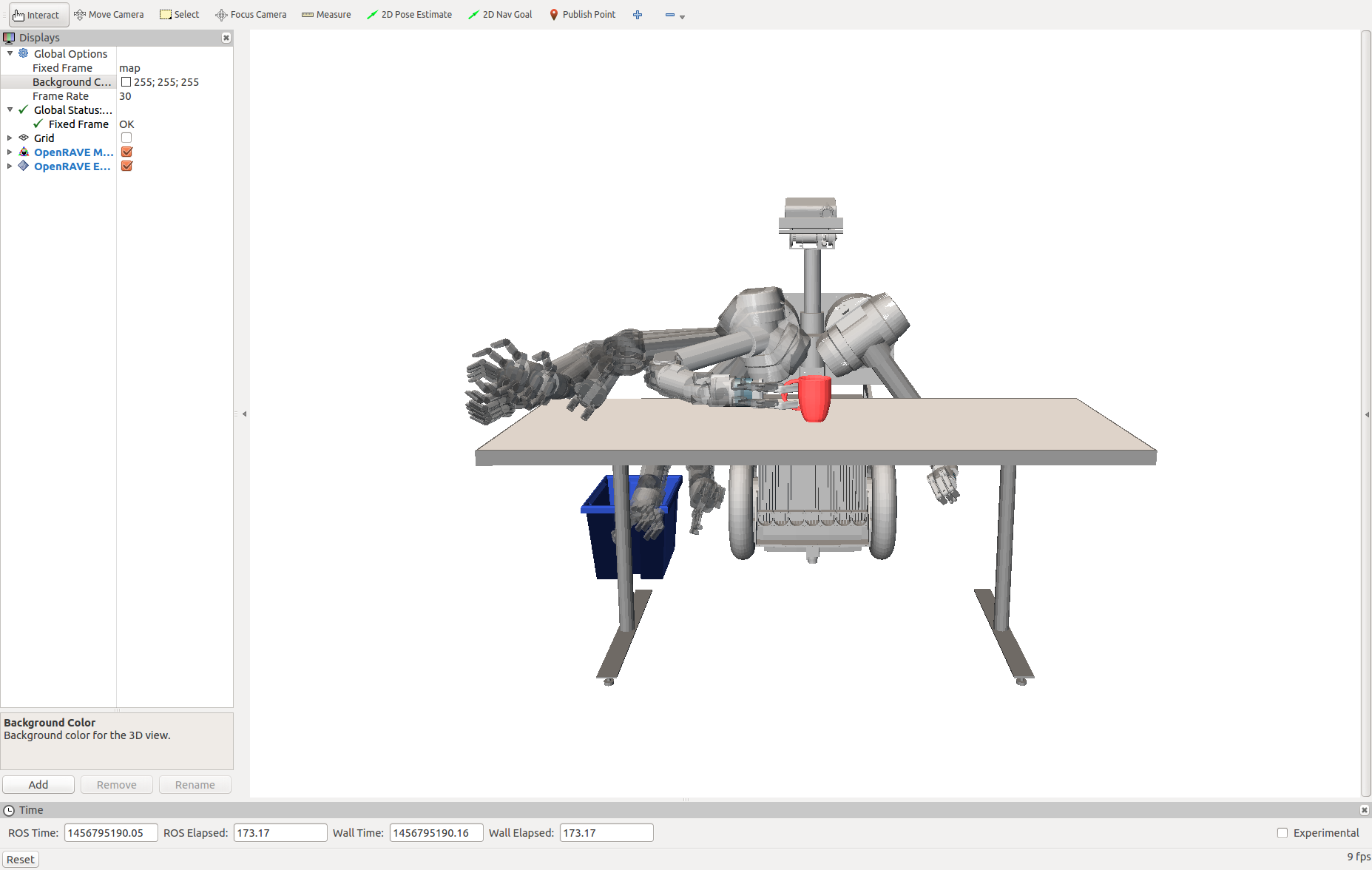}
        \caption{}
        \label{fig:anytimereel2}
    \end{subfigure}
    \begin{subfigure}[b]{0.66\columnwidth}
    \centering
        \includegraphics[trim={600 300 300 150},clip,width=0.8\columnwidth]{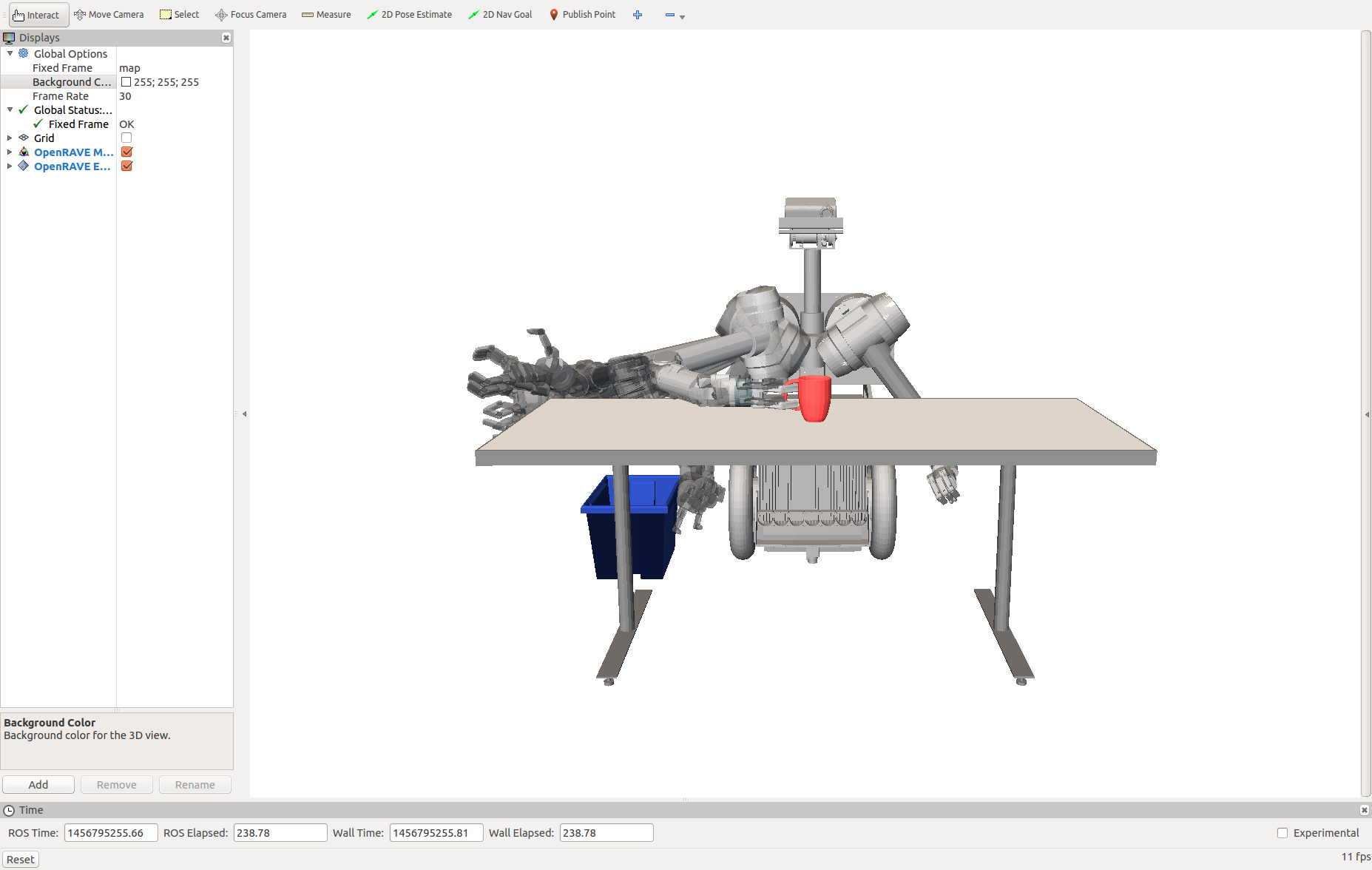}
        \caption{}
        \label{fig:anytimereel3}
    \end{subfigure}

    \begin{subfigure}[b]{0.99\columnwidth}
    \centering
        \includegraphics[trim={600 350 400 100},clip,width=0.7\columnwidth]{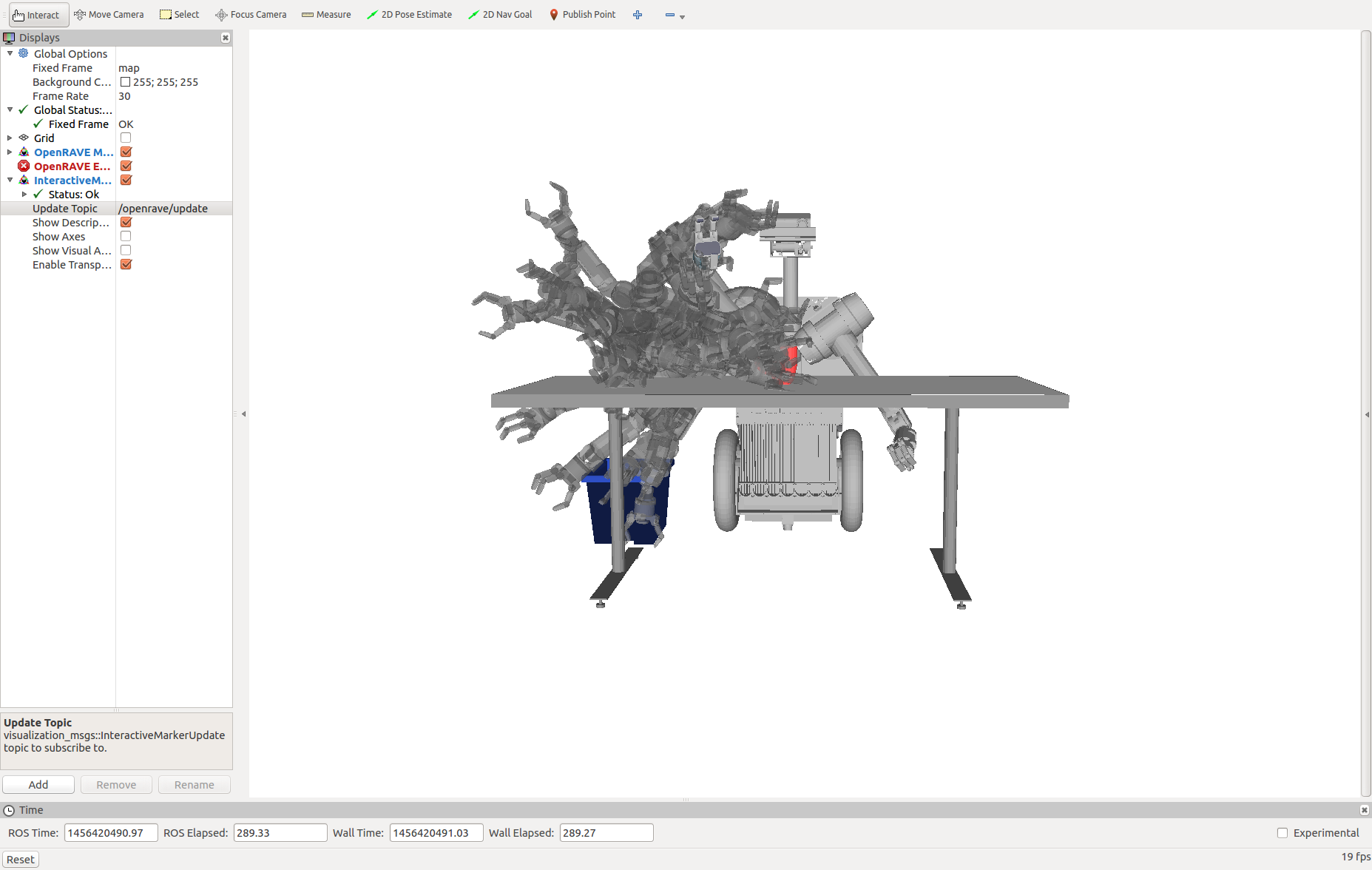}
        \caption{POMP}
        \label{fig:withchecks}
    \end{subfigure}
    \begin{subfigure}[b]{0.99\columnwidth}
    \centering
        \includegraphics[trim={600 350 400 100},clip,width=0.7\columnwidth]{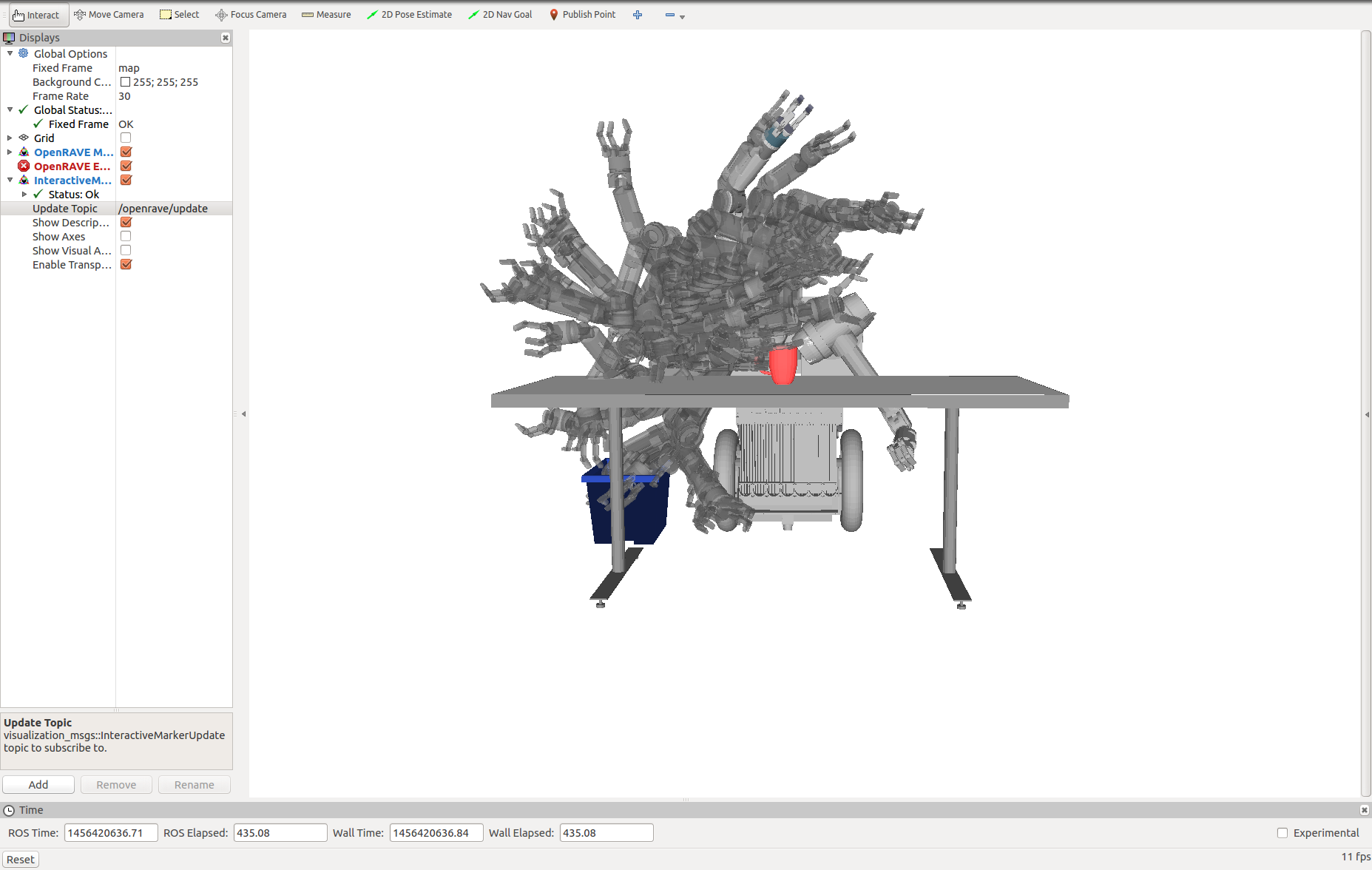}
        \caption{RRTConnect}
        \label{fig:rrtchecks}
    \end{subfigure}
    \caption{ The first row demonstrates the anytime behaviour of POMP. 
    A sequence of successively shorter trajectories is shown, from the first feasible path obtained in (\subref{fig:anytimereel1}) to the shortest feasible path in (\subref{fig:anytimereel3}). The second row shows the collision checks required by (\subref{fig:withchecks}) POMP and (\subref{fig:rrtchecks}) RRTConnect respectively to obtain a feasible path. POMP uses a belief model about the configuration space to guide the search for paths that are most likely to be free, and requires fewer collision checks than RRTConnect.}
    \label{fig:fig1}
\end{figure*}

\begin{theorem}
\label{th:ellipse}

Running edge batching or hybrid batching in an obstacle-free $d$-dimensional Euclidean space over
a roadmap constructed using a deterministic low-dispersion sequence with
$r_0 > 2D_{N}$ and $r_{i+1} = 2^{1/d} r_{i}$,
while 
using sample pruning and rejection, 
makes the worst-case complexity
of the total search, measured in edge evaluations, $O(N^{1+ 1/d})$.
\end{theorem}

\begin{proof}
We first prove the result for edge batching, as our corresponding result
for hybrid batching will be expressed in terms of this one.
\\
Let $\cbest{i}$ denote the cost of the solution obtained after $i$ iterations by our edge batching algorithm,
and $\cmin = \zeta(s_1, s_2)$ where $\cmin \leq \sqrt{d}$ denote the cost of the optimal solution.
Using \eref{eq:dispersion_suboptimality},
\begin{equation}
\label{eq:cbest}
\cbest{i} \leq \left( 1 + \varepsilon_i \right) \cmin,
\end{equation}
where 
$\varepsilon_i = \frac{2D_N}{r_i - 2D_N}$.
Using the parameters for edge batching,
\begin{equation}
\label{eq:epsdec}
\eps{i+1} = \frac{2D_N}{r_{i+1} - 2D_N} = \frac{2D_N}{2^{1/d} \cdot r_i - 2D_N} \leq \frac{\eps{i}}{2^{1/d}}.
\end{equation}
Let $\imax$ be the maximum number of iterations and recall that we want $\imax = O \left( \log_2 N \right)$. 
As $r_{\imax} \leq \sqrt{d}$, $r_i = r_0 \cdot 2^{i/d}$ and $r_0 = O \left( N^{-1/d} \right)$ we have
\begin{equation}
\label{eq:imax}
\imax 
\leq 
d \cdot \log_2 \frac{\sqrt{d}}{r_0} 
\leq 
\log_2 d^{\frac{d}{2}} \cdot r_0^{-d} 
\leq 
O \left( \log_2 N \right).
\end{equation}
Note that pruning away edges and vertices outside the subset does not change the bound provided in \eref{eq:cbest}.
To compute the actual number of edges considered at the $i^{\text{th}}$ iteration, we bound the volume of the prolate hyperspheriod $\mathcal{X}_{\cbest{i}}$ in $\mathbb{R}^{d}$ (see~\citep{gammell2014batch}) by,
\begin{equation}
	\label{eq:ellipse}
	\vol{\mathcal{X}_{\cbest{i}}} 
	= \frac{ \cbest{i} \left( \left(\cbest{i}\right)^2 - \cmin^2 \right)^{\frac{d-1}{2}} \xi_d}{2^d},
\end{equation}
where $\xi_{d}$ is the volume of an $\mathbb{R}^{d}$ unit-ball. 
Using \eref{eq:cbest} in \eref{eq:ellipse},
\begin{equation}
	\vol{\mathcal{X}_{\cbest{i}}} 
	\leq \eps{i}^{\frac{d-1}{2}} \left(1 + \eps{i} \right)\left(2 + \eps{i}\right)^{\frac{d-1}{2}}  \const_d, 
\end{equation}
where $\const_d = \xi_d \cdot \left( \cmin / 2\right)^d$
 is a constant.
Using \eref{eq:epsdec} we can bound the volume of the ellipse used at the $i$'th iteration, where $i \geq 1$,
\begin{equation}
\begin{aligned}
	\vol{\mathcal{X}_{\cbest{i}}} 	& 
	\leq 
	\eps{i}^{\frac{d-1}{2}} \left(1 + \eps{0} \right)\left(2 + \eps{0}\right)^{\frac{d-1}{2}} \const_d \\
									& 
	\leq 
	\eta^{-\frac{i(d-1)}{2}} \eps{0}^{\frac{d-1}{2}} \left(1 + \eps{0} \right)\left(2 + \eps{0}\right)^{\frac{d-1}{2}} \const_d \\
									& 
	\leq 
	2^{-\frac{i(d-1)}{2d}} \vol{\mathcal{X}_{\cbest{0}}}.
\end{aligned}
\end{equation}
Furthermore, we choose  $r_0$ such that $\vol{\mathcal{X}_{\cbest{0}}} \leq \vol{\mathcal{X}}$.
Now, the number of vertices in $\spacebest$ can be bounded by,
\begin{equation}
n_{i+1} = \frac{ \vol{\spacebest} }{ \vol{\mathcal{X}} } N \leq 2^{-\frac{i(d-1)}{2d}} N.
\end{equation}
We measure the amount of work done by the search at iteration $i$ using $|E_i|$, the number of edges considered. Thus, 
\begin{multline}
|E_i| = O\left(n_i^2 \cdot r_i^d \right) = O\left(N^2 \cdot 2^{-\frac{i(d-1)}{d}} \left(r_0 \cdot 2^{\frac{i}{d}} \right)^d \right) \\ = O\left(N \cdot 2^{\frac{i}{d}} \right).
\end{multline}
Finally, the total work done by the search over all iterations is
\begin{equation}
O \left( \sum \limits_{i=0}^{\log_2 N} N \cdot 2^{i/d} \right)
=
O \left( N \sum \limits_{i=0}^{\log_2 N} 2^{i/d} \right)
= 
O \left( N^{1+1/d} \right).
\end{equation}
\hfill\\
\\
Now we briefly outline the corresponding proof for hybrid batching. Hybrid batching proceeds in two 
phases, one in which both vertices and edges
are added and another in which only edges are added.
\\
In phase 1, $n_{i+1} = 2 \cdot n_{i}$ and $r_{i+1} = 3 \cdot \left(n_{i+1}\right)^{-1/d} = 2^{-1/d} \cdot r_{i}$.
Therefore we have,
\begin{multline}
\eps{i+1} = \frac{2 D_{n_{i+1}}}{r_{i+1} - 2 D_{n_{i+1}}}
= \frac{ 2^{-1/d} \cdot 2 D_{n_i} } { 2^{-1/d} \left( r_i - 2 D_{n_i} \right) }
\\ = \frac{2 D_{n_i}}{r_i - 2 D_{n_i}} = \eps{i}.
\end{multline}
\\
Therefore the ellipse does not shrink between successive solutions.
Thus, in phase 1, the total work done is as follows:
\begin{multline}
O \left(\sum\limits_{i=0}^{\log_2 N} |E_i| \right)
= O \left( \sum\limits_{i=0}^{\log_2 N} n_i^2 \cdot r_i^d \right)
\\ = O \left( \sum\limits_{i=0}^{\log_2 N} 2^{i} \right)
= O(N).
\end{multline}
Furthermore, in phase 2, hybrid batching proceeds exactly as edge batching does,
so its worst-case work here is $O(N^{1 + 1/d})$. Therefore, the total work
done over all iterations for hybrid batching is also
\begin{equation}
O\left(N\right) + O\left(N^{1 + 1/d}\right) = O\left(N^{1 + 1/d}\right).
\end{equation}
\end{proof}
\noindent
A similar result for vertex batching cannot be obtained simply because in the obstacle free case, vertex 
batching would find a solution immediately, rendering this analysis trivial.
\\
\\
For the densification results that we discuss in \sref{sec:experiments-densification}, our implementations of
edge batching, vertex batching and hybrid batching use the various parameters and optimizations
discussed above. Together, they contribute to a significant improvement in runtime performance
of the densification strategies.


\section{Searching over configuration-space beliefs}
\label{sec:cspacebelief}

We now discuss in detail our second key idea - an algorithm for anytime planning 
on roadmaps that uses a model of the configuration space to search for successively 
shorter paths that are likely to be feasible~\citep{choudhury2016pareto}. Our algorithm, which we call
Pareto-Optimal Motion Planner (POMP), searches for paths that are Pareto-optimal in 
path length and collision probability. As mentioned in \sref{sec:intro-ideas}, we propose 
using POMP for searching each individual subgraph generated by the specific densification strategy used.
In this section, however, we will discuss its properties and behaviour independent
of any densification framework.
Therefore, for this specific section,
we do not assume the roadmaps are large and dense, as we allow the densification
strategy to handle that constraint.

\begin{figure*}[t]
\captionsetup[subfigure]{justification=centering}
	\begin{subfigure}[b]{0.5\columnwidth}
	\centering
		\includegraphics[trim=10 20 1 5,clip,width=\columnwidth]{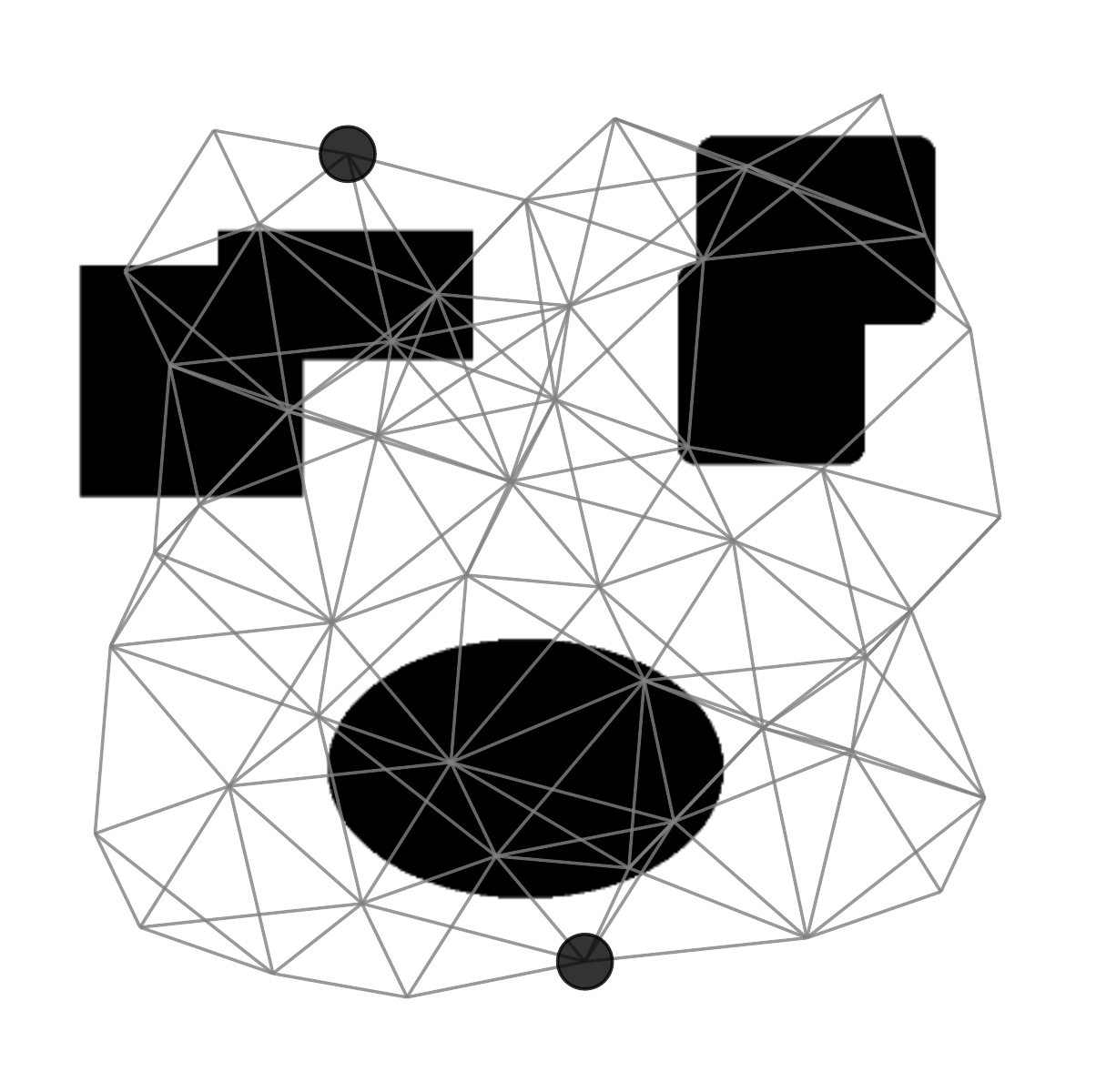}
		\caption{Problem Environment}
		\label{fig:with-illus-problem}
	\end{subfigure}
	\begin{subfigure}[b]{0.5\columnwidth}
	\centering
		\includegraphics[width=0.9\columnwidth]{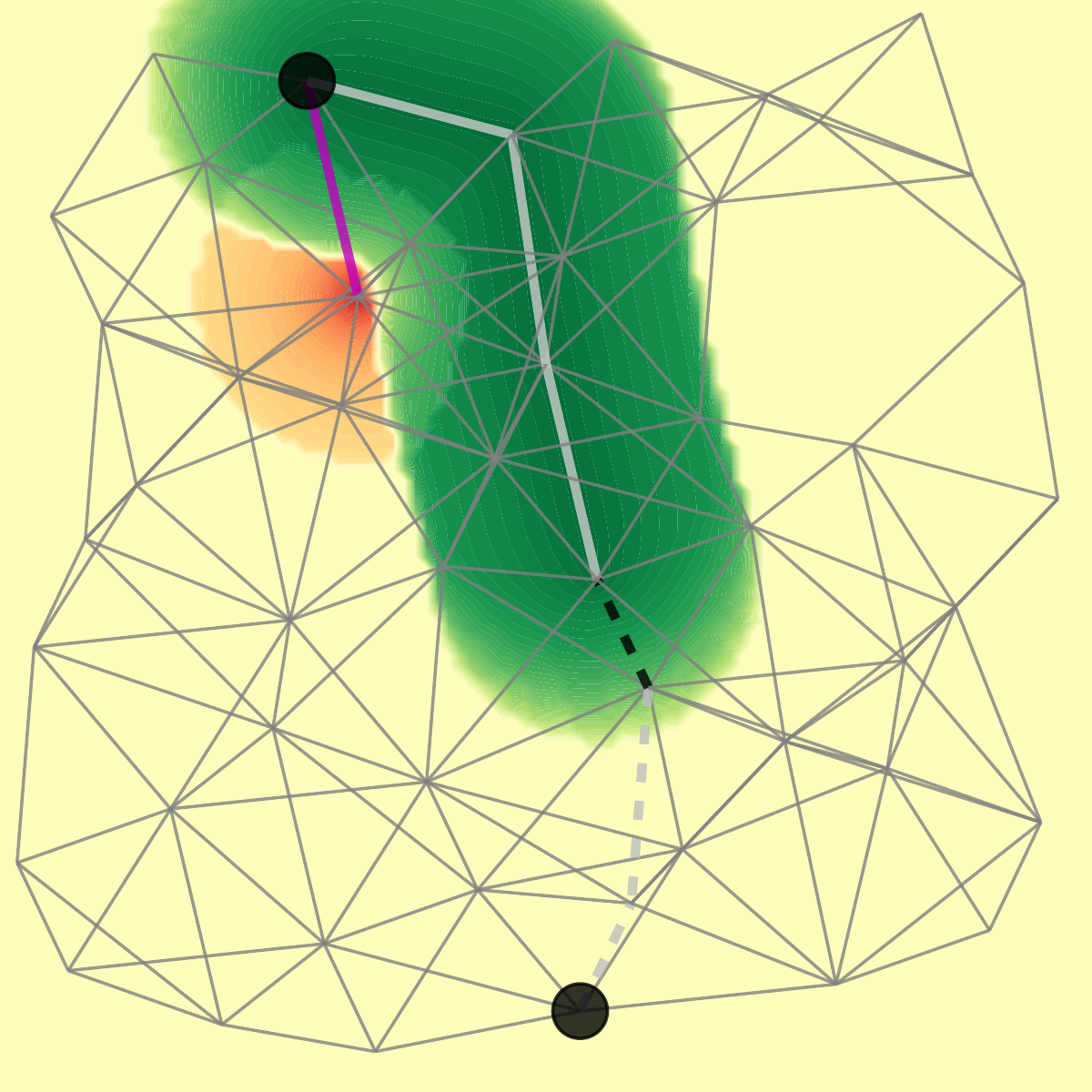}
		\caption{Model - Iteration 4}
		\label{fig:with-illus-path4}
	\end{subfigure}
	\begin{subfigure}[b]{0.5\columnwidth}
	\centering
		\includegraphics[width=0.9\columnwidth]{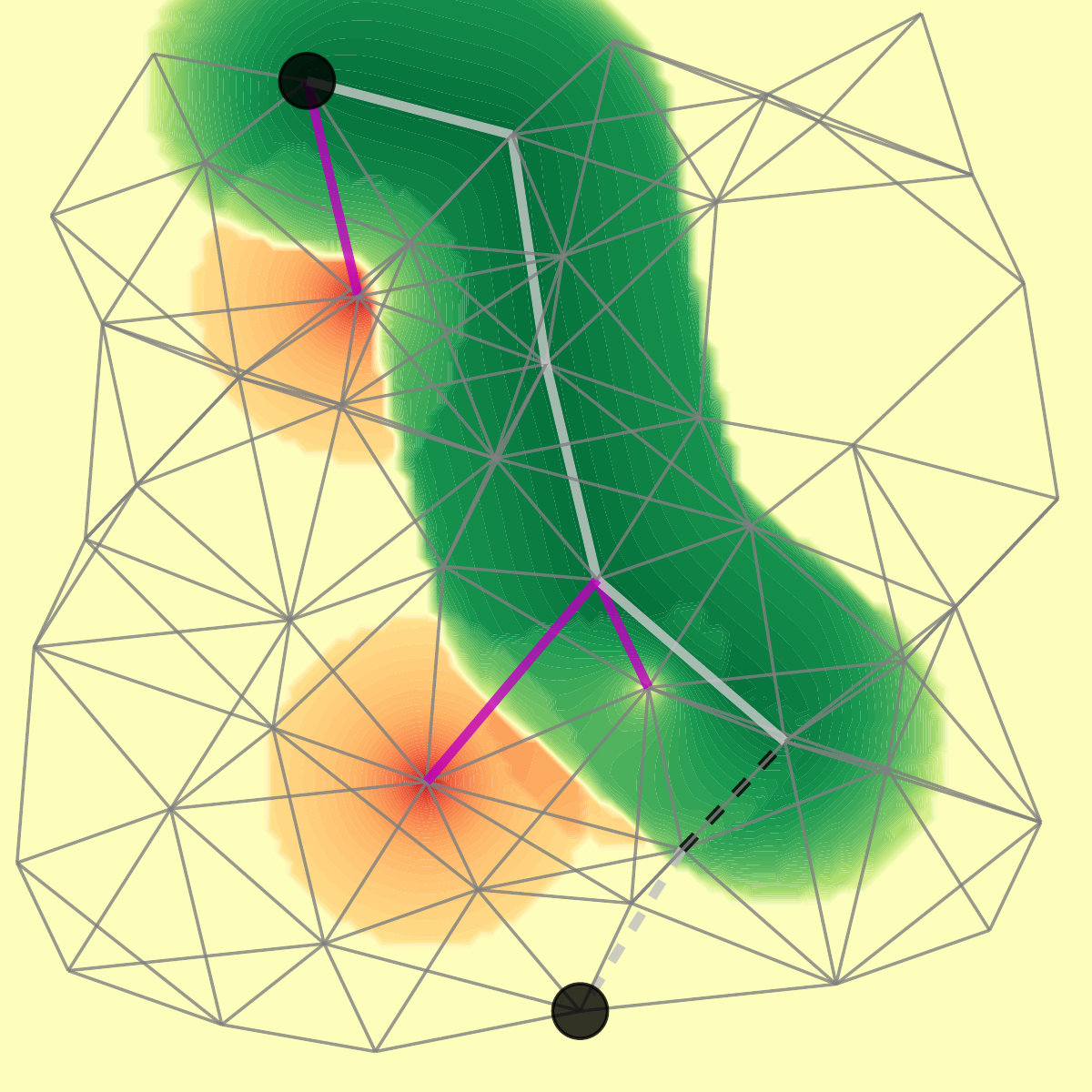}
		\caption{Model - Iteration 7}
		\label{fig:with-illus-path9}
	\end{subfigure}
	\begin{subfigure}[b]{0.5\columnwidth}
	\centering
		\includegraphics[width=0.9\columnwidth]{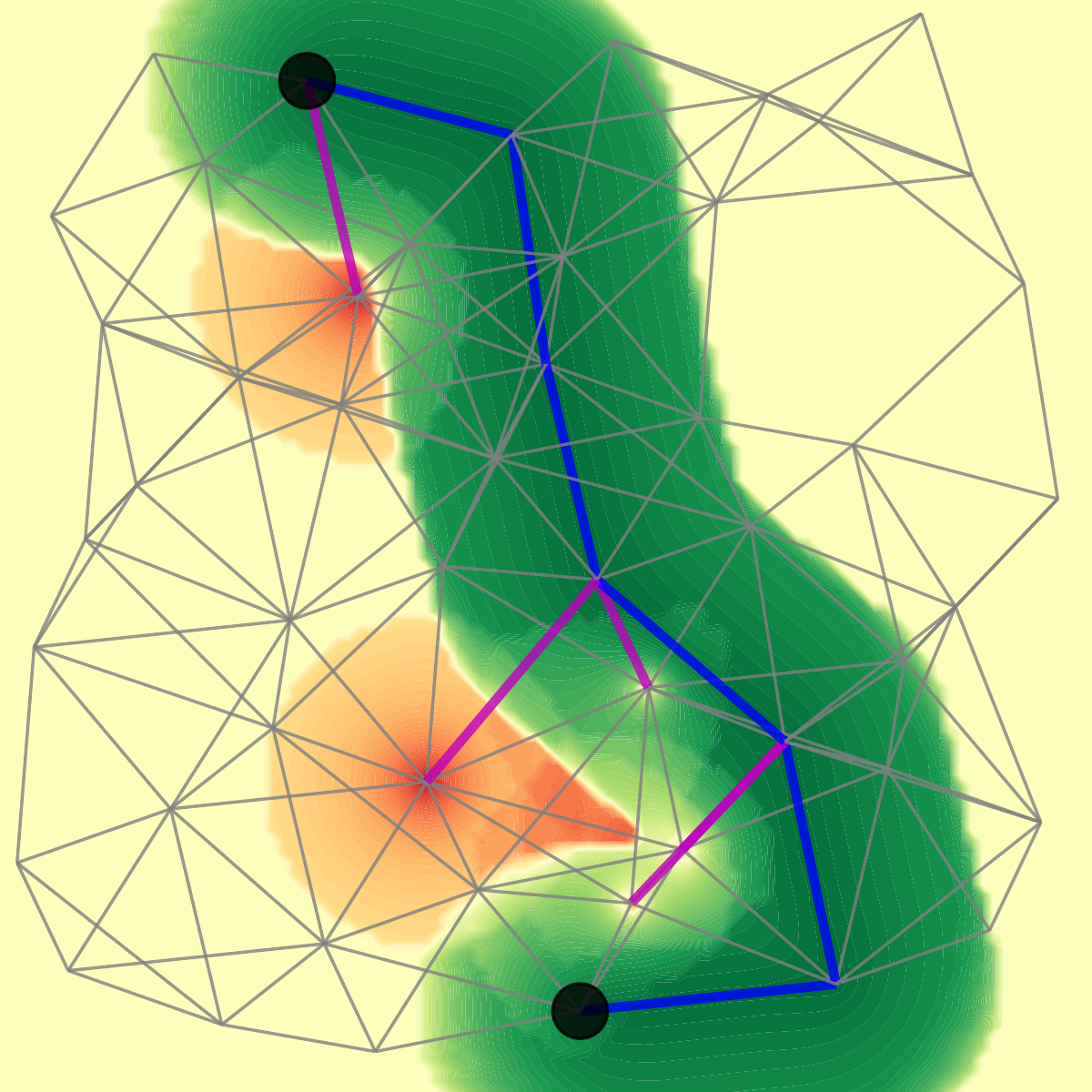}
		\caption{Model - First Path}
		\label{fig:with-illus-path12}
	\end{subfigure}

	\begin{subfigure}[b]{0.5\columnwidth}
	\centering
		\includegraphics[width=0.9\columnwidth]{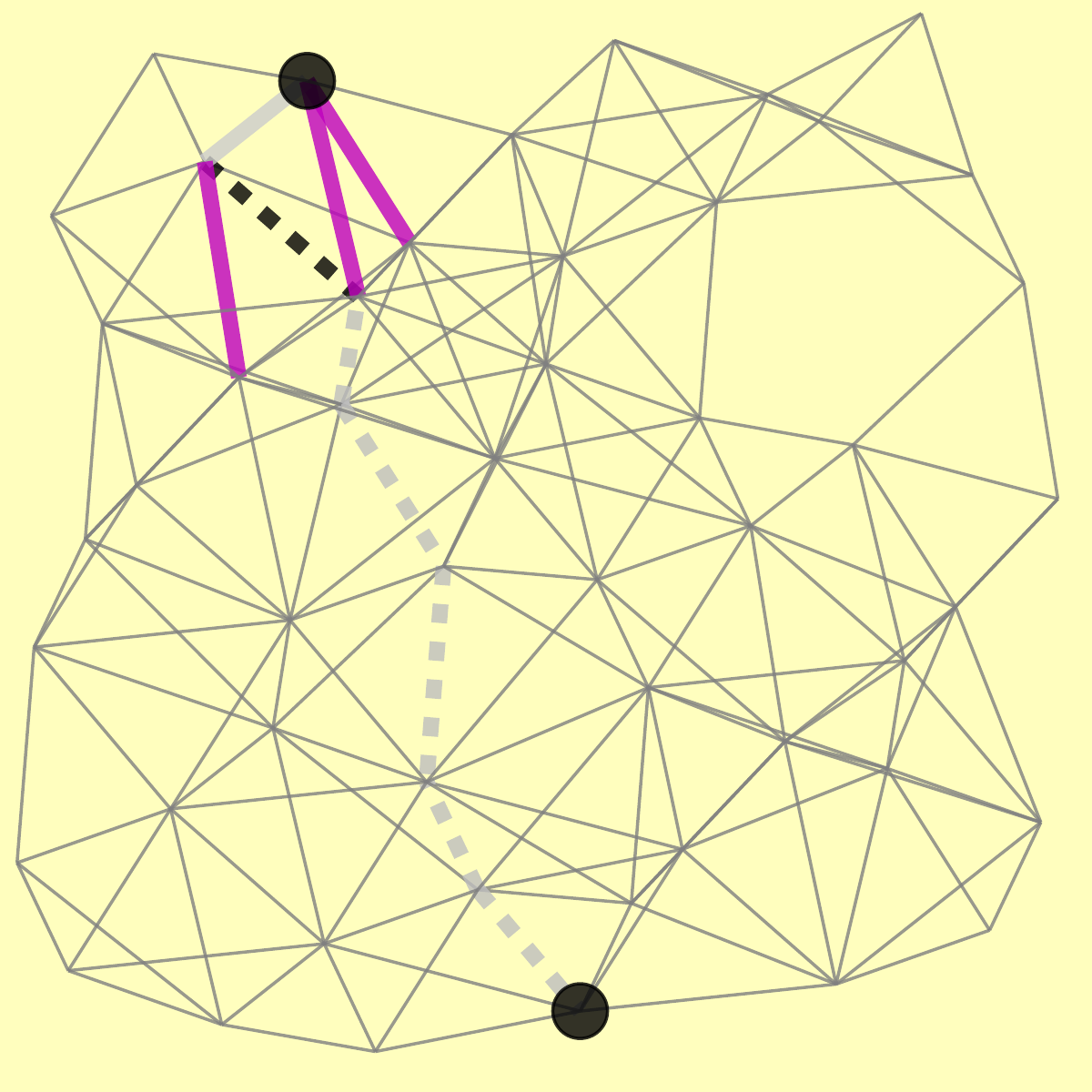}
		\caption{No Model - Iteration 4}
		\label{fig:with-illus-problem}
	\end{subfigure}
	\begin{subfigure}[b]{0.5\columnwidth}
	\centering
		\includegraphics[width=0.9\columnwidth]{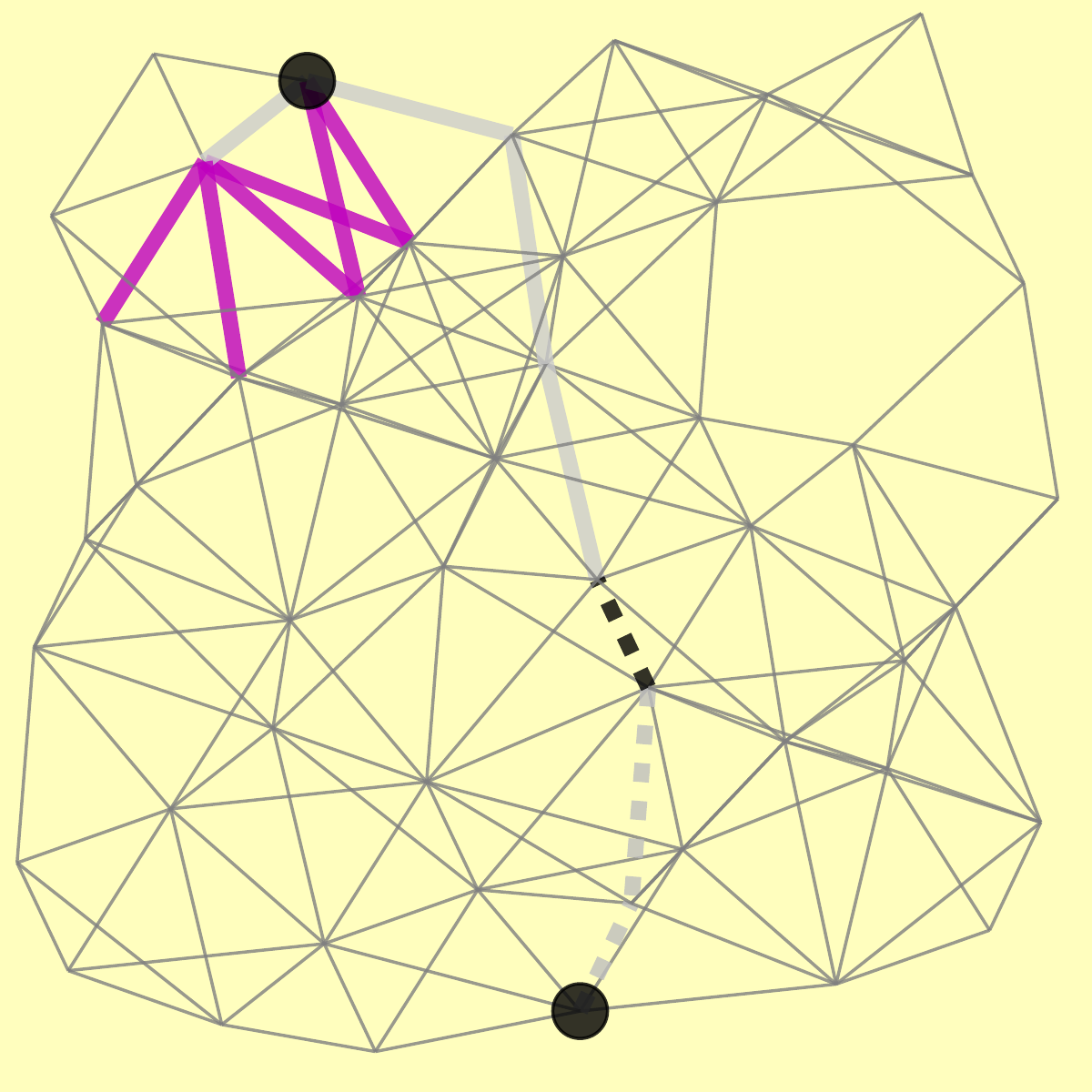}
		\caption{No Model - Iteration 10}
		\label{fig:with-illus-path4}
	\end{subfigure}
	\begin{subfigure}[b]{0.5\columnwidth}
	\centering
		\includegraphics[width=0.9\columnwidth]{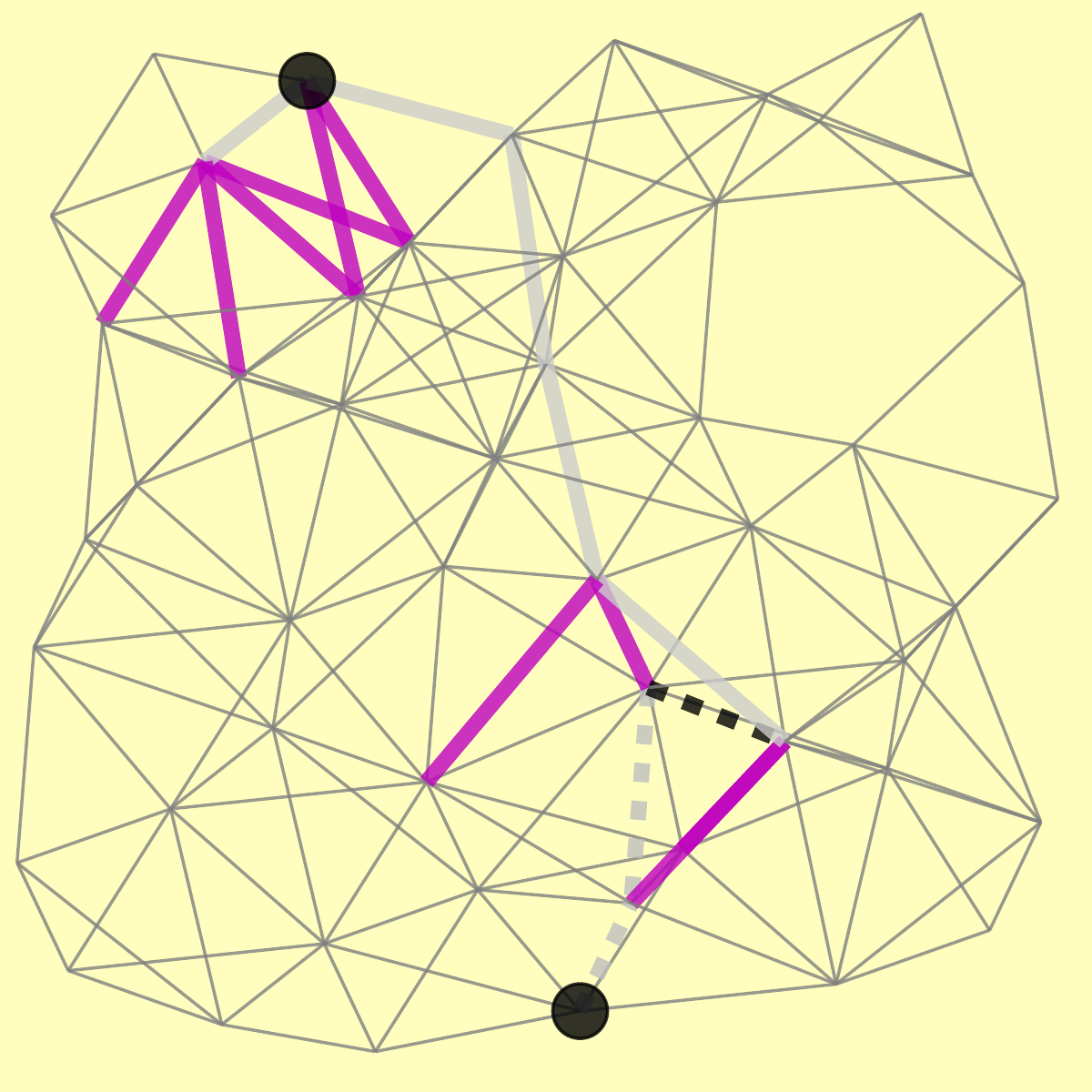}
		\caption{No Model - Iteration 15}
		\label{fig:with-illus-path9}
	\end{subfigure}
	\begin{subfigure}[b]{0.5\columnwidth}
	\centering
		\includegraphics[width=0.9\columnwidth]{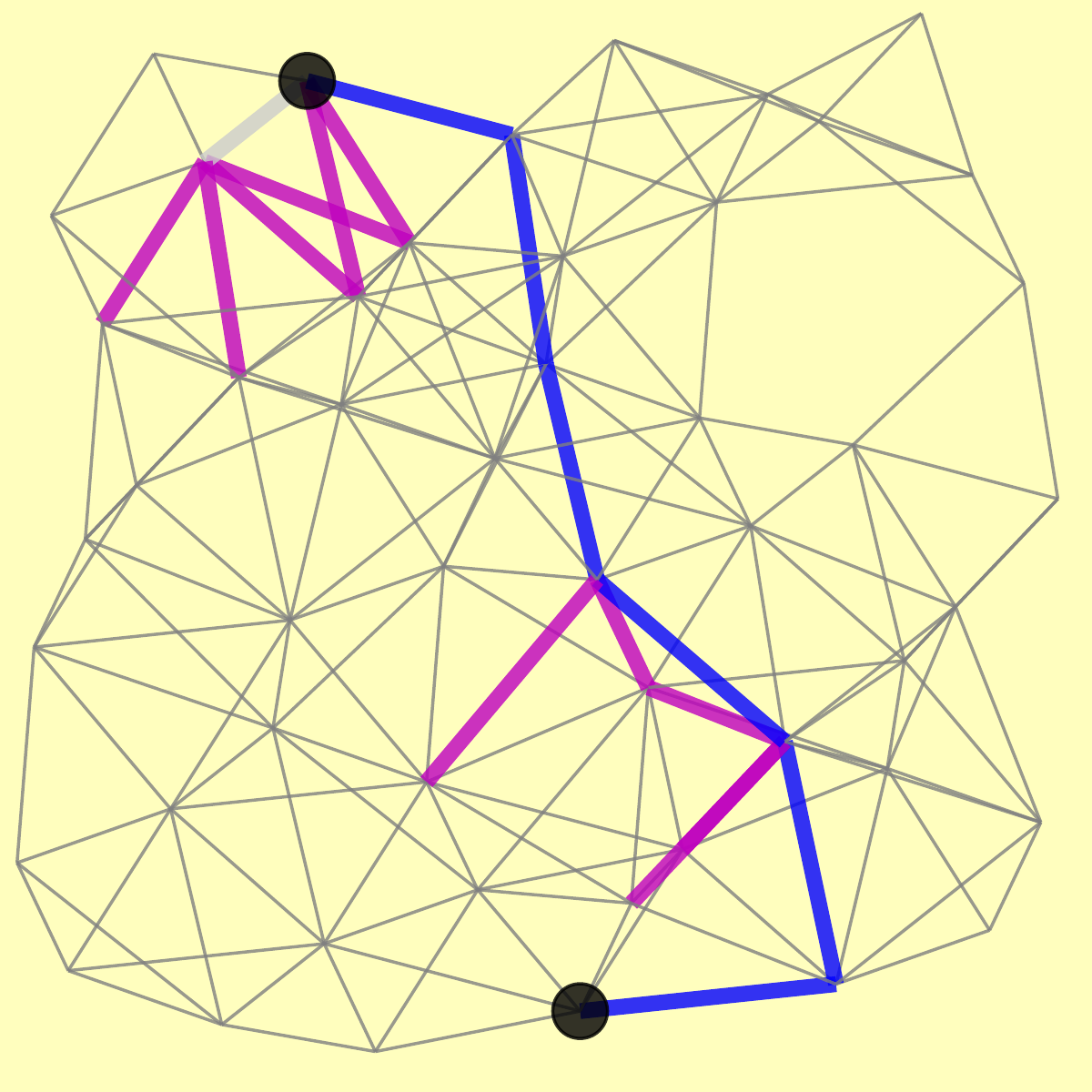}
		\caption{No Model - First Path}
		\label{fig:with-illus-path12}
	\end{subfigure}
	\caption{An illustration of the benefit of using configuration space beliefs. The upper and lower rows show runs of POMP on a 2D planning problem, with some finite model radius and zero radius respectively. The heatmap represents the belief model with green representing the belief of being free, and orange the belief of being in collision. The thin grey edges are unevaluated. The dashed edges are being evaluated. Thick grey edges are evaluated free, and thick magenta edges are evaluated in collision. The blue edges in the rightmost pictures represent the first feasible path in each case. Using the belief model, POMP requires 10 evaluations for the first feasible path, while without the model, it requires 18 evaluations.}
	\label{fig:illus_2d}
\end{figure*} 

\begin{figure}
\centering
	\includegraphics[height = 4.0cm,keepaspectratio]{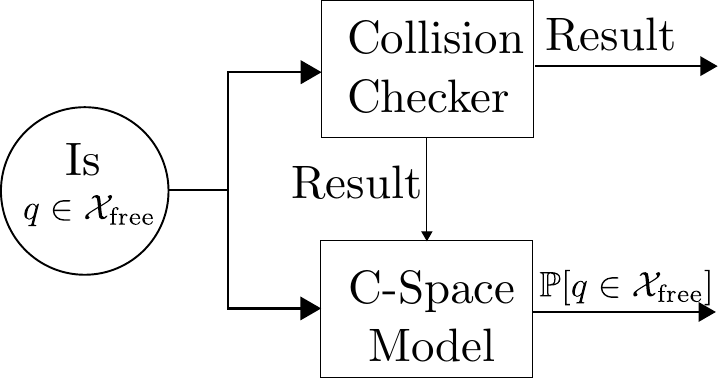}
	\caption[Schematic of the configuration space belief model]{
	 	The configuration space belief model is updated with the results of collision checks
	 	and is queried to obtain the collision probability of an unknown configuration.}
	\label{fig:cspacemodel}
\end{figure}

\subsection{Configuration Space Beliefs}
\label{sec:cspacebelief-beliefs}

The collision detector $\calM$ (\sref{sec:problem}) can be thought of as a perfect yet expensive \emph{binary classifier}, 
deciding if a queried configuration is in $\Cobs$ or $\Cfree$. Additionally, we consider an inexpensive 
but uncertain model $\Phi$ that takes in a configuration and outputs its belief that the query 
is collision-free, represented as  $\rho: \calX \mapsto [0,1]$. We can build and update this model 
using any black-box learner. Given a query $q$, we now have the choice of either 
inexpensively evaluating $\rho(q)$ from the model $\Phi$, or expensively querying the collision checker. 
A representation is shown in  \figref{fig:cspacemodel}. A number of different models exist in the literature~\citep{burns2005sampling,knepper2012real,pan2013faster} which could be used. We are concerned not with the specific model implementation
but with how to use this model to guide the search for paths on the roadmap.

\subsection{Edge Weights}
\label{sec:cspacebelief-wts}

Our POMP algorithm searches for paths that are Pareto-optimal in two criteria. The path criterion functions are  
obtained from two edge weight functions, which we define here.
The first is $w_l : E \rightarrow [0,\infty)$, which  measures the length of an edge based on our 
distance metric $\zeta$ on $\calX$. 
For an edge $(u,v)$ that is unevaluated or has been found (by $\calM$) to be collision-free, the weight is (optimistically) set to $\zeta(u,v)$.
For edges that have been evaluated to be in collision, the weight is set to $\infty$. 
The path length is obtained as $L(\gamma) = \sum\limits_{e \in \gamma} w_l(e)$.

The second weight function is $w_{m} : E \rightarrow [0,\infty)$, and it relates to the probability 
of the edge to be collision-free, based on our model $\Phi$. Specifically, $w_{m}(e) = -\text{log }\rho(e)$, 
where $\rho(e)$ is the probability of $e$ to be collision-free. We define $w_{m}(e)$ as a negative log to avoid numerical underflow
issues with $\rho(e)$, which is typically computed as $\rho(e) = \prod\limits_{q \in e}\rho(q)$, where
$\rho(q)$ is always less than 1, and edges have several embedded configurations.
An edge evaluated (by $\calM$) to be collision-free has $w_m(e) =0$ and 
a known-colliding edge has $w_m(e) = \infty$. If we assume conditional independence of configurations 
given the edge, we can express the log-probability of a path being in collision, $M(\gamma)$, in the same summation form as $L(\gamma)$:

\begin{multline}
\label{eq:collision_measure}
M(\gamma) = -\text{log } \mathbb{P}(\gamma  \in \Cfree) = -\mathrm{log} \prod \limits_{e \in \gamma} \rho(e) \\ =
\sum \limits_{e \in \gamma} -\text{log } \rho(e) =  \sum \limits_{e \in \gamma} w_{m}(e).
\end{multline}
We refer to $M$ as the \emph{collision measure} of the path. Having both $M(\gamma)$ and $L(\gamma)$ be additive over edges enables efficient searches over the roadmap, as we will describe subsequently.

\subsection{Weight Constrained Shortest Path}

Our first objective is to obtain some initial feasible path quickly, irrespective of path length. We search for paths that are most likely to be free according to $\Phi$. Once we have a feasible path, we search only for paths of shorter length, based on their likelihood of being free. Specifically, we want to search over paths most likely to be free, with a length lower than some upper bound, where the bound reduces over time, with each successive feasible solution. 
One way to represent this is by repeatedly solving the problem

\begin{equation}
\label{eq:wcsp}
\begin{aligned}
& \underset{\gamma}{\text{argmin}}
& & M(\gamma)\\
& \text{subject to}
& & L(\gamma) < L(\hat{\gamma}_i)
\end{aligned}
\end{equation}
and subsequently evaluating the returned solution for feasibility. The initial bound is $\infty$, after which the first bound is $L(\hat{\gamma}_0)$, where $\hat{\gamma}_0$ is the first feasible solution, and then $L(\hat{\gamma}_1)$ and so on. Therefore, the first iteration of the problem is an unconstrained shortest path problem. For a particular finite upper bound, however, this problem is an instance of the Weight Constrained Shortest Path (WCSP) problem.

\begin{figure}
\centering
	\includegraphics[width=0.9\columnwidth]{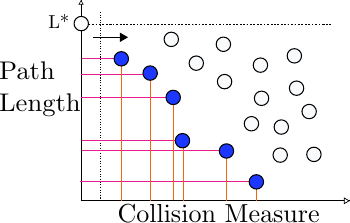}
	\caption
	{
		The LazyWCSP algorithm performs a horizontal sweep on the path measure plane to select the left-most point. If each point chosen is feasible (collision-free), and there are no updates to the model $\Phi$ between searches, this sweeps out the Pareto frontier of valid points, with respect to initial weights.
	}
	\label{fig:wcsp-sweep}
\end{figure}

The above is one way we could have formulated the problem, based on our intuition. A closer look will reveal, however, that this is not the most appropriate. Therefore \textbf{we will use an alternative formulation} to \eref{eq:wcsp}, which we will explain shortly. We visualize paths on a 2D plane in terms of their two weights - the path length $L$ and the collision measure $M$. Each path is a point on this \emph{path measure plane}, as shown in \figref{fig:wcsp-sweep}.

For bicriteria problems such as the one we are facing, a point (path) is \emph{strictly dominated} by another point if it is worse off in both criteria than the latter. For instance, if both criteria are to be minimized, and if there are two points $\tau, \tau^{\prime}$ such that $L(\tau) < L(\tau^{\prime})$ and $M(\tau) < M(\tau^{\prime})$, then $\tau$ strictly dominates $\tau^{\prime}$, i.e. $\tau \succ \tau^{\prime}$. A point is \emph{Pareto optimal} if it is not strictly dominated by any other point. The set of Pareto optimal points is known as the \emph{Pareto frontier}.

Consider a simple approach that uses WCSP and evaluates paths lazily, updating the model $\Phi$ after each search and solving the updated problem defined in \eref{eq:wcsp}. Let us call this the \emph{LazyWCSP} method. It repeatedly performs a horizontal sweep over the points in the plane under some horizontal line, the upper length bound (initially there is no line as the bound is $\infty$). It selects the left-most one (with minimum collision measure $M$) to evaluate. If the path is infeasible, it is moved infinitely to the right, and if feasible, it is moved left onto the $y$-axis, with the collision measure set to 0, and is flagged as the current best solution $\gamma_i$. The upper bound $L^{*}$ is now $L(\gamma_i)$, represented by a horizontal line. The collision testing of previously unknown configurations updates the model $\Phi$ with the configuration and its collision status, which in turn updates the $x$-coordinate of certain points.

\begin{figure}
\centering
\captionsetup[subfigure]{justification=centering}
\begin{subfigure}[b]{0.99\columnwidth}
\centering
	\includegraphics[width=0.9\columnwidth]{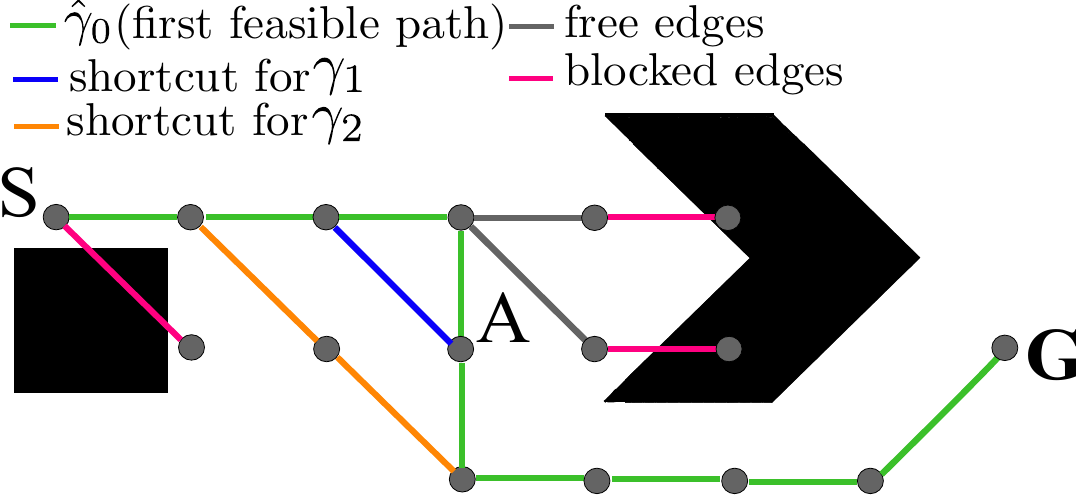}
	\caption{}
	\label{fig:wcspp-toy}
\end{subfigure}

\begin{subfigure}[b]{0.49\columnwidth}
\centering
	\includegraphics[width=0.99\columnwidth]{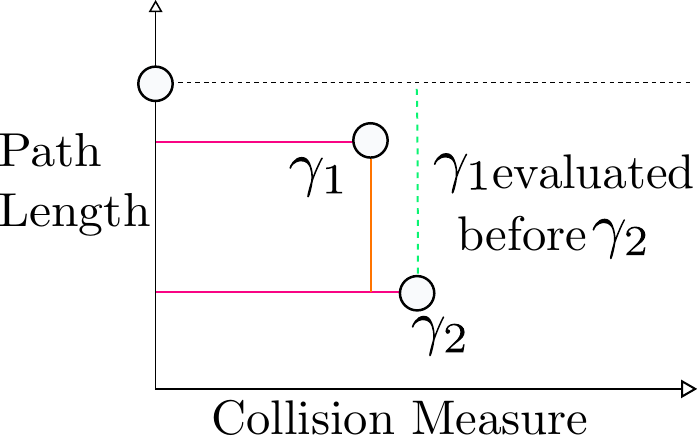}
	\caption{}
	\label{fig:wcspp-ex1}
\end{subfigure}
\begin{subfigure}[b]{0.49\columnwidth}
\centering
	\includegraphics[width=0.9\columnwidth]{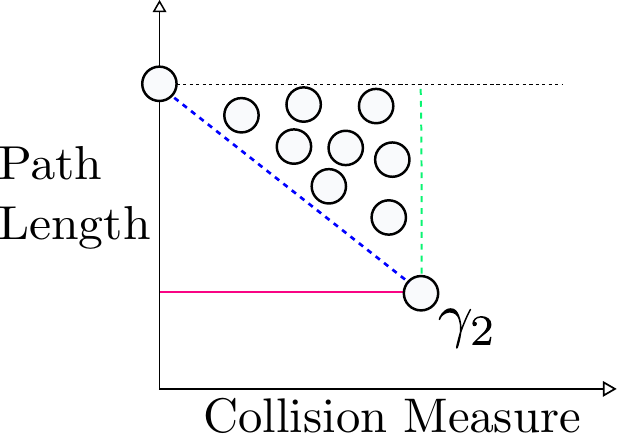}
	\caption{}
	\label{fig:wcspp-ex2}
\end{subfigure}
\caption{Two problematic scenarios for using LazyWCSP. A toy case is shown in (\subref{fig:wcspp-toy}), along with two possible corresponding path measure plots. In (\subref{fig:wcspp-ex1}), a small decrement in collision measure for $\gamma_1$ is prioritized over a larger decrement in path length for $\gamma_2$. In (\subref{fig:wcspp-ex2}), all points above the blue line and to the left of $\gamma_2$ are evaluated before the more promising $\gamma_2$. These points correspond to several paths through A.}
\label{fig:wcspp-exs}
\end{figure}

\begin{figure*}[t]
\centering
\captionsetup[subfigure]{justification=centering}
	\begin{subfigure}[b]{0.5\columnwidth}
	\centering
		\includegraphics[width=0.9\columnwidth]{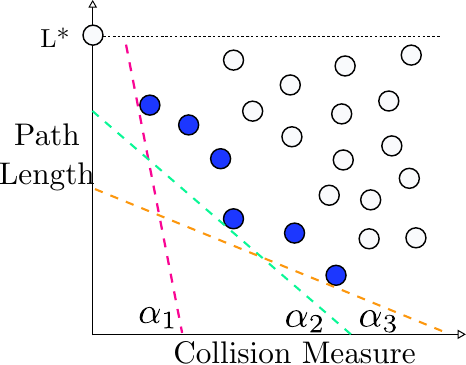}
		\caption{}
		\label{fig:alg-hulltrace}
	\end{subfigure}
	\begin{subfigure}[b]{0.5\columnwidth}
	\centering
		\includegraphics[width=0.9\columnwidth]{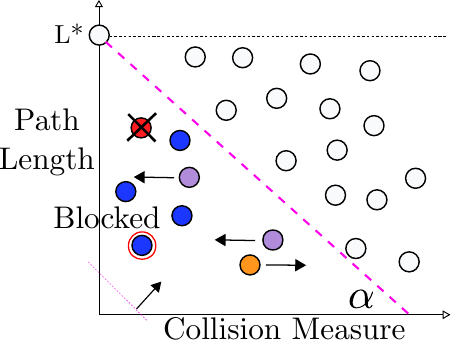}
		\caption{}
		\label{fig:alg-diagsweep}
	\end{subfigure}
	\begin{subfigure}[b]{0.5\columnwidth}
	\centering
		\includegraphics[width=0.9\columnwidth]{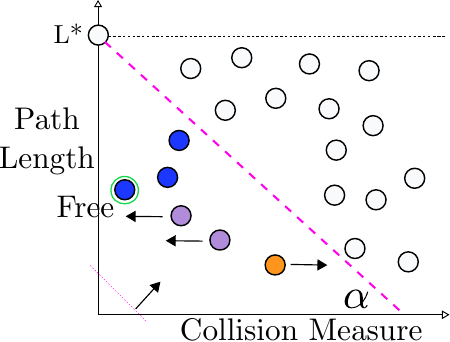}
		\caption{}
		\label{fig:alg-convswitch}
	\end{subfigure}
	\begin{subfigure}[b]{0.5\columnwidth}
	\centering
		\includegraphics[width=0.9\columnwidth]{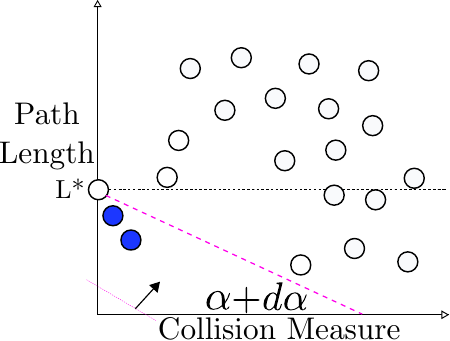}
		\caption{}
		\label{fig:alg-afterswitch}
	\end{subfigure}
	\caption{The key insights to our algorithm POMP. (\protect\subref{fig:alg-hulltrace}) Under the assumption of no model updates, 
 	POMP selects candidates paths along the convex hull of the Pareto frontier for different values of $\alpha$. The actual behaviour of POMP, without that assumption, is shown through (\protect\subref{fig:alg-diagsweep}) - (\protect\subref{fig:alg-afterswitch}). The diagonal sweep corresponds to searching with $J^{\alpha}$, and the first point found in the sweep corresponds to the path that minimizes $J^{\alpha}(\gamma)$. When a path is evaluated, other paths may have their $M$-values \protect\tikz{\protect\node[fill=orange,draw=black]{};}\; increased or \protect\tikz{\protect\node[fill=mypurple,draw=black]{};}\; decreased, and they may be \protect\tikz{\protect\node[fill=red,draw=black]{};}\; deleted if they share infeasible edges.}
	\label{fig:algfig}
\end{figure*}

\subsection{POMP}
\label{sec:cspacebelief-pomp}

The performance of LazyWCSP (and related algorithms in this setting) is affected by the frequency of 
updates to the model $\Phi$. Updating $\Phi$ after each candidate path evaluation makes the searches
more informed but is more computationally expensive than, say, updating $\Phi$ only after a new 
feasible solution is found. Given the specific belief model our algorithm uses, we are able to perform model updates efficiently (\ref{sec:cspacebelief-algorithm}) and so we update $\Phi$ after each candidate path evaluation.

Our characterization of the behaviour of LazyWCSP (and related algorithms), however, is easier if 
we update the model $\Phi$ only when a new feasible solution is found. Under this assumption, between any two successive feasible solutions, LazyWCSP traces out the Pareto frontier of the feasible paths with respect to their initial coordinates, as shown in \figref{fig:wcsp-sweep}. In both of the examples in \figref{fig:wcspp-exs}, this defers the evaluation of the more promising $\gamma_2$.

Therefore, instead of doing LazyWCSP, we explicitly control the tradeoff between the two weights by defining the objective function for paths as a convex combination of the two weights,

\begin{equation}
\label{eq:weight_comb}
J^{\alpha}(\gamma) = \alpha L(\gamma) + (1-\alpha) M(\gamma) \ , \ \alpha \in [0,1].
\end{equation}

Minimizing $J^{\alpha}$ for various choices of $\alpha$, traces out the convex hull of the Pareto frontier of the initial coordinates, as shown in \figref{fig:alg-hulltrace}. This is the key idea behind our algorithm POMP, or Pareto-Optimal Motion Planner. The $\alpha$ parameter represents the tradeoff between the weights. Also, optimizing over $J^{\alpha}$ implicitly satisfies the constraint on $L$. If the current solution is $\gamma_i$, then for any path $\gamma^{\prime}$

\begin{align*}
& J^{\alpha}(\gamma^{\prime}) < J^{\alpha}(\gamma_i) \\
&\implies \alpha L(\gamma^{\prime}) + (1-\alpha) M(\gamma^{\prime}) < \alpha L(\gamma_i) \ \ [ \text{as} \  M(\gamma_i) = 0]\\
&\implies L(\gamma^{\prime}) < L(\gamma_i).
\end{align*}

Therefore our problem can be restated as
\begin{equation}
\label{eq:pomp}
\begin{aligned}
& \underset{\gamma}{\text{argmin}}
& & J^{\alpha}(\gamma),
\end{aligned}
\end{equation}
where each candidate path is evaluated lazily for collision, as done before.

The path measure functions are additive over edges, and so is $J^{\alpha}$. Therefore, each iteration of the algorithm is now a shortest path search problem. 
This allows us to trace out the convex hull of the pareto frontier without explicitly enumerating all possible 
candidate paths in the roadmap, which is $O(N!)$.
The edge weight function for each search is obtained as 

\begin{equation}
\label{eq:conv_obj}
w_{j}^{\alpha}(e) = \alpha w_{l}(e) + (1-\alpha) w_{m}(e) \ , \ \alpha \in [0,1].
\end{equation}
When the previous assumption is relaxed, i.e. when collision measures of paths are updated after each search, the corresponding points move and the Pareto frontier moves as well after each search. A visual description of an intermediate 
search of POMP is in Figure \ref{fig:algfig}.

\subsection{Algorithm}
\label{sec:cspacebelief-algorithm}

POMP is outlined in Algorithm \ref{alg:prob_search}.
The $\alpha$ parameter begins from 0, and there is initially no feasible solution (Line 1).
For all $\alpha \leq 1$,
POMP carries out repeated shortest path searches with the 
weight function $w_{\text{j}}^{\alpha}$.
We use the $\text{A}^{*}$ algorithm~\citep{hart1968formal} for the underlying search 
where the heuristic function is the Euclidean length scaled by $\alpha$. 
When $\alpha = 0$, i.e. when the weight function is only $w_m$, the Djikstra~\citep{dijkstra1959note} search algorithm 
($\text{A}^{*}$ with the zero heuristic) is used.
This corresponds to there being no heuristic for the collision measure $w_m$.

Each POMP search is done optimistically, without any edge evaluations.
If the path returned by this search, $\gamma_{\text{new}}$, is the same as the 
current shortest feasible path $\gamma_{\text{curr}}$
(this can only happen once the first feasible solution has been found),
then no shorter candidate path can be found with the current $\alpha$.
In this case POMP immediately increases $\alpha$ (Line 14), which has the effect
of prioritizing length $w_l$ more, and 
resumes searching with a new $w_j^{\alpha}$.

When a new candidate path $\gamma_{\text{new}}$ is obtained, it is lazily evaluated using the
helper method \texttt{LazyEvalPath}, outlined in Algorithm \ref{alg:eval_path}.
\texttt{LazyEvalPath} steps through each  edge in $\gamma_{\text{new}}$, evaluating 
it and updating its weight according to the result. It also returns
the status of the path to POMP. If $\gamma_{\text{new}}$ is feasible (Line 7),
then POMP updates its current best solution $\gamma_{\text{curr}}$ to $\gamma_{\text{new}}$ (Line 8)
and reports it (Line 9). Before starting the next search, POMP increments the $\alpha$ parameter (Line 14).
If, however, $\gamma_{\text{new}}$ turns out to be in collision (Line 10),
POMP simply resumes searching the roadmap with the same $w_j^{\alpha}$.

We increase $\alpha$ monotonically from $0$ to $1$ (rather than restarting from $0$ after each search)
for the same reason we advocated for POMP
over LazyWCSP - we want to prioritize length more and more as we search for feasible paths
in the roadmap. We also noted empirically that the monotonic increase of $\alpha$ traced out a sequence of candidate paths
similar to that obtained from restarting, while reducing the total number of searches conducted.
Regarding the step-size $d\alpha$, the two quantities that are added in the cost function, $w_l$ and $w_m$, have different units, and so the $d\alpha$ value needs to be chosen to balance them reasonably. This is a domain-specific exercise.

\begin{algorithm}[t]
\caption{{POMP}} \label{alg:prob_search}
\begin{algorithmic}[1]

\renewcommand{\algorithmicrequire} {\textbf{Input :} }
\Require $G = (V,E)$, $w_{l}$, $w_{m}$, $s$, $t$, $\Phi$
\State $\alpha \gets 0$ , \ $\gamma_{\text{curr}} = \emptyset$
\While{$\alpha \leq 1$}
	\State $w_{\text{j}}^{\alpha}(e) = \alpha w_{l}(e) + (1-\alpha) w_{m}(e), \ \forall e \in E$
	\State $\gamma_{\text{new}} \gets $ \verb!AStar_Path!$(G,w_{\text{j}}^{\alpha})$
	\If{$\gamma_{\text{new}} \neq \gamma_{\text{curr}}}$
		\State \verb!LazyEvalPath!$(G,\gamma_{new},\Phi)$
		\If{$\gamma_{\text{new}} \in \Cfree$}
			\State $\gamma_{\text{curr}} \gets \gamma_{\text{new}}$
			\State \textbf{yield} $\gamma_{\text{curr}}$
		\Else 
			\State \textbf{continue}
		\EndIf
	\EndIf
	\State $\alpha \gets \alpha + d\alpha$
\EndWhile
\State $\gamma^{*} \gets \gamma_{\text{curr}}$

\end{algorithmic}
\end{algorithm}

\subsection{Implementation}

Here we outline two key implementation issues for POMP, that are related to the specific
kind of belief model that we use.

\subsubsection{Configuration Space Model}
\hfill\\
We utilize a $k$-NN method similar to one used previously~\citep{pan2013faster}.
When $q_i \in \calX$ is evaluated for feasibility, we obtain $F(q_i) = 0 \text{ if } q_i \in \Cfree$ or $1$ otherwise. Then we add $(q_i,F(q_i))$ to the model. Given some new query point q, we obtain the $k$ closest known instances to q, say $\left\{q_1,q_2 \ldots q_k\right\}$, and then compute a weighted average of $F(q_i)$ with weight $w_i = \frac{1}{\zeta(q,q_i)}$. To this quantity, we also do additive smoothing with a prior probability $\lambda$ of weight $w_{\lambda}$. This is a common technique used to smooth the value
returned by an estimator by injecting a prior value. Additive smoothing is particularly useful for getting less noisy estimates when the model has scarce information.
Therefore,
\begin{align*}
&\mathbb{P}[q \in \Cobs] = \frac{\mathbf{w} . \mathbf{F} + w_{\lambda}\cdot \lambda}{|\mathbf{w}| + w_{\lambda}} \\
&\rho(q) = 1 - \mathbb{P}[q \in \Cobs]
\end{align*}
where $\mathbf{w} = [w_1 , w_2 \dotsc w_k]^{T}$ and $\mathbf{F} = [F(q_1) , F(q_2) \dotsc F(q_k)]^{T}$.

In principle, the model lookup and update costs increase with the number of samples, while a collision check is always $O(1)$.
In practice, a datastructure like the Geometric Near-neighbour Access Tree (GNAT)~\citep{brin1995near} makes model interactions efficient, much more so than the average collision check (especially for articulated manipulators).
The GNAT data structure is also optimized for fast querying rather than updating - this aligns well with POMP as it queries the model
multiple times in each search, but only updates the model after evaluating a candidate path.
Asymptotically, however, the model lookup time will exceed check time~\citep{kleinbort2016collision}.

\subsubsection{Efficient Model Updates}
\hfill\\
An important subtlety in Algorithm~\ref{alg:eval_path} is Line 4, where the model is updated 
with the data from collision-checking edge $e$. A change in $\Phi$ implies a change
in the $w_m$ function for potentially all edges.
Naively, this update is done by re-computing $w_m$ for each unevaluated edge
remaining in the roadmap. This is quite wasteful, however, as there are several edges in the roadmap
whose collision measure would be unaffected by the status of edge $e$.
On the other hand, the minimal set of potentially affected edges is the set of all edges which have an embedded configuration
for which any configuration embedded in $e$ is a $k$-nearest neighbour. Computing this
involves running several \emph{reverse k-NN} searches on the configurations embedded in $e$,
an expensive procedure~\citep{singh2003high}.

\begin{algorithm}[t]
\caption{{LazyEvalPath}} \label{alg:eval_path}
\begin{algorithmic}[1]

\renewcommand{\algorithmicrequire} {\textbf{Input :} }
\Require $G = (V,E)$,$\gamma$, $\Phi$
	\For{$e \in \gamma$} 
		\If{$e$ is unevaluated}
			\State \verb!Evaluate!$(e)$
			\State Update $\Phi$ with collision data for $e$.
			\If{$e \in \Cobs$}
				\State $w_{l}(e),w_{m}(e) \gets \infty$  \Comment{Known blocked edge.}
				\State\textbf{return} $\gamma \in \Cobs$
			\Else
				\State $w_{m}(e) \gets 0$ \Comment{Known free edge.}
			\EndIf
		\EndIf
	\EndFor
	\State \textbf{return} $\gamma \in \Cfree$

\end{algorithmic}
\end{algorithm}

\begin{figure*}[t]
	\centering
	\captionsetup[subfigure]{justification=centering}
	\begin{subfigure}[b]{0.99\columnwidth}
		\centering
		\includegraphics[width=0.9\textwidth]{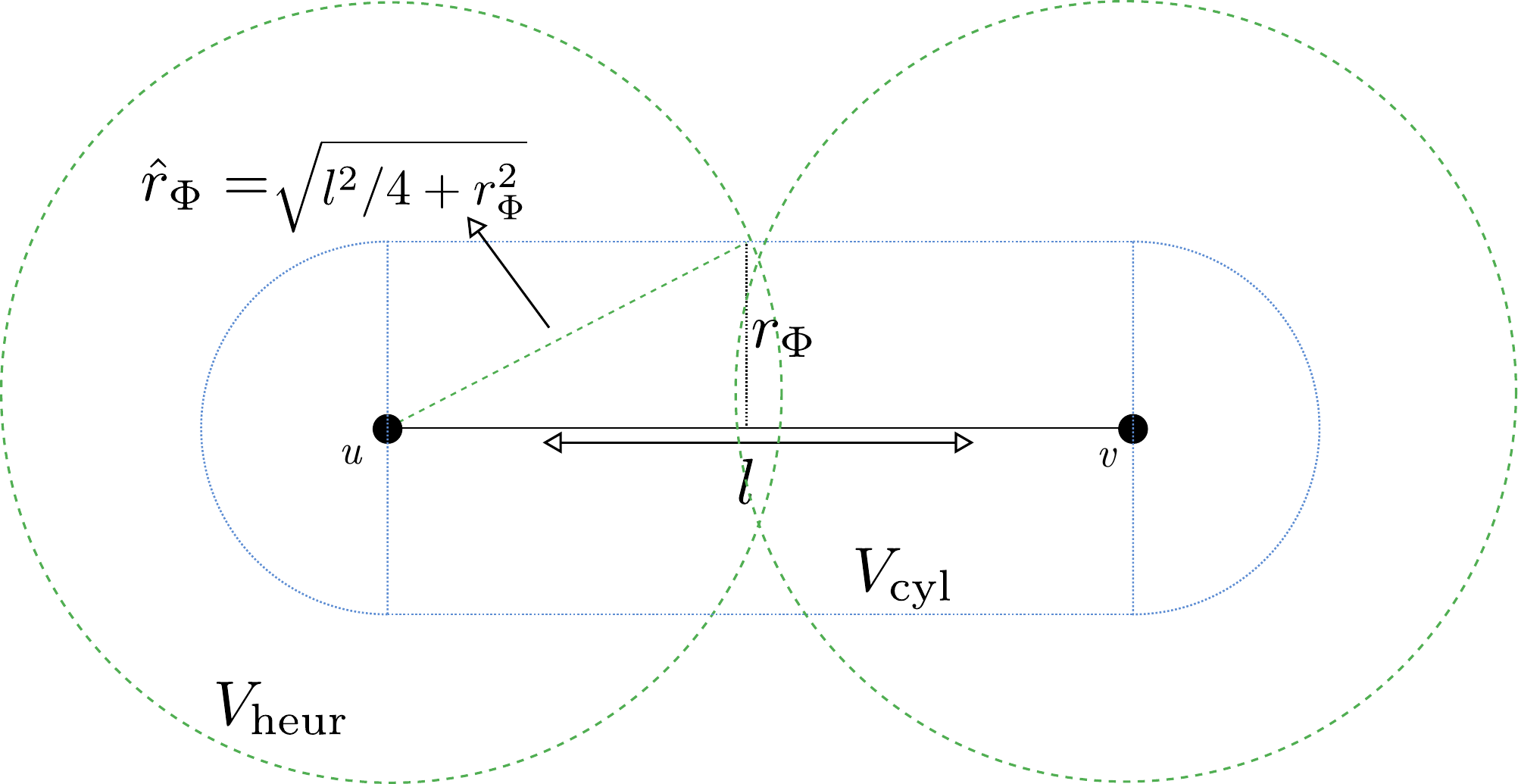}
		\caption{}
		\label{fig:affected-1}
	\end{subfigure}
	\begin{subfigure}[b]{0.99\columnwidth}
		\centering
		\includegraphics[width=0.9\textwidth]{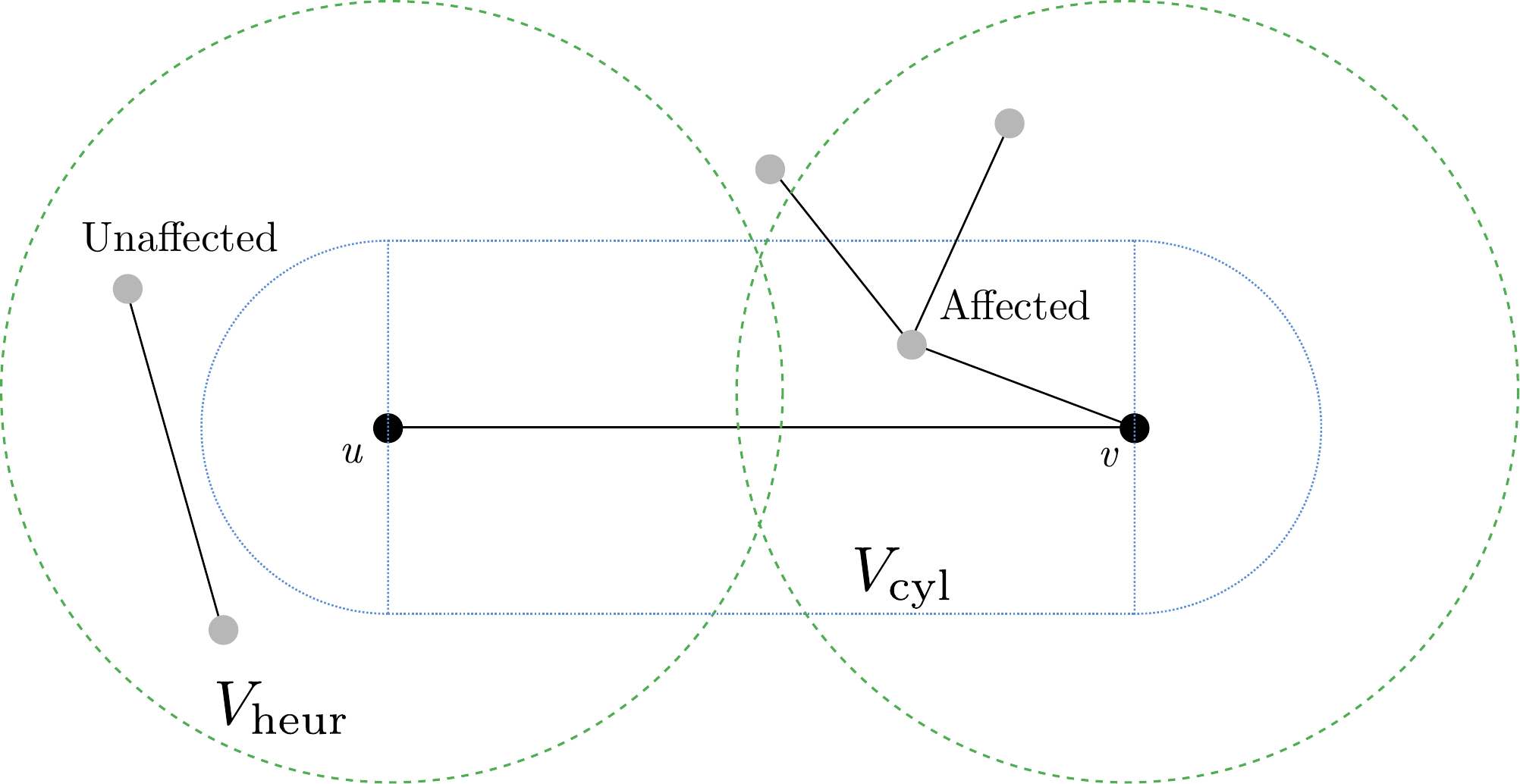}
		\caption{}
		\label{fig:affected-2}
	\end{subfigure}
	\caption{ (\subref{fig:affected-1})~When the edge $e = (u,v)$ is evaluated for collision, the
	 edges affected by the result are those with some configurations in the blue hyper-cylinder around $e$, where
	 $r_{\Phi}$ is the distance threshold for the model. This is expensive to compute exactly, so we instead 
	 query all vertices in the green hyper-sphere of radius $\hat{r}_{\Phi}$ around each of $u$ and $v$,
	 and update all incident edges for those vertices. (\subref{fig:affected-2})~While efficient,
	 this potentially recomputes the collision measure for unaffected edges that lie outside the hyper-cylinder.
	}
	\label{fig:affected}
\end{figure*}

With a simple assumption on our $k$-NN model $\Phi$, we can be efficient in obtaining
the set of potentially affected edges for each evaluated edge. The assumption is that
there is some minimum distance threshold $r_{\Phi}$, beyond which the result of an
evaluated configuration does not affect the estimate for the query
(even if the former happens to be a $k$-nearest neighbour of the latter).
This is a reasonable assumption to make, particularly in the beginning when
the model is very sparse and the $k$-nearest neighbours of a queried configuration
may contain evaluated configurations that are far away.

Under this assumption, given an edge $e = \left(u, v \right)$ of length $l$ 
that is evaluated, all configurations that could possibly be affected are in the union of the
 $d$-dimensional
hyper-cylinder of radius $r_{\Phi}$ and axial length $l$, and the two hyperspheres
of radius $r_{\Phi}$ around the edge endpoints (\figref{fig:affected}).
The volume of this cylinder is 
\begin{equation}
\xi_{\text{cyl}} = \xi_{d-1} \cdot r_{\Phi}^{d-1} \cdot l + \xi_{d} \cdot r_{\Phi}^{d},
\end{equation}
where $\xi_d$ is the volume of the $d$-dimensional unit hypersphere. So the set of all
affected edges is the set of incident edges of vertices within $\xi_{\text{cyl}}$.
As mentioned previously, obtaining this exactly is difficult and expensive.

In practice, we use the heuristic of querying all neighbouring vertices within a radius
$\hat{r}_{\Phi} = \sqrt{l^2/4 + r_{\Phi}^2}$ from each of $u$ and $v$. This involves 
exactly $2$ queries of $O(\text{log }|V|)$ time complexity each, where $|V|$ is the 
current number of vertices in the roadmap.
As shown in \figref{fig:affected}, this ensures that we include all possible affected configurations,
but is an over-estimate as well. The volume that we consider via this heuristic is the 
union of the volume of the 2 hyperspheres, each of radius $\hat{r}_{\Phi}$, around
$u$ and $v$. The heuristic volume is then
\begin{equation}
\xi_{\text{heur}} \leq 2 \cdot \xi_{d} \cdot \left(l^2/4 + r_{\Phi}^2\right)^{\frac{d}{2}}.
\end{equation}

This heuristic is a conservative approximation, thus $\xi_{\text{heur}}$ is strictly greater than $\xi_{\text{cyl}}$
for all values of the parameters. However, other than for evaluated edges where $l \approx \sqrt{d}$, the number of edges
in the hyperspheres is far less than the total number of edges in the roadmap being considered, which is
the naive update strategy. Furthermore, if any unnecessary edge update must be avoided, the set of vertices inside
the hyperspheres can be used as a candidate set from which only vertices inside the hypercylinder are returned,
via a (relatively inexpensive) hypercylinder membership test.

\subsection{Minimizing Expected Length}

Each roadmap search of POMP uses as the objective function a convex combination of edge length and collision measure
based on the current C-space belief. We have explained how this is motivated by our intution that tracing
the convex hull of the Pareto frontier of the path measure plane achieves a desirable tradeoff between
solution quality and likelihood of feasibility. 
Given a C-space belief model, an alternate (perhaps more natural) approach would be to iteratively search for paths 
of minimum expected length. In this section, we will relate the search behaviour of POMP to the search
behaviour induced by this minimum expected length algorithm, to provide
additional theoretical motivation for POMP.

If $\hat{J}(\gamma) = \mathbb{E}[J(\gamma)]$ is the expected length of a path $\gamma$, we know, by
linearity of expectation, that $\hat{J}(\gamma) = \sum\limits_{e \in \gamma}\hat{w_{j}}(e)$.
For the expected length of an edge $\hat{w_j}$, the length of the edge,
if collision-free is $w_{l}(e)$, and if in collision is $w_l(e) + \beta$, where $\beta$ is a length-independent
factor $(\beta \geq 0)$.
Though we require $\gamma \in \Cfree$, we do not use a formulation where the length of $e$ is $w_{l}(e)$ if free and  $\infty$ if
in collision as that would make the expected length of any unevaluated edge $\infty$, which makes analysis meaningless. Because the algorithm eventually evaluates edges, no infeasible paths will be reported as solutions so we do consider the cost of
collision $\infty$ in implementation anyway. The above formulation has appeared in a similar context
for motion-planning roadmaps~\citep{missiuro2006adapting}. We denote $\hat{w_j}$
as $\hat{w_j}^{\beta}$ to indicate that $\beta$ is a parameter.

The larger the value of $\beta$, the more we penalize an edge with a higher likelihood
of collision, regardless of its length. An anytime algorithm that minimizes
expected length while searching for shorter paths would gradually reduce this penalty 
and prioritize the length more. We will show that the sequence of candidate paths
selected by this minimum expected length formulation has a similar trend, 
in terms of length and collision measure, to 
that selected by POMP. Specifically, candidate paths that POMP selects (assuming no model updates while searching) 
have decreasing path length and increasing collision measure as the sequence progresses. 
This is what we mean by the \emph{trend} of the sequence.
It turns out that the minimum expected length algorithm also generates a sequence of candidate paths
with this trend.
We will support this claim by showing that the edge weight function for minimizing expected length, $\hat{w_j}^{\beta}$
behaves similarly to the edge weight function for POMP, $w_j^{\alpha}$, as $\beta$ decreases from $\beta \rightarrow \infty$ to $0$,
and $\alpha$ correspondingly increases from $0$ to $1$.

\begin{figure*}[t]
    \captionsetup[subfigure]{justification=centering}
    \begin{subfigure}[b]{0.99\columnwidth}
        \centering
        \includegraphics[width=0.8\textwidth]{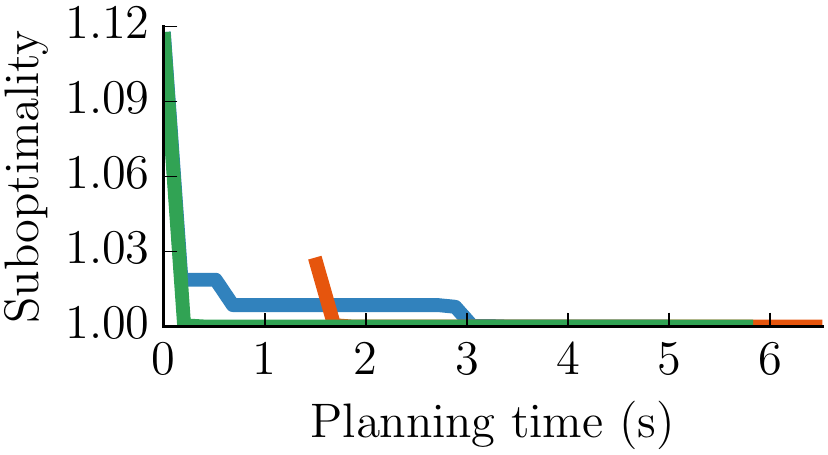}
        \caption{ $\mathbb{R}^{2}$ - Easy}
        \label{fig:results_2d_easy}
    \end{subfigure}
    \begin{subfigure}[b]{0.99\columnwidth}
        \centering
        \includegraphics[width=0.8\textwidth]{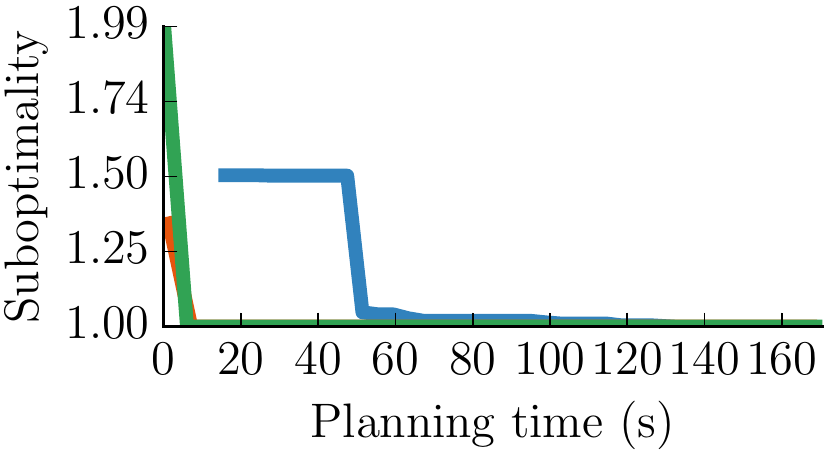}
        \caption{ $\mathbb{R}^{2}$ - Hard}
        \label{fig:results_2d_hard}
    \end{subfigure}

    \begin{subfigure}[b]{0.99\columnwidth}
        \centering
        \includegraphics[width=0.8\textwidth]{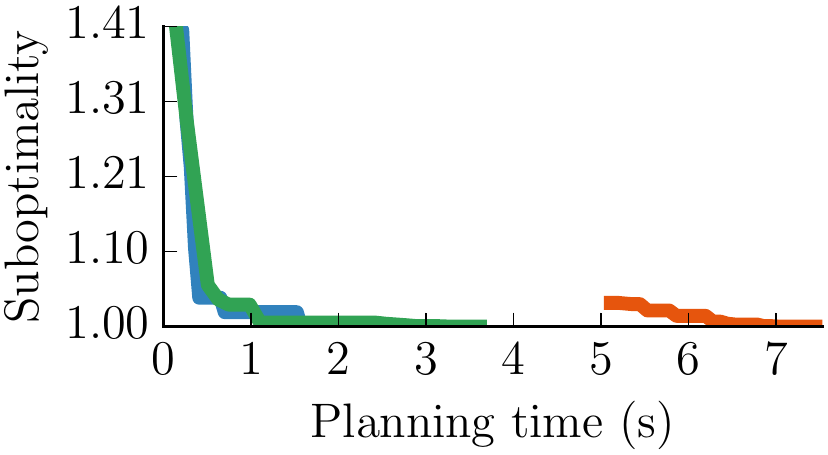}
        \caption{ $\mathbb{R}^{4}$ - Easy}
        \label{fig:results_4d_easy}
    \end{subfigure}
    \begin{subfigure}[b]{0.99\columnwidth}
        \centering
        \includegraphics[width=0.8\textwidth]{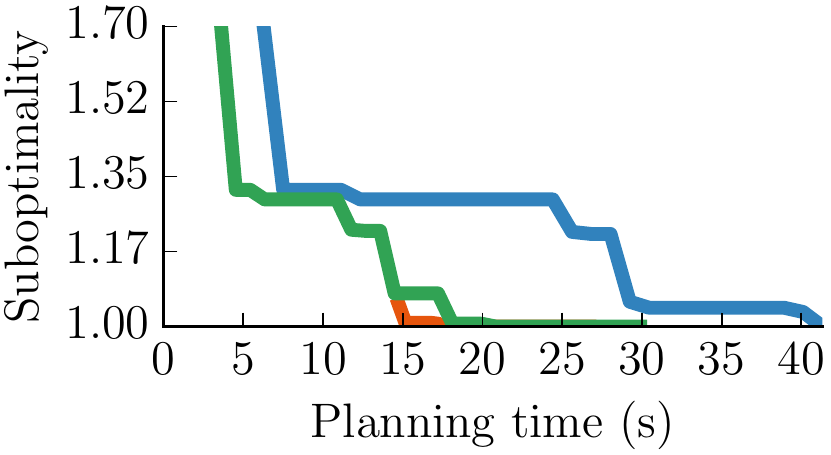}
        \caption{ $\mathbb{R}^{4}$ - Hard}
        \label{fig:results_4d_hard}
    \end{subfigure}

    \caption{ Experimental results in random unit hybercube scenarios for
    \protect\tikz{\protect\node[fill=myblue,draw=black]{};}\; vertex batching
    \protect\tikz{\protect\node[fill=myred,draw=black]{};}\; edge batching, and \protect\tikz{\protect\node[fill=mygreen,draw=black]{};}\; hybrid batching.
    The $y$-axis is the ratio between the length of the path produced by the algorithm and length of $\gamma^*$ (the shortest path
    on $\calG$) for that problem. The naive strategy of searching the complete graph required the following times to find a solution
    - (\subref{fig:results_2d_easy}) $\mathbf{44}$s, (\subref{fig:results_2d_hard}) $\mathbf{200}$s, (\subref{fig:results_4d_easy}) $\mathbf{12}$s and (\subref{fig:results_4d_hard}) $\mathbf{56}$s. In each case this was significantly more than the time for any other strategy to reach the optimum.
    Figure best viewed in color.}

    \label{fig:results_2d_4d}
\end{figure*}

\begin{proposition}
\label{th:minexp}
The search behaviour of POMP on the roadmap generates a sequence of candidate paths that has similar trends
(in terms of path length and collision measure)
to the sequence obtained by searching for paths of minimum expected length on the roadmap.
\end{proposition}

\begin{proof}

By our chosen expected length model,
\begin{multline}
\hat{w_{j}}^{\beta}(e) = \rho({e})w_{l}(e) + \left(1-\rho({e})\right)\left(w_l(e) + \beta\right)
\\ = w_l(e) + \left(1-\rho({e})\right) \beta,
\end{multline}
where $\rho(e)$ is the probability of edge $e$ to be collision-free, using the standard notion of expectation over a single event. 
Compare this to $w_{j}^{\alpha}(e)$ in Eq. \ref{eq:conv_obj}
\begin{align*}
w_{j}^{\alpha}(e) &= \alpha w_{l}(e) + (1-\alpha)w_{m}(e)\\
&\equiv w_{l}(e) + \frac{(1-\alpha)}{\alpha}w_{m}(e) \ \ \text{[equivalent for minimizing]}\\
&= w_l(e) + \frac{(1-\alpha)}{\alpha}\left(-\text{log }\rho(e)\right)
\end{align*}
Consider the relationship between the terms added to $w_l(e)$ in each case - $\left(1-\rho({e})\right) \beta$
and $\frac{(1-\alpha)}{\alpha}\left(-\text{log }\rho(e)\right)$.
For any unevaluated edge (which we care about as expected length is irrelevant for evaluated edges),
$\rho(e) \in (0,1)$. Both functions of $\rho(e)$ decrease monotonically with $\rho(e)$ 
in this interval.
\begin{equation}
\frac{d \left(-\text{log } \rho(e)\right)}{d \left(\rho(e)\right)} = -\frac{1}{\rho(e)} < 0  \ \ \forall \rho(e) \in (0,1)
\end{equation}
\begin{equation}
\frac{d \left(1-\rho(e)\right)}{d \left(\rho(e)\right)} = -1 < 0  \ \ \forall \rho(e) \in (0,1)
\end{equation}
Moreover,
\begin{equation}
\forall \alpha \in (0,1] \ \ \exists \ \beta \ \text{such that} \ 
\frac{1-\alpha}{\alpha}\left(-\text{log }\rho(e)\right) = \beta(1-\rho(e))
\end{equation}
Therefore, $\beta$ varies as $\frac{1-\alpha}{\alpha}$ for $\alpha \in (0,1]$. The corner case of $\alpha = 0$ is handled by
searching only based on $w_m(e)$ in the case of POMP and based on $(1-\rho(e))$ in the case of minimizing expected length.

Of course, in general, we would not compute a different $\beta$ for every edge based on $\rho(e)$. This is why we make a
weaker but useful claim that the sequence of candidate paths selected by POMP, by increasing $\alpha$ continuously from $0$
to $1$, is similar (though not equal) to the sequence obtained by searching for paths of shortest expected length by decreasing
$\beta$ continuously from $\beta \rightarrow \infty$ to $0$. In practice, neither of these variations would be
continuous but rather at some chosen discretization, which would affect the profile of candidate paths in each case, however, we make this assumption for sake of simplicity.
\end{proof}

The $\beta$ parameter represents the penalty factor that the minimum expected length algorithm assigns to additional collision checks. Reducing $\beta$ and increasing $\alpha$ both represent the increasing risk of collision that the respective search algorithms are willing to take while searching for edges that, if collision-free, may potentially lead to shorter paths. It should also be noted that at the stage where $\alpha = 1 \implies \beta = 0$, POMP is equivalent to LazyPRM~\citep{bohlin2000path} which searches for paths based only on their optimistic length.

Even though POMP and the minimum expected length algorithm have similar behaviour, we advocate for 
using the former in practice. The minimum expected length algorithm would begin by prioritizing
low probability of collision and end by prioritizing length. However, it does not necessarily select points on the convex hull
of the Pareto frontier of the path measure plane. We have explained in \sref{sec:cspacebelief-pomp} 
how such behaviour does not represent the kind of length-probability tradeoff that we want.


\section{Experiments}
\label{sec:experiments}

We extensively tested our key ideas, both individually
and as a combined framework for anytime motion planning
on large dense roadmaps. In this section, we discuss the results
of those experiments. All algorithms were implemented in
\texttt{C++} with the OMPL~\citep{SMK12} library, and the manipulation planning examples were run
from a \texttt{Python} interface. All testing was done on an
Ubuntu 14.04 system with a 3.4 GHz processor and 16GB RAM.
The relative performance of the algorithms
is of interest to us rather than the absolute timing data.

\begin{figure*}[t]
    \captionsetup[subfigure]{justification=centering}
    \begin{subfigure}[b]{0.99\columnwidth}
        \centering
        \includegraphics[width=0.7\textwidth]{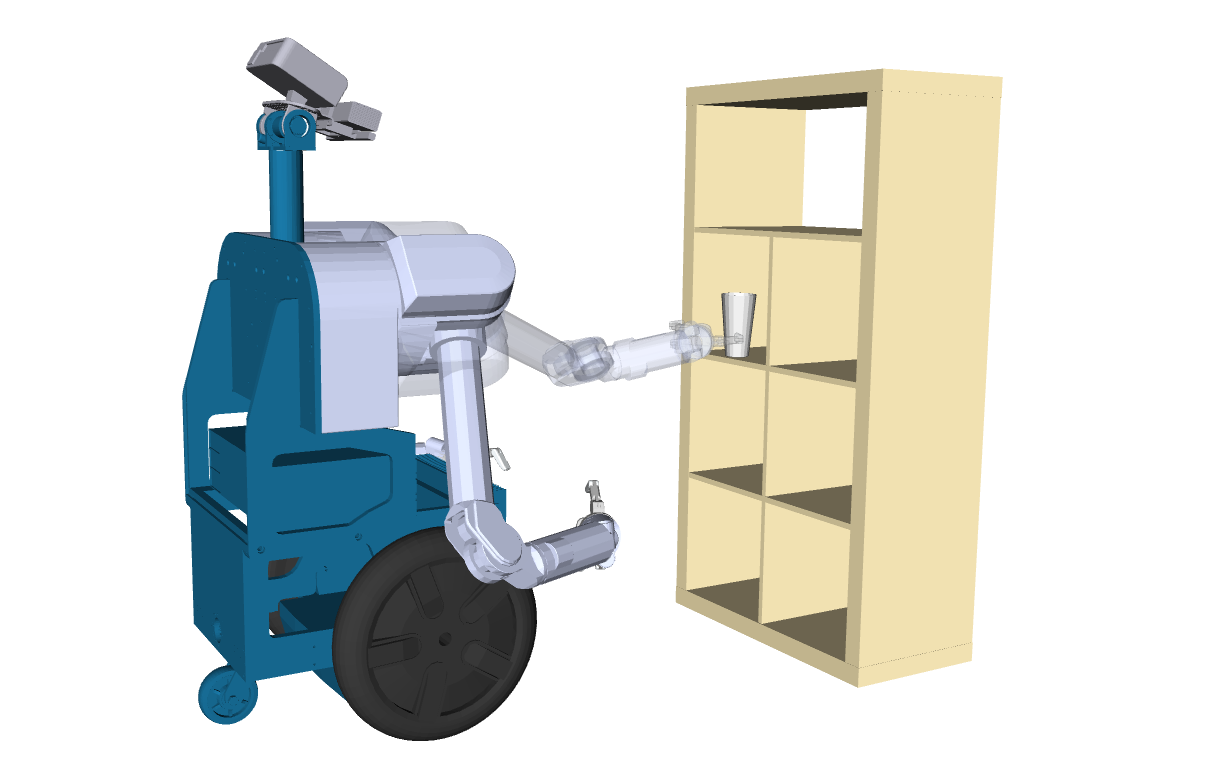}
        \caption{}
        \label{fig:herbprob1}
    \end{subfigure}
    \begin{subfigure}[b]{0.99\columnwidth}
        \centering
        \includegraphics[width=0.8\textwidth]{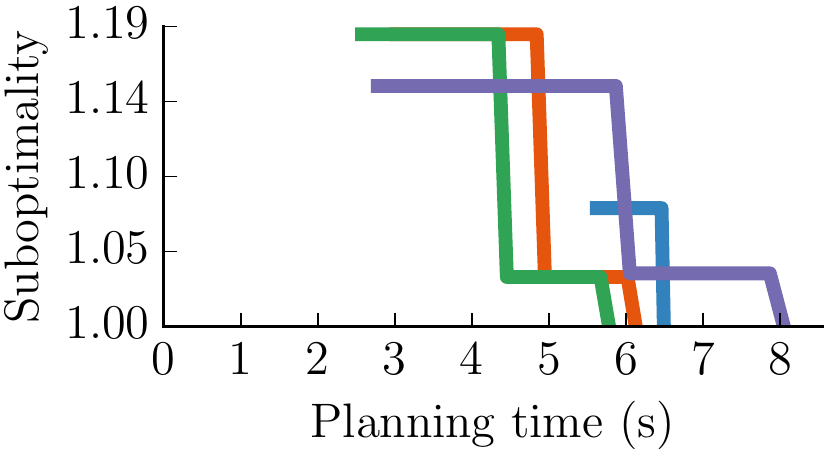}
        \caption{}
        \label{fig:herbplot1}
    \end{subfigure}

    \begin{subfigure}[b]{0.99\columnwidth}
        \centering
        \includegraphics[width=0.7\textwidth]{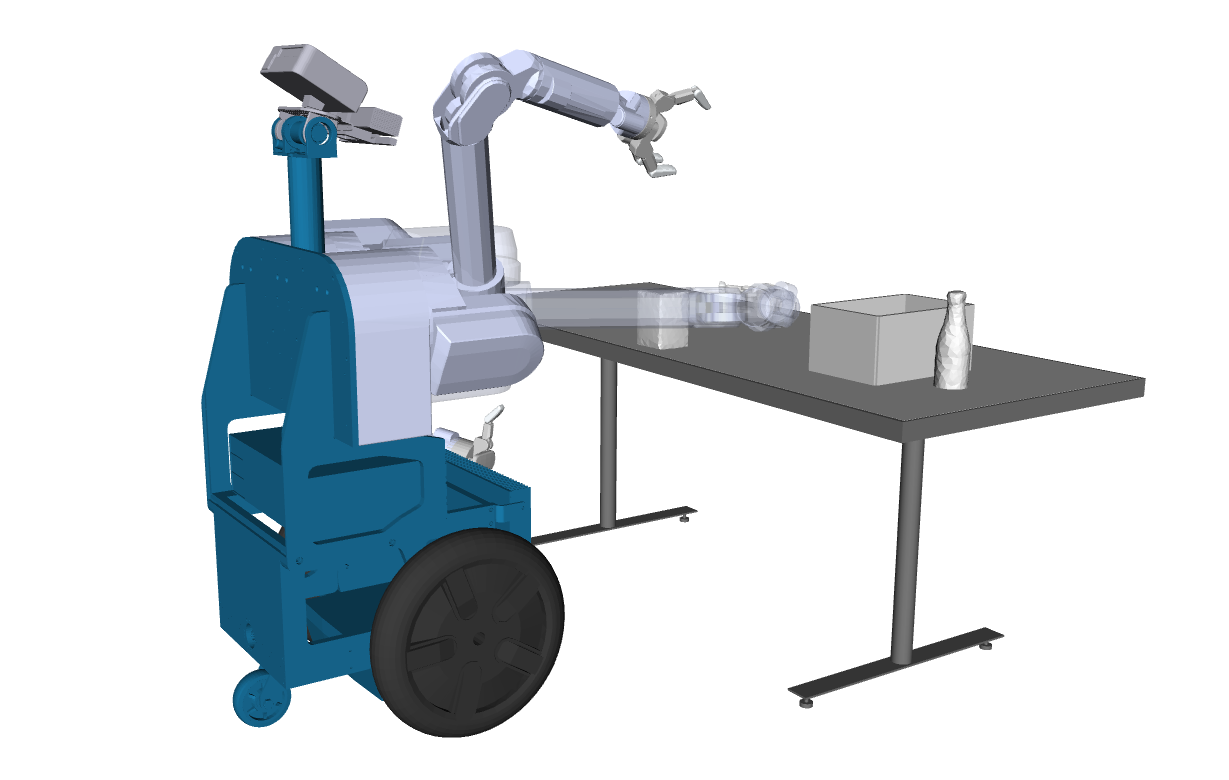}
        \caption{}
        \label{fig:herbprob2}
    \end{subfigure}
    \begin{subfigure}[b]{0.99\columnwidth}
        \centering
        \includegraphics[width=0.8\textwidth]{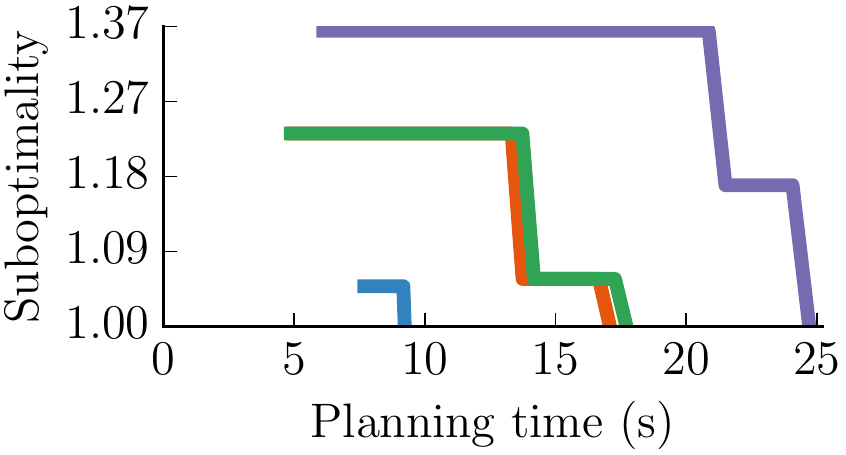}
        \caption{}
        \label{fig:herbplot2}
    \end{subfigure}

    \caption{ We show results on 2 manipulation problems for \protect\tikz{\protect\node[fill=myblue,draw=black]{};}\; vertex batching
    \protect\tikz{\protect\node[fill=myred,draw=black]{};}\; edge batching, \protect\tikz{\protect\node[fill=mygreen,draw=black]{};}\; hybrid batching
    and \protect\tikz{\protect\node[fill=mypurple,draw=black]{};}\; $\text{BIT}^{*}$. For each problem the goal configuration of the right
    arm is rendered translucent. Both of the problems were fairly constrained and non-trivial. The problem depicted in (\subref{fig:herbprob2}) has a 
    large clear area in front of the starting configuration, which may allow for a long edge. This could explain the better performance
    of vertex batching. The naive strategy takes $\mathbf{25}$s for (\subref{fig:herbplot1}) and $\mathbf{44}$s for (\subref{fig:herbplot2}) respectively.}

    \label{fig:results_herb}
\end{figure*}

\subsection{Densification}
\label{sec:experiments-densification}

Our implementations of the various densification strategies were based on the publicly available
OMPL implementation of $\text{BIT}^{*}$~\citep{gammell2014batch}. Other than the specific parameters
and optimizations mentioned earlier, we used the default parameters of $\text{BIT}^{*}$. 
Notably, we used the Euclidean distance heuristic and limited graph pruning to changes in path length greater than 1\%.
For all our densification experiments, we plot the anytime performance of the various algorithms. This
is the curve of solution sub-optimality (with respect to the optimal solution on the roadmap) vs. planning
time to find each solution. Between any two curves demonstrating anytime planning performance, the one more to the lower left indicates better performance.

\subsubsection{Random hypercube scenarios}
\hfill\\
The different batching strategies were compared to each other on problems in $\mathbb{R}^{d}$ for $d =2,4$.
The domain was the unit hypercube $[0,1]^{d}$ while the obstacles were randomly generated axis-aligned
$d$-dimensional hyper-rectangles. 
All problems had a start configuration of $[0.25,0.25,\ldots]$
and a goal configuration of $[0.75,0.75,\ldots]$.
We used the first $n = 10^{4}$ and $n = 10^{5}$ points of the Halton sequence for the~$\mathbb{R}^{2}$ 
and~$\mathbb{R}^{4}$ problems respectively. The roadmaps were all complete, i.e. all edges
between vertices were considered.

Two parameters of the obstacles were varied to approximate the notion of problem hardness
described earlier -- the number of obstacles and the fraction
of $\mathcal{X}$ which is in $\Cobs$, which we denote by $\xi_{\text{obs}}$. 
An easy problem is one which
has fewer obstacles and a smaller value of $\xi_{\text{obs}}$.
The converse is true for a hard problem. 
Specifically, in $\mathbb{R}^{2}$, we had 
easy problems with $100$ obstacles and $\xi_{\text{obs}} = 0.33$, and 
hard problems with $1000$ obstacles and $\xi_{\text{obs}} = 0.75$. 
In $\mathbb{R}^{4}$
we maintained the same values for $\xi_{\text{obs}}$, but used $500$ and $3000$ obstacles for easy and hard problems, respectively. For each problem setting ($\mathbb{R}^{2}$/$\mathbb{R}^{4}$; easy/hard) we generated~$30$ different random
scenarios and evaluated each strategy with the same set of samples on each of them. 
Each random scenario had a different set of solutions, so we show a representative
result for each problem setting in \figref{fig:results_2d_4d}. The majority
of the $30$ results show similar relative behaviour between the three batching
strategies.

For easy problems, vertex batching had better anytime performance than
edge batching and vice versa for hard problems. Hybrid batching performed well
in both problem settings.
These results align with our analysis
of the relative performance of the densification strategies on easy and hard problems.
The naive strategy of searching $\calG$ with $\text{A}^{*}$ directly required considerably more time to report the 
optimum solution than any other strategy. We mention the numbers 
in the accompanying caption of \figref{fig:results_2d_4d} but avoid plotting them so
as not to stretch the figures. The hypercube plots show the reasonable performance of hybrid batching across problems and difficulty levels.

\subsubsection{Manipulation planning problems}
\hfill\\
We also ran simulated experiments on HERB~\citep{srinivasa2010herb}, a mobile manipulator designed and built by the Personal Robotics Lab at Carnegie Mellon University. 
The planning problems were for the right arm, a $7$-DOF Barrett WAM, on the problem
scenarios shown in ~\figref{fig:results_herb}. 
We used a complete roadmap of $10^5$ vertices defined using a Halton sequence~$\calS$ which was generated using the first $7$ prime numbers.
In addition to the batching strategies, we also evaluated the performance
of $\text{BIT}^{*}$, which was forced to use the same set of samples~$\calS$. 
$\text{BIT}^{*}$ has been shown to achieve anytime performance
superior to contemporary anytime algorithms.
The hardness of the problems in terms of clearance
is difficult to visualize because of the high-dimensional C-space.

In our results (~\figref{fig:results_herb}), for both problems, hybrid batching had better
anytime performance than $\text{BIT}^{*}$. Furthermore, no batching method did significantly
worse than $\text{BIT}^{*}$. Of course, these are individual examples that do not admit of any
statistical significance. The purpose of these $2$ experiments was to demonstrate that the batching
strategies we propose perform comparably to the $\text{BIT}^{*}$ strategy on sample robot planning problems.

\begin{figure*}[t]
\captionsetup[subfigure]{justification=centering}
    \begin{subfigure}[b]{0.33\columnwidth}
    \centering
        \includegraphics[width=0.6\columnwidth]{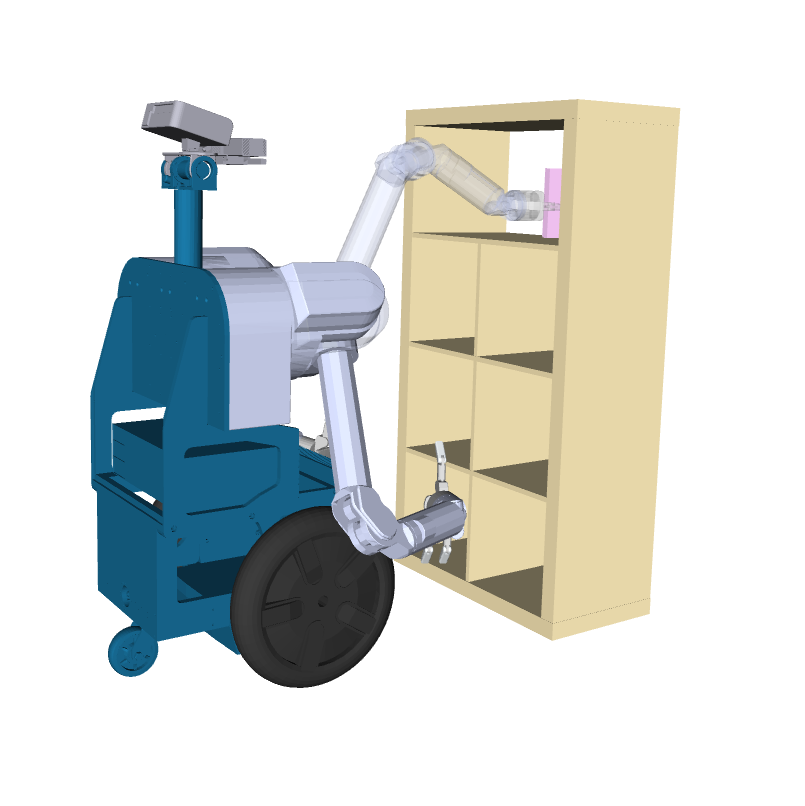}
        \caption{P1}
        \label{fig:prob1}
    \end{subfigure}
    \begin{subfigure}[b]{0.33\columnwidth}
    \centering
        \includegraphics[width=0.6\columnwidth]{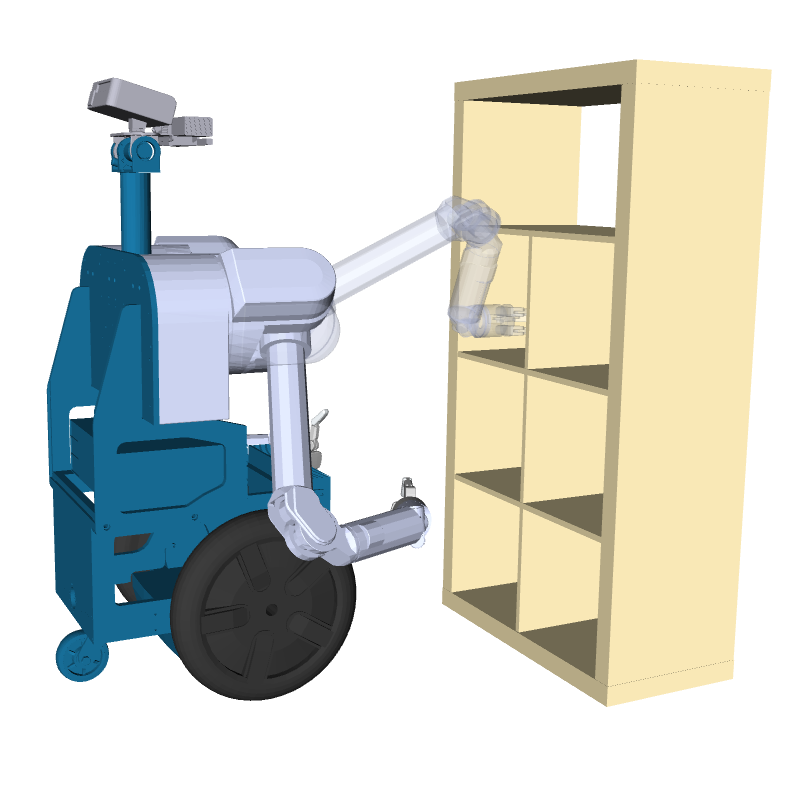}
        \caption{P2}
        \label{fig:prob2}
    \end{subfigure}
    \begin{subfigure}[b]{0.33\columnwidth}
    \centering
        \includegraphics[width=0.6\columnwidth]{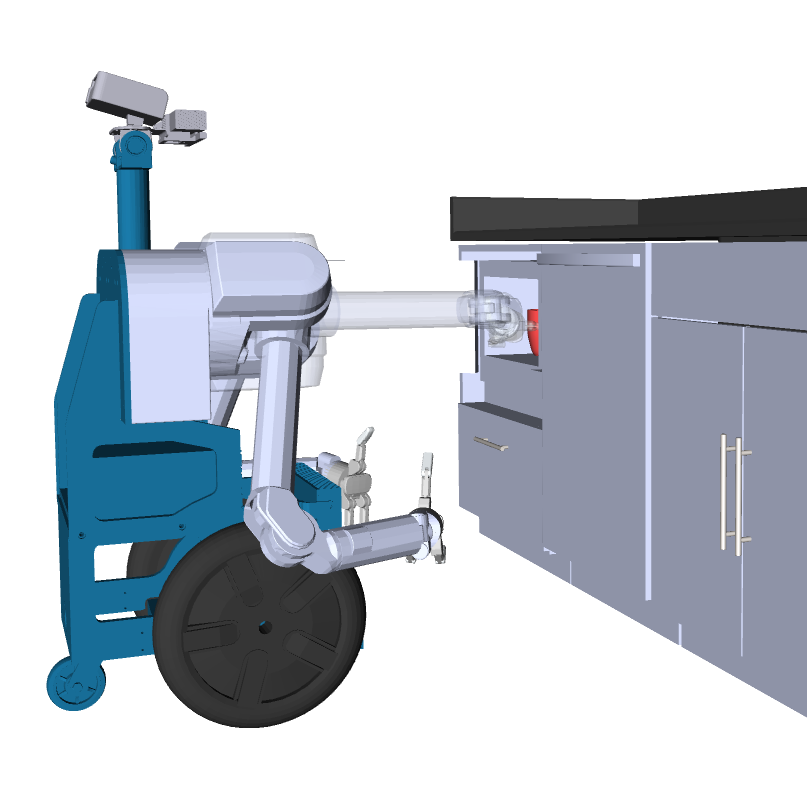}
        \caption{P3}
        \label{fig:prob3}
    \end{subfigure}
    \begin{subfigure}[b]{0.33\columnwidth}
    \centering
        \includegraphics[width=0.7\columnwidth]{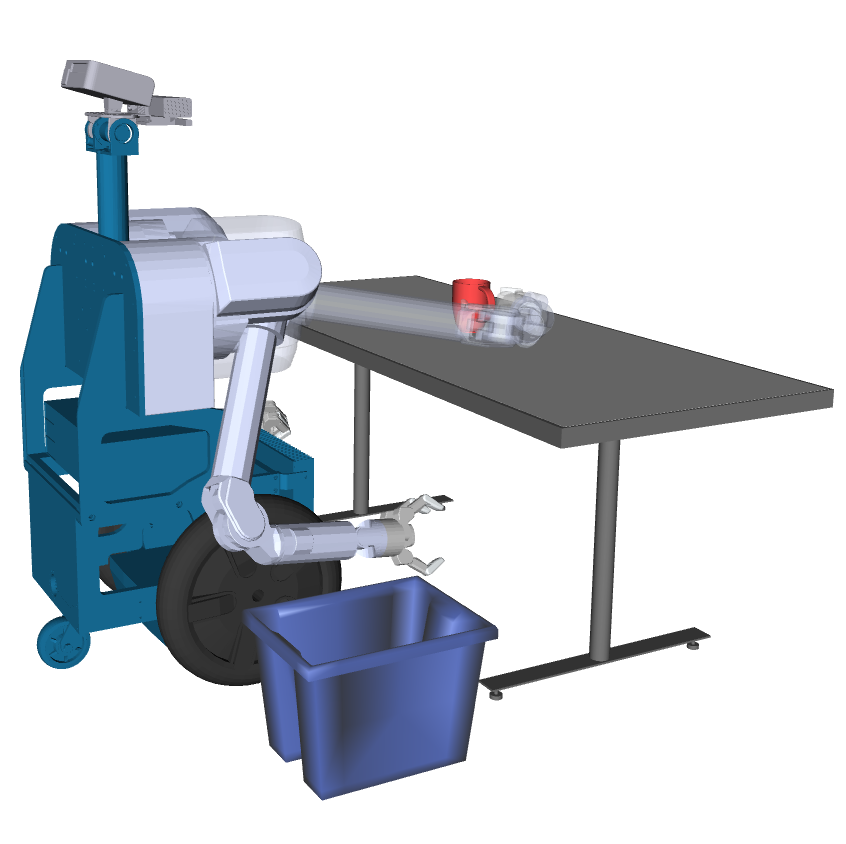}
        \caption{P4}
        \label{fig:prob1}
    \end{subfigure}
    \begin{subfigure}[b]{0.33\columnwidth}
    \centering
        \includegraphics[width=0.7\columnwidth]{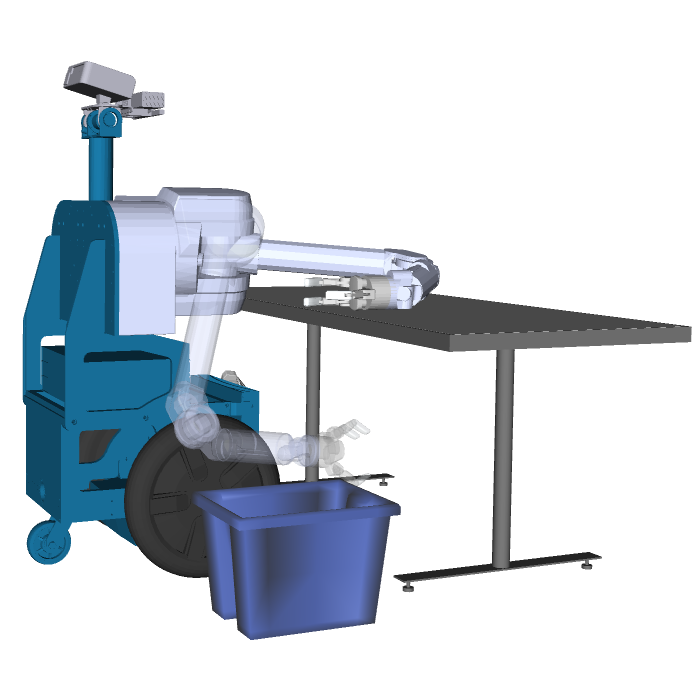}
        \caption{P5}
        \label{fig:prob2}
    \end{subfigure}
    \begin{subfigure}[b]{0.33\columnwidth}
    \centering
        \includegraphics[width=0.7\columnwidth]{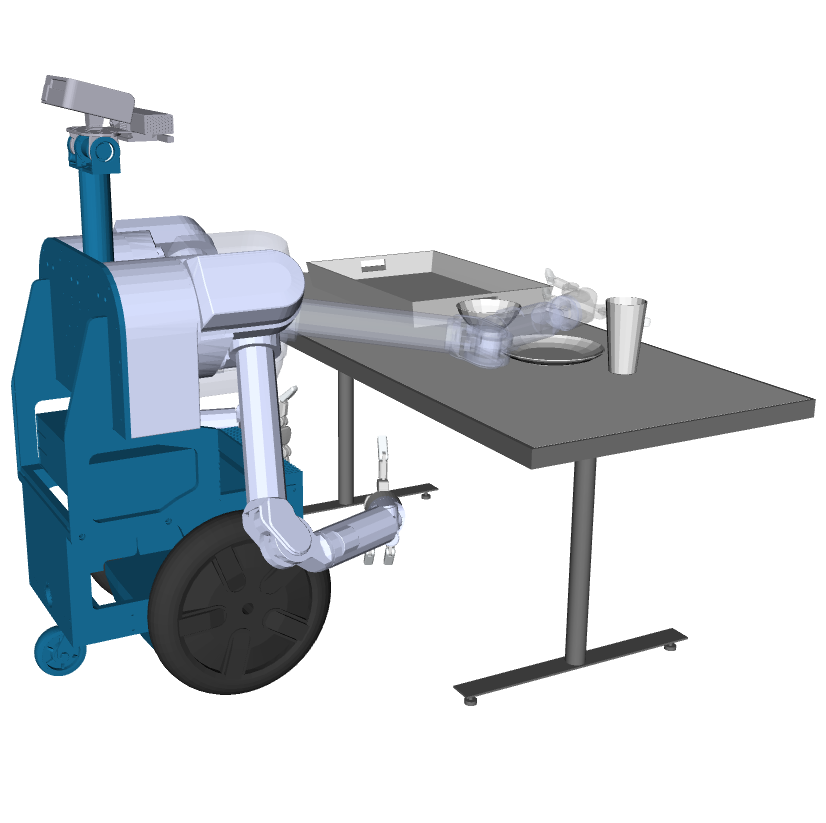}
        \caption{P6}
        \label{fig:prob3}
    \end{subfigure}

    \caption{The test cases we use for our POMP experiments. We name them P1 through P6 for reference. The planning is for the right arm of the robot, which is at the starting configuration in each case. The translucent rendered arm represents the desired goal configuration.}

    \label{fig:all_probs}
\end{figure*}

\begin{figure*}
\captionsetup[subfigure]{justification=centering}
    \begin{subfigure}[b]{0.99\columnwidth}
    \centering
        \includegraphics[width=\textwidth]{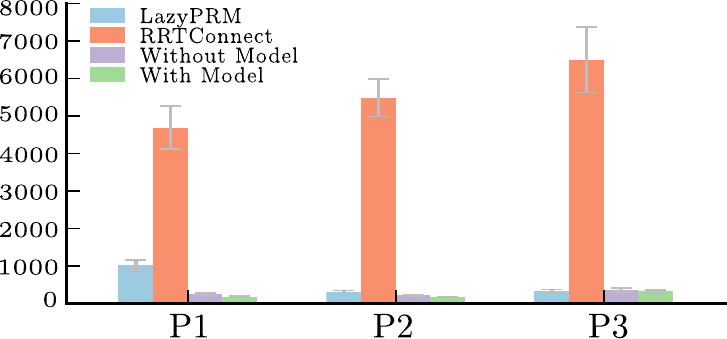}
        \caption{Number of Collision Checks}
        \label{fig:perfhyp-check}
    \end{subfigure}
    \begin{subfigure}[b]{0.99\columnwidth}
    \centering
        \includegraphics[width=\textwidth]{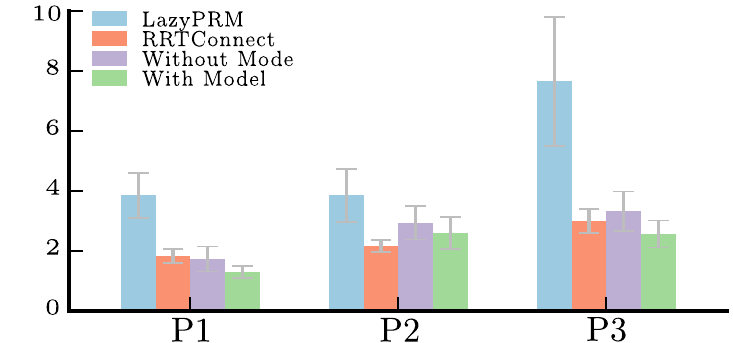}
        \caption{Planning Time (s)}
        \label{fig:perfhyp-time}
    \end{subfigure}
    \caption{ A comparison of our algorithm POMP with LazyPRM and RRTConnect, in terms of the average planning time and collision checks required for computing the first feasible path. POMP requires far fewer checks than RRTConnect, but spends additional time searching and updating the large roadmap.}
    \label{fig:perfhyp}
\end{figure*}

\subsection{POMP}

\begin{figure*}[t]
\captionsetup[subfigure]{justification=centering}
    \begin{subfigure}[b]{0.66\columnwidth}
    \centering
        \includegraphics[width=0.9\columnwidth]{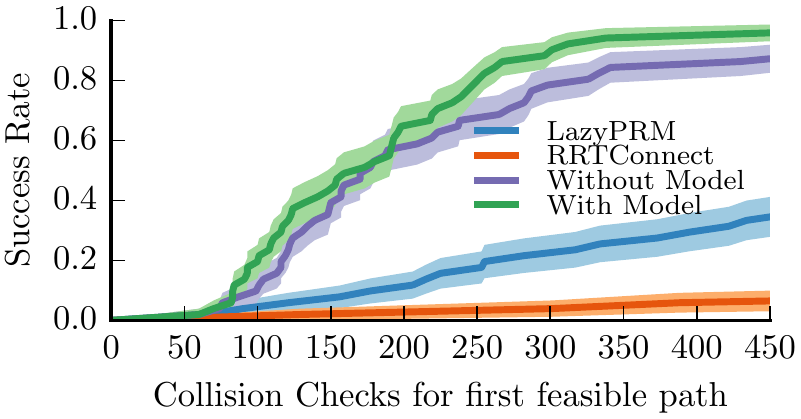}
        \caption{P1}
        \label{fig:jenchecks1}
    \end{subfigure}
    \begin{subfigure}[b]{0.66\columnwidth}
    \centering
        \includegraphics[width=0.9\columnwidth]{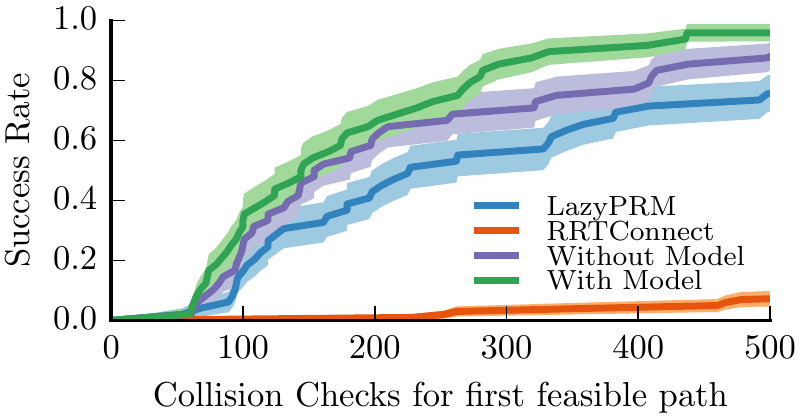}
        \caption{P2}
        \label{fig:jenchecks2}
    \end{subfigure}
    \begin{subfigure}[b]{0.66\columnwidth}
    \centering
        \includegraphics[width=0.9\columnwidth]{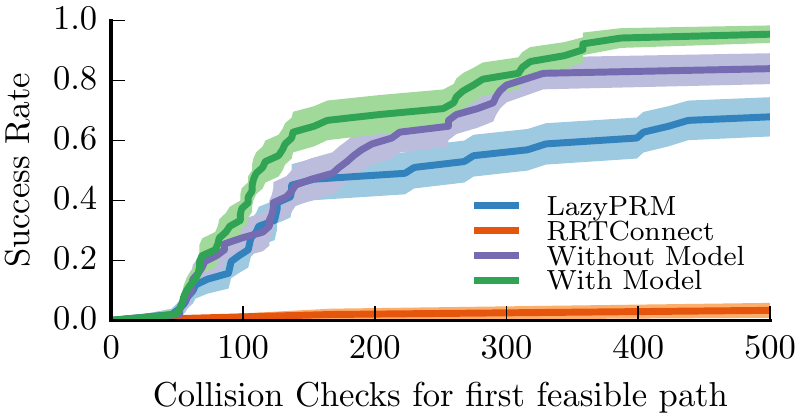}
        \caption{P3}
        \label{fig:jenchecks3}
    \end{subfigure}

    \begin{subfigure}[b]{0.66\columnwidth}
    \centering
        \includegraphics[width=0.9\columnwidth]{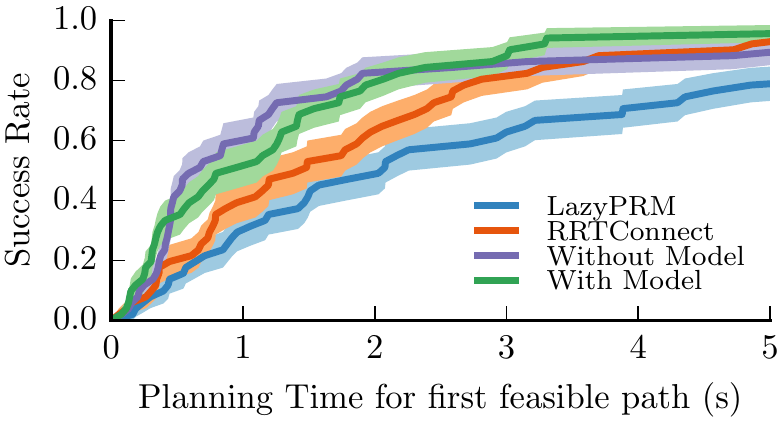}
        \caption{P1}
        \label{fig:jentime1}
    \end{subfigure}
    \begin{subfigure}[b]{0.66\columnwidth}
    \centering
        \includegraphics[width=0.9\columnwidth]{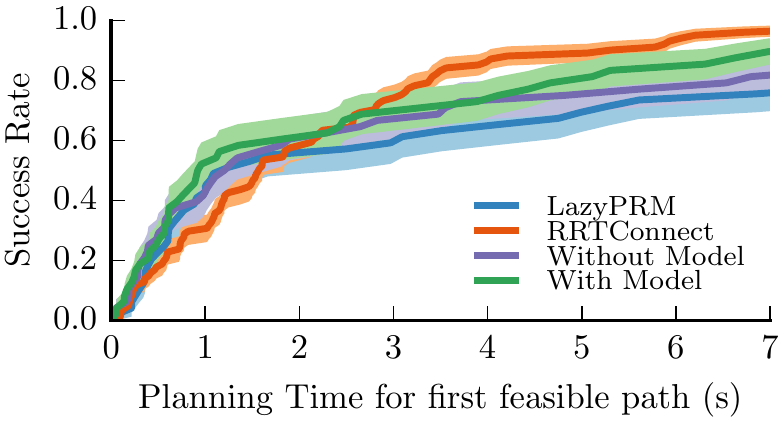}
        \caption{P2}
        \label{fig:jentime2}
    \end{subfigure}
    \begin{subfigure}[b]{0.66\columnwidth}
    \centering
        \includegraphics[width=0.9\columnwidth]{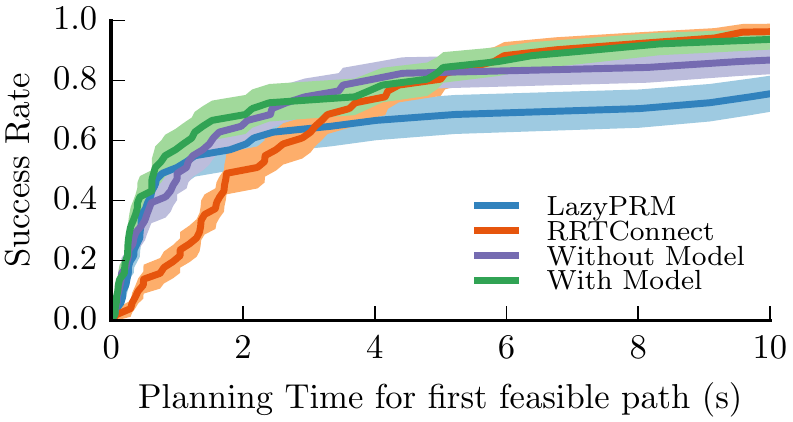}
        \caption{P3}
        \label{fig:jentime3}
    \end{subfigure}
    
    \caption{These plots shows how the collision checks and planning time required varies with the percentage of successful runs for each algorithm. Note that With Model, Without Model and LazyPRM were all run on the same roadmap, and would have all reported a feasible solution if one existed on the roadmap. In each case, the $x$-axis is cut off after all runs of With Model, on roadmaps with feasible solutions, have concluded. RRTConnect would keep searching till it found a solution, and asymptotically its success rate would be 1.}
    \label{fig:jenplots}
\end{figure*}

\begin{figure*}
    \captionsetup[subfigure]{justification=centering}
    \begin{subfigure}[b]{\columnwidth}
        \centering
        \includegraphics[width=0.8\columnwidth]{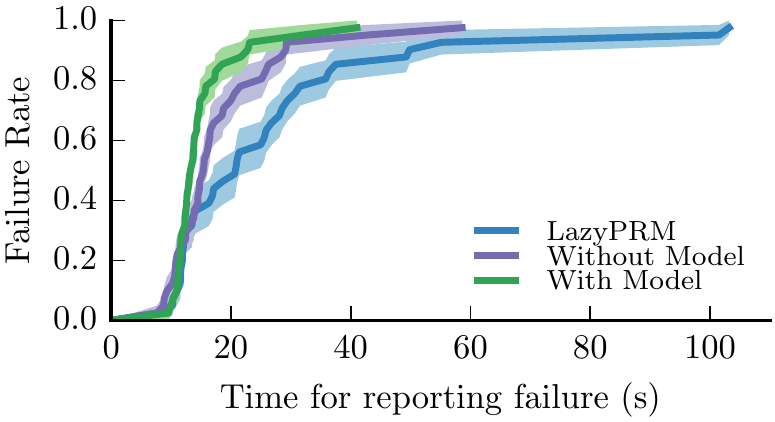}
        \caption{}
        \label{fig:failtime}
    \end{subfigure}
    \begin{subfigure}[b]{\columnwidth}
        \centering
        \includegraphics[width=0.8\columnwidth]{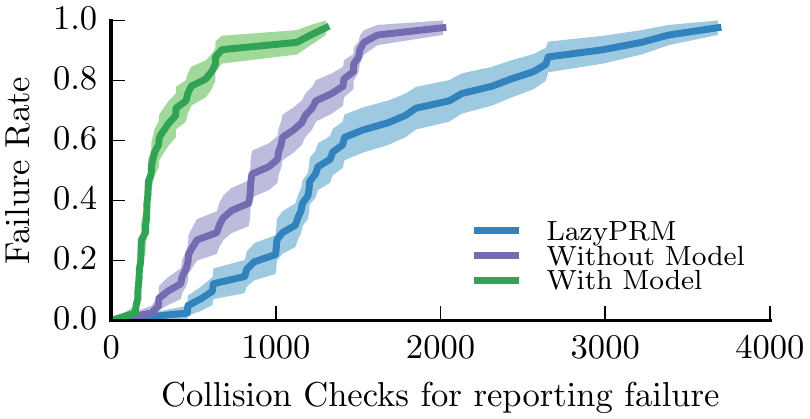}
        \caption{}
        \label{fig:failchecks}
    \end{subfigure}
    
    \caption{POMP with a model reports failure faster than without a model, and LazyPRM,
    for the same roadmaps and problem instances. This is aggregated over all failure cases
    from P1 through P3.}
    \label{fig:failfast}
\end{figure*}

\begin{figure*}
\captionsetup[subfigure]{justification=centering}
    \begin{subfigure}[b]{0.66\columnwidth}
    \centering
        \includegraphics[width=\columnwidth]{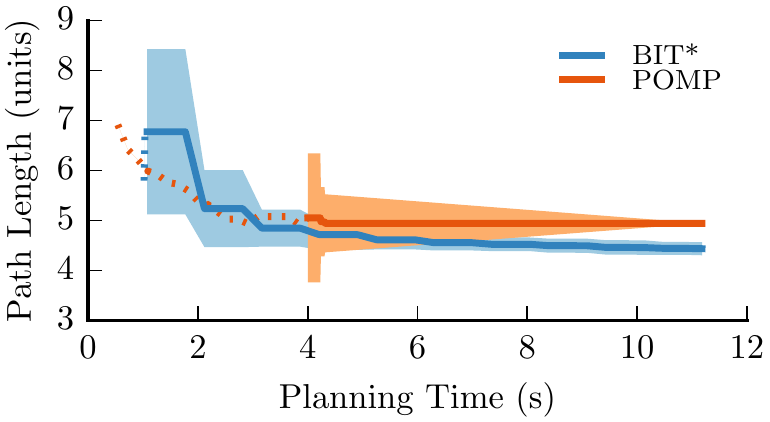}
        \caption{Anytime - P4}
        \label{fig:anytimehyp-p1}
    \end{subfigure}
    \begin{subfigure}[b]{0.66\columnwidth}
    \centering
        \includegraphics[width=\columnwidth]{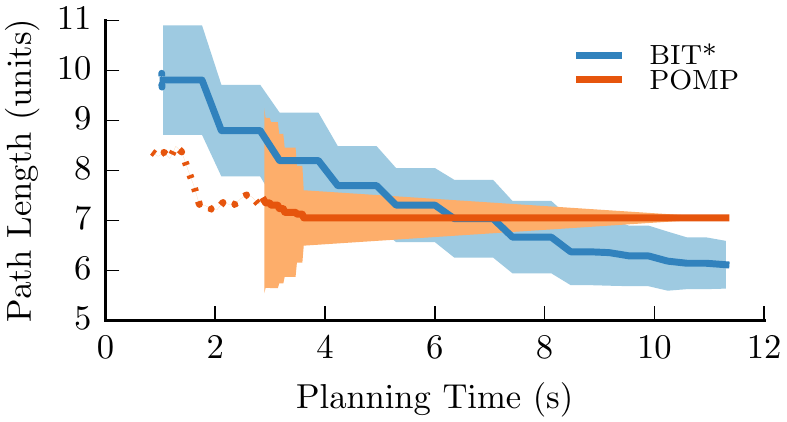}
        \caption{Anytime - P5}
        \label{fig:anytimehyp-p2}
    \end{subfigure}
    \begin{subfigure}[b]{0.66\columnwidth}
    \centering
        \includegraphics[width=\columnwidth]{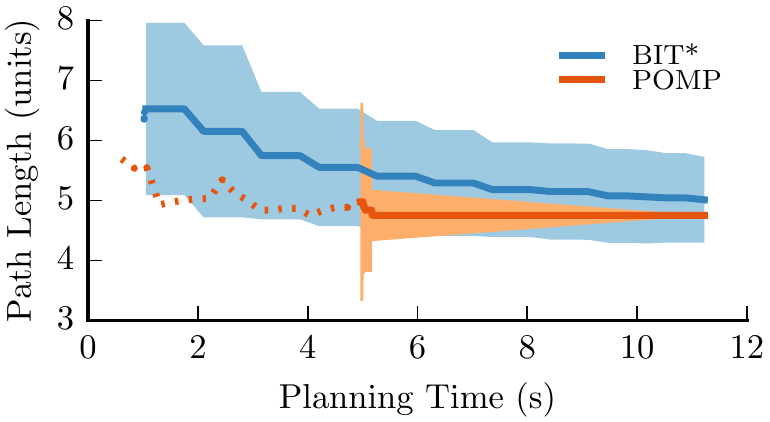}
        \caption{Anytime - P6}
        \label{fig:anytimehyp-p5}
    \end{subfigure}
    \caption{ A comparison of the anytime performance of POMP with that of $\text{BIT}^*$. This was done on 3 separate problems that better demonstrate anytime behaviour than the ones used earlier. The dotted line begins after 50\% of the runs have found a solution. The solid line begins after all runs have found a solution. The flattening of the lines for POMP happens after the final roadmap finds the shortest path, as there is no further scope of improvement.}
    \label{fig:anytimehyp}
\end{figure*}

We evaluated POMP through a number of simulated experiments on HERB. We considered two hypotheses about POMP - the benefit of using and updating a C-space belief model for computing the first solution, and the anytime performance. Our experiments were run on 6 different planning problems for the 7-DOF right arm, shown in \figref{fig:all_probs}. The first three problems - P1, P2, P3 - were used for evaluating the first hypothesis. They have goal configurations with significant visibility constraints. The next three problems - P4, P5, P6 - were used for the second hypothesis. Their goal configurations are less constrained than the first three. Thus they have more feasible solutions and better demonstrate anytime behaviour.

For each problem, we tested POMP over 50 different roadmaps. For each roadmap, the samples were generated from a Halton sequence and the node positions were offset by random amounts. The roadmaps each had approximately 14000 nodes, and the $r$-disk radius for connectivity was $0.3$ radians. This radius induced an average degree of about $16$, which is a little more than $\mathrm{log}_2 14000$.

Using explicit fixed roadmaps allowed us to eliminate all nodes and edges which have configurations in self collision in a pre-processing step, thereby requiring us to only evaluate environmental collisions at runtime. We utilized the same set of default model parameters for each run of POMP -  the $k$ for $k$-NN lookup is $15$, the prior belief is $0.5$, and the weight of this prior $0.25$ and $\alpha$ increases in steps of $0.1$.

\subsubsection{Benefit of C-space belief model for first feasible path}
\hfill\\
We evaluated the planning time and the number of collision checks required to obtain a feasible solution. We compared against the widely used LazyPRM~\citep{bohlin2000path} and RRT-Connect~\citep{kuffner2000rrt}. For RRTConnect, we used the standard OMPL implementation. For LazyPRM, we used the search of POMP with $\alpha = 1$ on the same roadmaps as POMP. We also compared against a variant of POMP that does not use a belief model - it assigns the same probability of collision $\rho(q) = 0.5$ to all unknown configurations and only sets them to 0 or 1 when they are evaluated. This was done to demonstrate the tradeoff between actually updating the model, at some computational expense, and using a prior without any further updates. We name the original `With Model' and the variant `Without Model'.

\figref{fig:perfhyp} shows the average collision checks and planning time to compute the first feasible solution for the various algorithms. This is for those roadmaps that have at least one feasible solution for the problem. A second perspective is shown in \figref{fig:jenplots}, which shows the success rate of the methods with time and checks. This plot considers all of the 50 roadmaps, whether they have a feasible solution or not, and so the success rate of the methods using them (With Model, Without Model, LazyPRM) all have the same upper bound (in the limit). 

The figures show that over all problems, POMP with a belief model achieved superior average-case performance. Furthermore, the length of the first feasible path returned by POMP is better than RRTConnect. For the three problems, the average length of feasible paths computed by POMP was approximately 60\% that of paths computed by RRTConnect. Additionally, for cases where the roadmap had no feasible solution, POMP using a model reported failure more quickly than the variant without a model and LazyPRM (\figref{fig:failfast}).

Observe in Figure \ref{fig:perfhyp} that though RRTConnect had an order of magnitude more collision checks than POMP, the planning time was still comparable. A qualitative breakdown of the timing shows that POMP spent far less time than RRT-Connect actually doing collision checking. However, it also had greater overhead for searching the roadmap for candidate paths and 
updating the collision measure of edges after collision tests. This reaffirms the importance of efficient belief model updates 
to the behaviour and performance of POMP.

\begin{figure*}[t]
    \centering
    \captionsetup[subfigure]{justification=centering}
    \begin{subfigure}[b]{0.66\columnwidth}
        \centering
        \includegraphics[width=\textwidth]{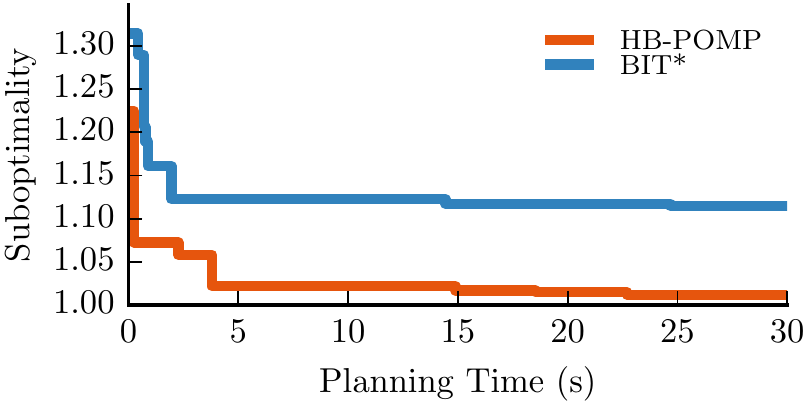}
        \caption{}
        \label{unit-6d-14-ns}
    \end{subfigure}
    \centering
    \begin{subfigure}[b]{0.66\columnwidth}
        \centering
        \includegraphics[width=\textwidth]{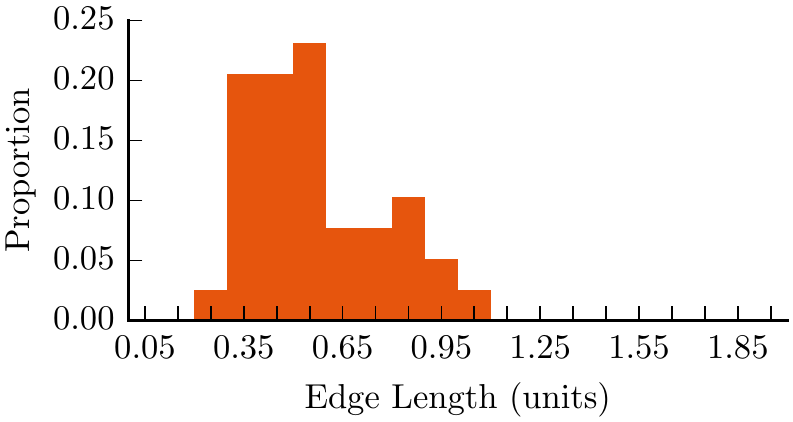}
        \caption{}
        \label{unit-6d-14-hb}
    \end{subfigure}
    \centering
    \begin{subfigure}[b]{0.66\columnwidth}
        \centering
        \includegraphics[width=\textwidth]{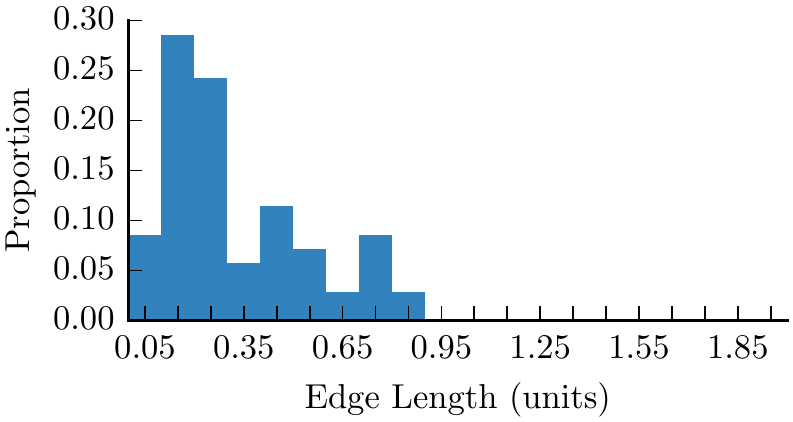}
        \caption{}
        \label{unit-6d-14-ns}
    \end{subfigure}
    \begin{subfigure}[b]{0.66\columnwidth}
        \centering
        \includegraphics[width=\textwidth]{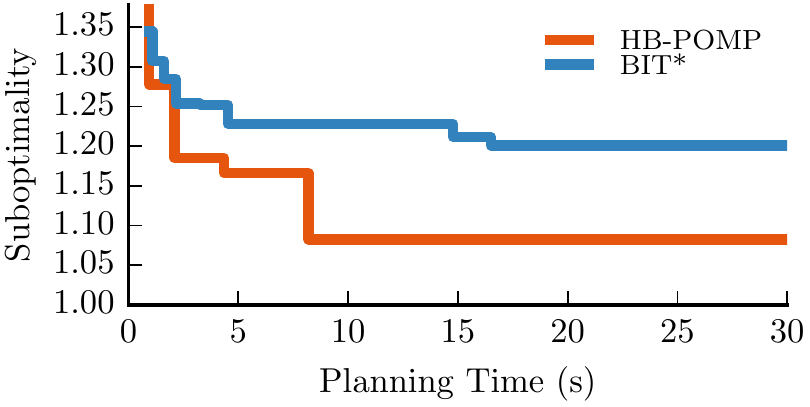}
        \caption{}
        \label{unit-6d-21-ns}
    \end{subfigure}
    \centering
    \begin{subfigure}[b]{0.66\columnwidth}
        \centering
        \includegraphics[width=\textwidth]{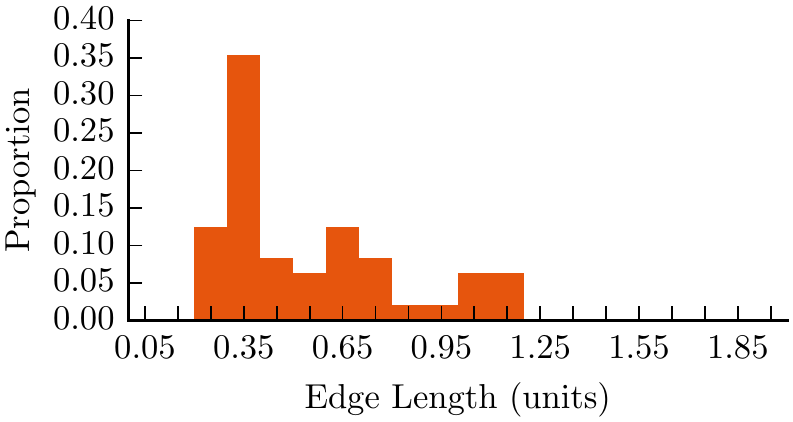}
        \caption{}
        \label{unit-6d-21-hb}
    \end{subfigure}
    \centering
    \begin{subfigure}[b]{0.66\columnwidth}
        \centering
        \includegraphics[width=\textwidth]{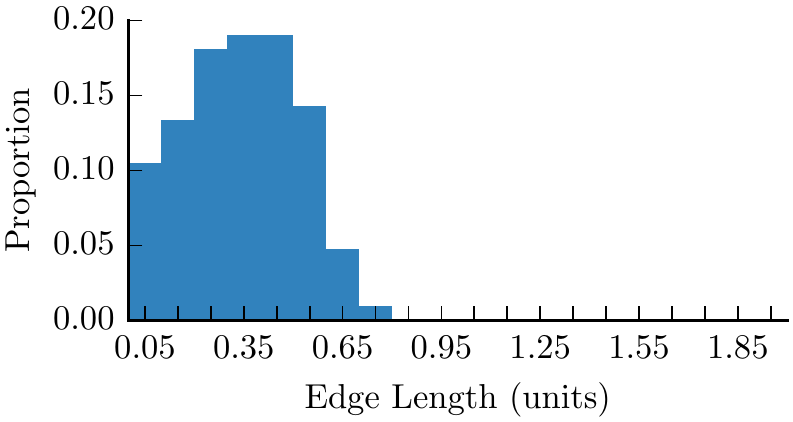}
        \caption{}
        \label{unit-6d-21-ns}
    \end{subfigure}
    \begin{subfigure}[b]{0.66\columnwidth}
        \centering
        \includegraphics[width=\textwidth]{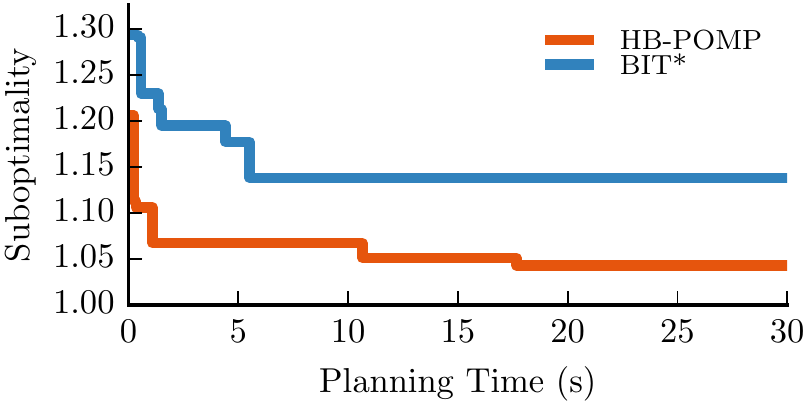}
        \caption{}
        \label{unit-6d-25-ns}
    \end{subfigure}
    \centering
    \begin{subfigure}[b]{0.66\columnwidth}
        \centering
        \includegraphics[width=\textwidth]{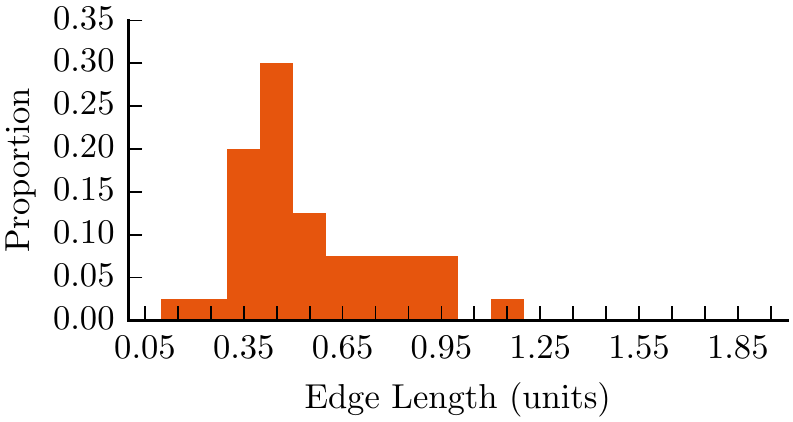}
        \caption{}
        \label{unit-6d-25-hb}
    \end{subfigure}
    \centering
    \begin{subfigure}[b]{0.66\columnwidth}
        \centering
        \includegraphics[width=\textwidth]{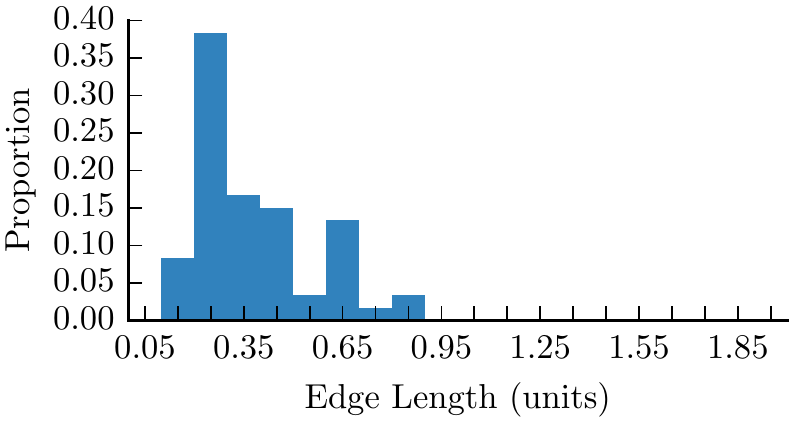}
        \caption{}
        \label{unit-6d-25-ns}
    \end{subfigure}
    \begin{subfigure}[b]{0.66\columnwidth}
        \centering
        \includegraphics[width=\textwidth]{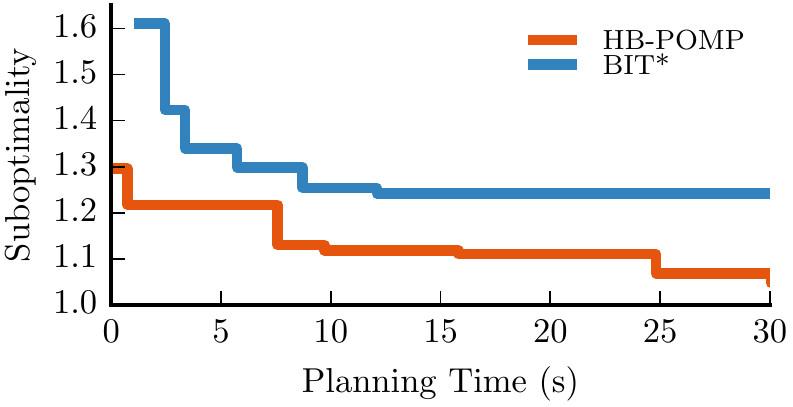}
        \caption{}
        \label{unit-6d-36-ns}
    \end{subfigure}
    \centering
    \begin{subfigure}[b]{0.66\columnwidth}
        \centering
        \includegraphics[width=\textwidth]{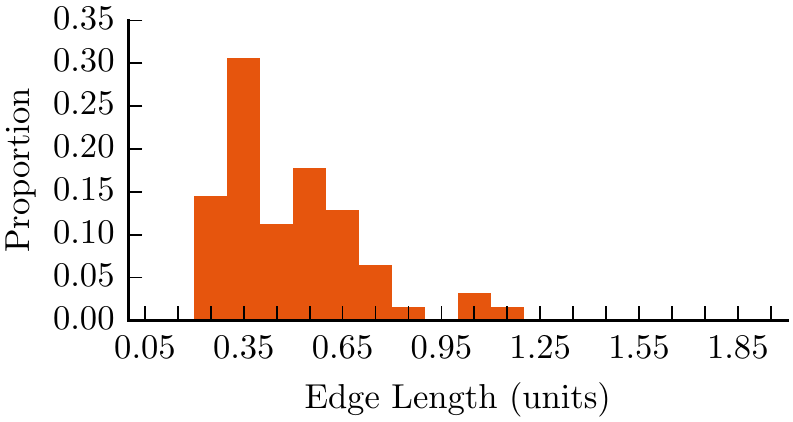}
        \caption{}
        \label{unit-6d-36-hb}
    \end{subfigure}
    \centering
    \begin{subfigure}[b]{0.66\columnwidth}
        \centering
        \includegraphics[width=\textwidth]{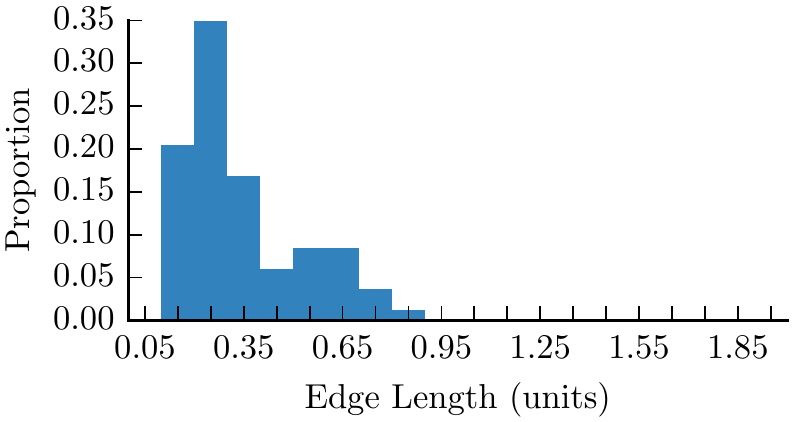}
        \caption{}
        \label{unit-6d-36-ns}
    \end{subfigure}
    \caption{
        Our combined algorithmic framework HB-POMP outperforms $\text{BIT}^{*}$ 
    in anytime planning performance on problems in $\R^6$, as shown in $4$ representative results here.  The histograms
    of edge lengths support the hypothesis that the inability of $\text{BIT}^{*}$ 
    to converge to a solution shorter than hybrid batching (even after $10$
    minutes of computation, not shown here) is partly due to not considering long edges.}
    \label{fig:unit-6d-hbpomp}
\end{figure*}

\subsubsection{Anytime Behaviour}
\hfill\\
We also evaluated the anytime performance of POMP by comparing against $\text{BIT}^*$. We ran tests for 3 different problems P4, P5 and P6 using 50 roadmaps with random offsets for POMP and 50 trials for $\text{BIT}^*$; the corresponding results are in Figure \ref{fig:anytimehyp}. Since POMP works with only the roadmaps provided, without any incremental sampling or rewiring, the path length does not improve once the shortest feasible path has been obtained. The shortest path on the roadmap can be arbitrarily sub-optimal with respect to the asymptotic optimum. $\text{BIT}^*$ adds more samples, however, and can continue to obtain improved paths with time. The results show that POMP has a comparable anytime planning performance.

\subsection{Hybrid Batching with POMP}

Our proposed algorithmic framework uses POMP as the underlying search algorithm along with the densification
strategies for the large dense roadmap.
We compared the anytime performance of the combined algorithm against $\text{BIT}^*$. 
We used hybrid batching in particular in order to have the most general framework (we call it HB-POMP); 
however, if the problem distribution is known to be easy or hard, then we suggest using
vertex batching or edge batching respectively.  An updated and more efficient implementation of both
densification and POMP was used for these experiments.

We have already described our extensive testing of the individual components of the framework, and the results support
our hypotheses about the value of those components. The purpose of this set of experiments is primarily to show that the
overall algorithm of HB-POMP works well
over a range of problems compared to our baseline of $\text{BIT}^*$, which in turn 
has been shown to outperform other anytime planning algorithms~\citep{gammell2014batch}. 
Therefore, we did not do any parameter tuning for HB-POMP - 
the hybrid batching parameters were the same as that used for the densification experiments, and the POMP parameters
were the same as the default set chosen for the earlier POMP experiments. Similarly, for $\text{BIT}^*$ we used
the default set of parameters for all experiments.

For evaluation, we tested on 50 problems in each of the $\R^6$ and $\R^8$ unit hypercubes, 
using large dense roadmaps with $N = 10^5$ vertices. To simulate a range of difficulties,
for each problem we randomly selected the number of obstacles (between 1000 and 5000)
and the $\xi_{\text{obs}}$ parameter (between $0.1$ and $0.5$). The rest of our 
experimental setup was similar to that in~\sref{sec:experiments-densification}.

\begin{figure*}[t]
    \centering
    \captionsetup[subfigure]{justification=centering}
    \begin{subfigure}[b]{0.66\columnwidth}
        \centering
        \includegraphics[width=\textwidth]{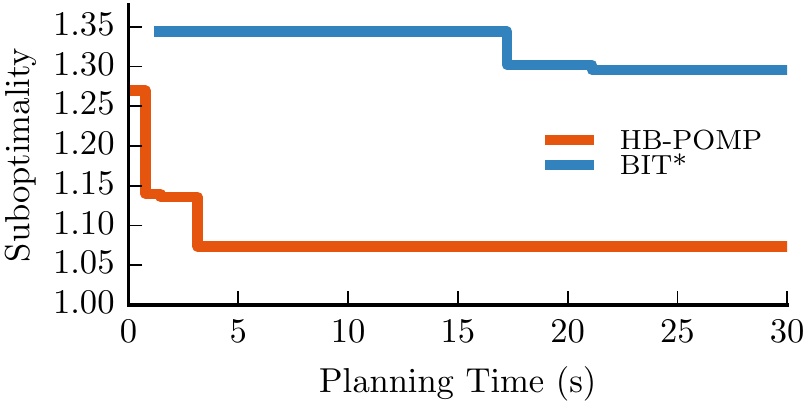}
        \caption{}
        \label{unit-8d-2-ns}
    \end{subfigure}
    \centering
    \begin{subfigure}[b]{0.66\columnwidth}
        \centering
        \includegraphics[width=\textwidth]{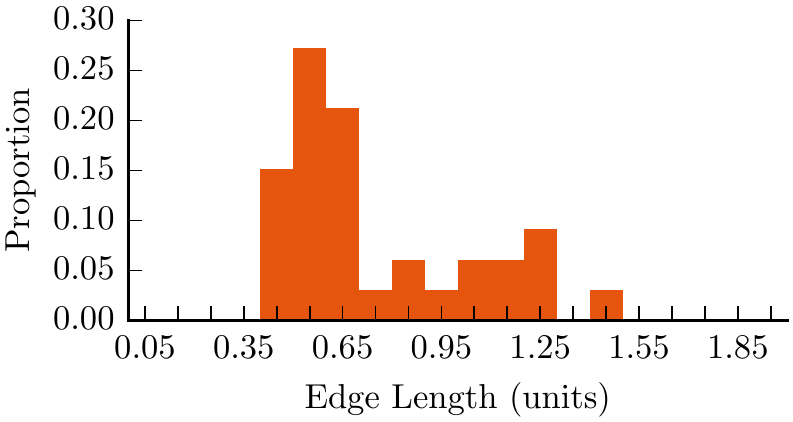}
        \caption{}
        \label{unit-8d-2-hb}
    \end{subfigure}
    \centering
    \begin{subfigure}[b]{0.66\columnwidth}
        \centering
        \includegraphics[width=\textwidth]{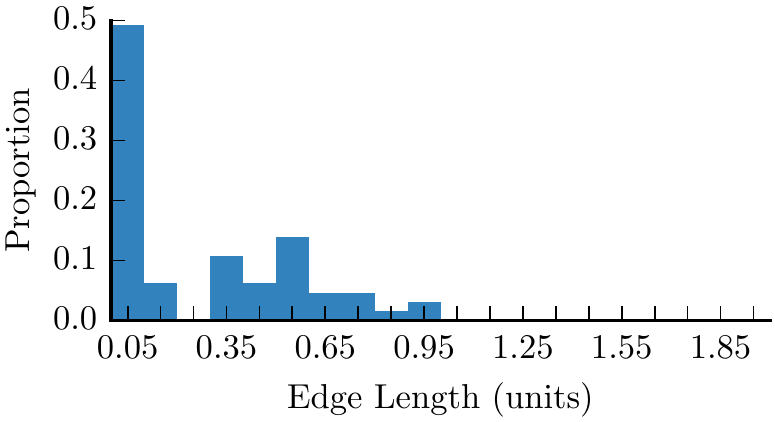}
        \caption{}
        \label{unit-8d-2-ns}
    \end{subfigure}
    \begin{subfigure}[b]{0.66\columnwidth}
        \centering
        \includegraphics[width=\textwidth]{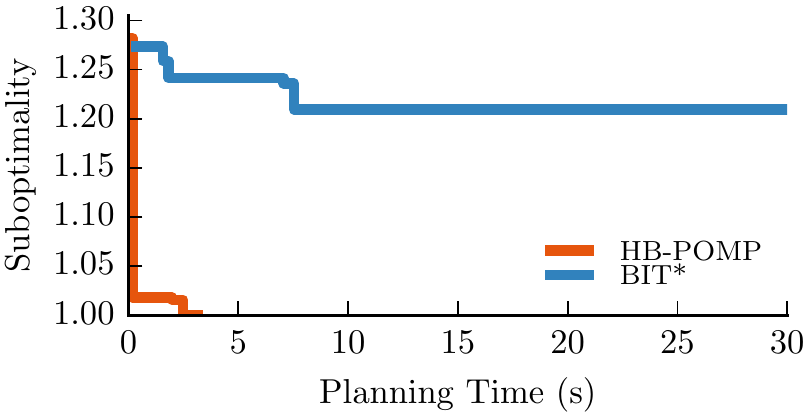}
        \caption{}
        \label{unit-8d-22-ns}
    \end{subfigure}
    \centering
    \begin{subfigure}[b]{0.66\columnwidth}
        \centering
        \includegraphics[width=\textwidth]{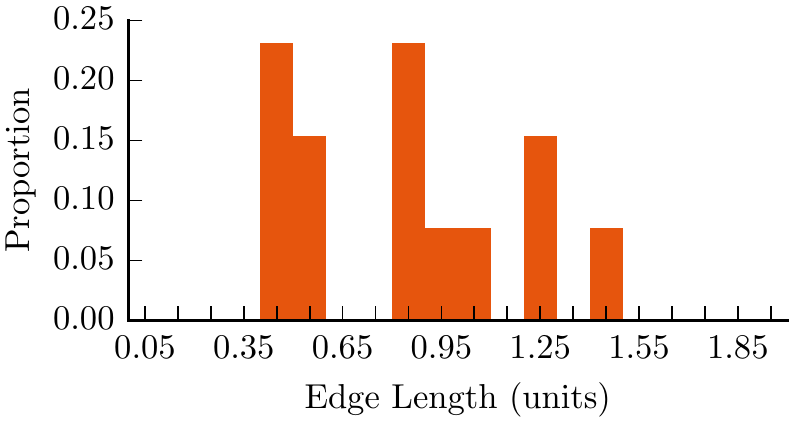}
        \caption{}
        \label{unit-8d-22-hb}
    \end{subfigure}
    \centering
    \begin{subfigure}[b]{0.66\columnwidth}
        \centering
        \includegraphics[width=\textwidth]{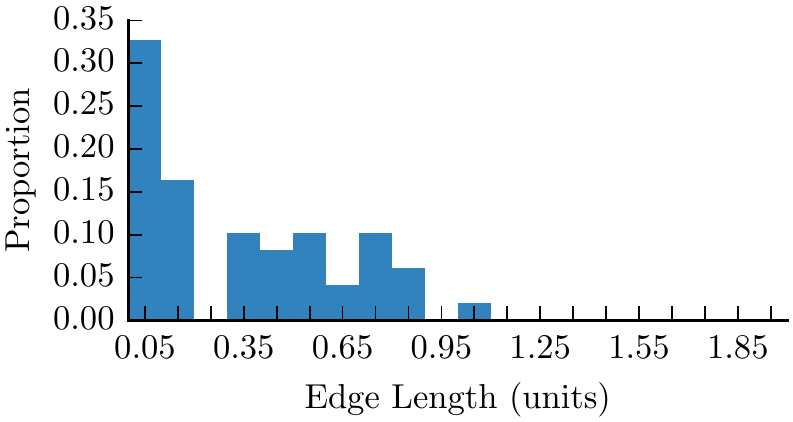}
        \caption{}
        \label{unit-8d-22-ns}
    \end{subfigure}
    \begin{subfigure}[b]{0.66\columnwidth}
        \centering
        \includegraphics[width=\textwidth]{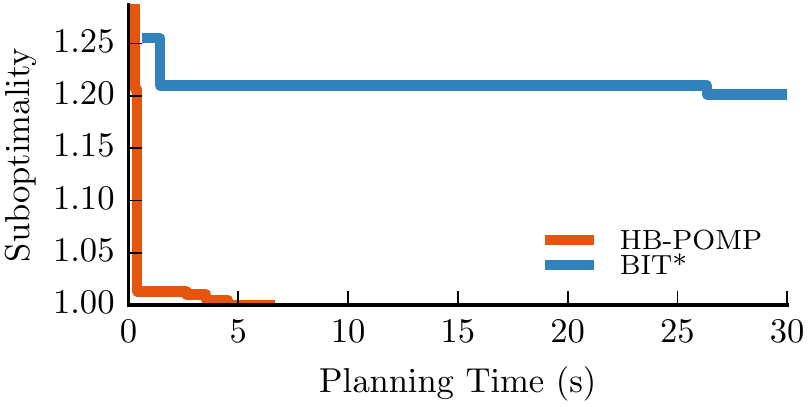}
        \caption{}
        \label{unit-8d-26-ns}
    \end{subfigure}
    \centering
    \begin{subfigure}[b]{0.66\columnwidth}
        \centering
        \includegraphics[width=\textwidth]{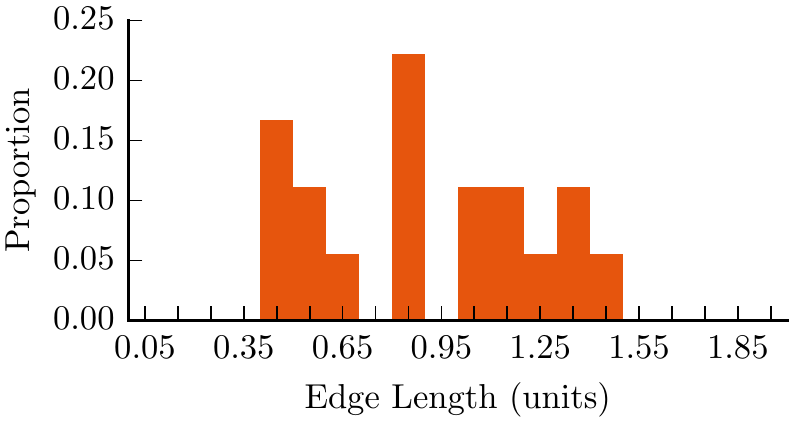}
        \caption{}
        \label{unit-8d-26-hb}
    \end{subfigure}
    \centering
    \begin{subfigure}[b]{0.66\columnwidth}
        \centering
        \includegraphics[width=\textwidth]{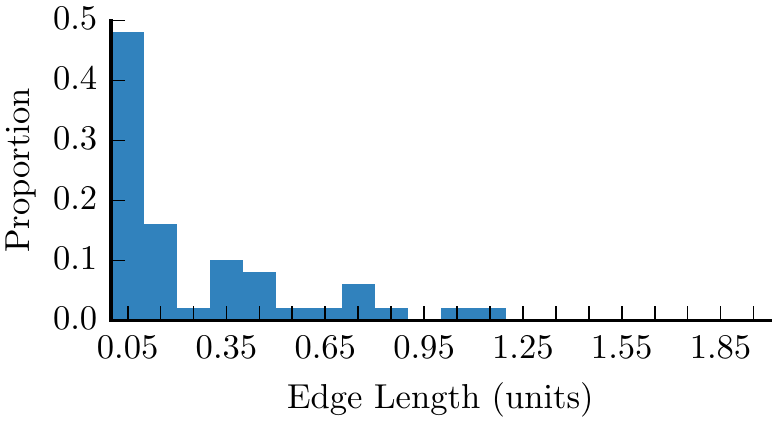}
        \caption{}
        \label{unit-8d-26-ns}
    \end{subfigure}
    \begin{subfigure}[b]{0.66\columnwidth}
        \centering
        \includegraphics[width=\textwidth]{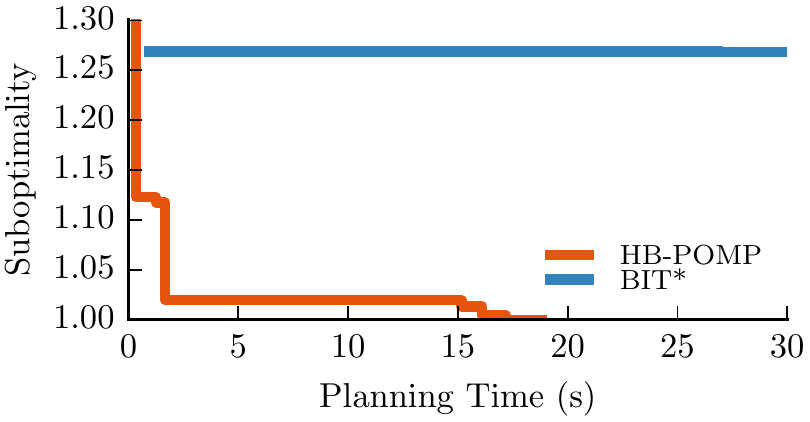}
        \caption{}
        \label{unit-8d-39-ns}
    \end{subfigure}
    \centering
    \begin{subfigure}[b]{0.66\columnwidth}
        \centering
        \includegraphics[width=\textwidth]{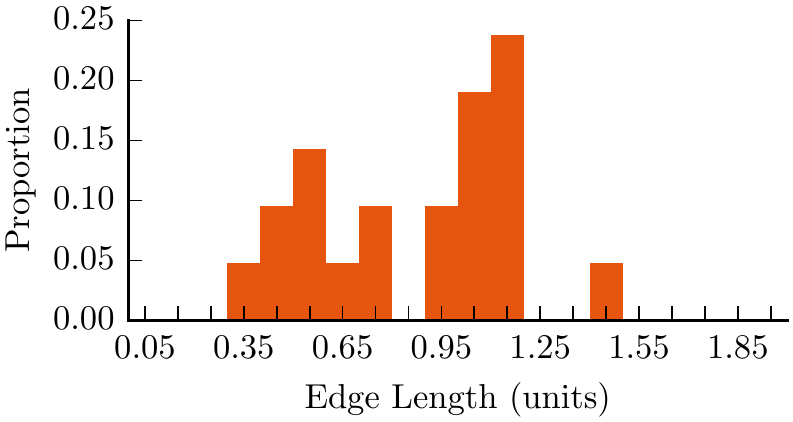}
        \caption{}
        \label{unit-8d-39-hb}
    \end{subfigure}
    \centering
    \begin{subfigure}[b]{0.66\columnwidth}
        \centering
        \includegraphics[width=\textwidth]{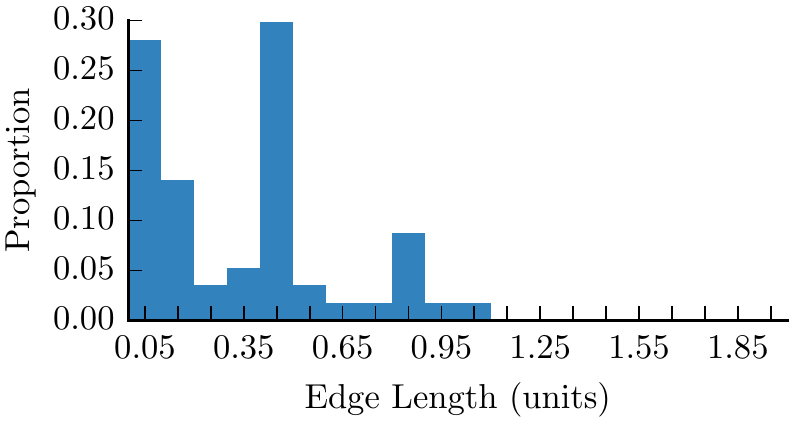}
        \caption{}
        \label{unit-8d-39-ns}
    \end{subfigure}
    \caption{For problems in $\R^8$, the improvement of HB-POMP over $\text{BIT}^{*}$
    in terms of anytime performance is even more pronounced than for $\R^6$ (\figref{fig:unit-6d-hbpomp}).
    Notably, so is the difference in edge length histograms. This supports the hypothesis that the penalty 
    for ignoring longer edges on
    the part of $\text{BIT}^{*}$ increases with increased dimensionality. }
    \label{fig:unit-8d-hbpomp}
\end{figure*}

For each problem, we ran $\text{BIT}^{*}$ for 10 minutes and HB-POMP until termination.
The sub-optimality is with respect to the best solution found by either
algorithm on that problem. In the plots themselves, we show results up to $30$ seconds
to better observe the anytime performance up to a reasonable time frame. 
For the vast majority of the $50$ problems, HB-POMP
had better anytime performance than $\text{BIT}^{*}$. Results of $4$ representative problems for each of $\R^6$ 
and $\R^8$ are shown in \figref{fig:unit-6d-hbpomp} and \figref{fig:unit-8d-hbpomp} respectively.

Despite being asymptotically
optimal, $\text{BIT}^{*}$ failed to obtain a solution shorter than the optimal solution
of HB-POMP even after 10 minutes of computation, on virtually all the problems. Our hypothesis is that since $\text{BIT}^{*}$
keeps shrinking the $r$-disk radius as more batches of samples are added, it is \emph{unable
to utilize long edges} between samples. The asymptotic optimality property states that in the limit,
a long edge can be approximated arbitrarily closely by a series of short edges, but that says nothing
about the approximation in finite time, which may be poor, especially for high-dimensional problems.
To support this hypothesis, for each representative problem we also computed histograms of edge lengths
for edges along the multiple feasible paths obtained by each algorithm on that problem.
We observe that the edge-length histograms for HB-POMP are shifted to the right as compared to those
for $\text{BIT}^{*}$, and this shift becomes more significant in $\R^8$ than in $\R^6$.
These results are in contrast to the common assumption in the sampling-based motion-planning community that one should minimize the length of edges used in roadmaps, i.e. that $r$-disk roadmaps should be used with the minimum possible radius to guarantee connectivity while still achieving asymptotic optimality. They make a strong case for the consideration of long
edges in roadmaps, which we had called out in \sref{sec:intro-motiv}.


\section{Conclusion}
\label{sec:conclusion}

In this paper, we proposed an algorithmic framework for anytime motion planning on
large dense roadmaps with expensive edge evaluations.
We argued for the benefits of using a roadmap that is large and dense, along with
efficient search techniques and lazy evaluation. We cast the problem of anytime
planning on a roadmap as searching for the shortest feasible path over a sequence of 
subgraphs of the roadmap, using some densification strategy. Our analysis obtained
effort-vs-suboptimality bounds over the sequence of subgraphs, in the case where
the roadmap samples are generated using a low-dispersion quasi-random Halton sequence.
We also proposed an algorithm for anytime planning on (reasonably-sized) roadmaps 
called Pareto-Optimal Motion Planner (POMP). It maintains an implicit belief over the configuration 
space and searches for paths which are Pareto-optimal in path length and collision probability.
Our experiments demonstrated the favourable characteristics of the individual ideas
and the good performance of the combined framework.
 
\subsection{Discussion}

While each of our key ideas is important for the anytime motion
planning problem, we advocate particularly for using POMP to organize the search over
the sequence of subgraphs of the roadmap.
POMP benefits from model re-use as more edges are evaluated in each batch,
it is efficient in domains with expensive edge evaluations,
and it creates two-level anytime planning behaviour in the overall
framework, which leads to a better effort-to-quality tradeoff.
When the current batch does not have an improved feasible solution, the
fail-fast nature of POMP makes detecting this very quick.

More generally, we believe that searching over large dense
roadmaps via a densification strategy and an underlying search algorithm that is efficient
with collision checks, achieves the favourable properties required
for anytime planning in domains with expensive edge evaluations.
We obtain an initial
feasible solution quickly with few collision checks, which is pivotal to restricting future work. We
can devise appropriate termination criteria from the explicit effort-suboptimality
tradeoff. By considering long edges we can get good quality solutions in finite time
even for high-dimensional problems.

\subsection{Future Work}

There are a number of interesting questions for future research. A natural extension to the densification
analysis is to provide similar analysis for a sequence of random i.i.d. samples.
When out of the starvation regions we would like to bound the quality obtained similar to the bounds 
provided by~\eref{eq:dispersion_suboptimality}.
A starting point would be to leverage recent results~\citep{DMB15} for Random Geometric Graphs
under expectation, albeit for a \emph{specific} radius $r$.

Another question related to densification is alternative possibilities to traverse the subgraph space of $\calG$.
As depicted in~\figref{fig:ve_batching}, our densification strategies are essentially
ways to traverse this space. 
We discussed three techniques that traverse relevant boundaries of the space. There are, however, innumerable
trajectories that a strategy can follow to reach the complete roadmap at the top right. Our current batching methods
could be compared, both theoretically and practically, to those that
go through the interior of the space.

For our implementation of POMP, the C-space belief model uses a simple but effective 
$k$-nearest neighbour lookup~\citep{pan2013faster}. More sophisticated models, using
Gaussian mixture models~\citep{huh2016learning} or reasoning about the topology of
configuration space from collision checks~\citep{pokorny2016topological}, could
also be applied. The tradeoff between model representation power, efficiency and 
utility as a search heuristic is a relevant question.

The belief model updates and queries for any reasonably sized roadmap are faster than the average
collision check. However, as more and more batches are added, more collision checks
are done and the model becomes more and more informed. Consequently, further updates become more expensive 
even as they potentially become less useful. There are some
interesting information theoretic questions about the utility of model updates as the number
of collision checks grows.

In the roadmap densification regime, at the end of each batch, the samples to be added
for the next batch are already decided beforehand based on the generating sequence. However, 
as the belief model is continuously being updated, it induces a prior
over the samples yet to be added. Whether this prior can be useful while adding
the next batch is another interesting question to explore.


\footnotesize{
\bibliographystyle{abbrv}
\bibliography{main}

\begin{thebibliography}{10}

\bibitem{arslan2015dynamic}
O.~Arslan and P.~Tsiotras.
\newblock Dynamic programming guided exploration for sampling-based motion
  planning algorithms.
\newblock In {\em Robotics and Automation (ICRA), 2015 IEEE International
  Conference on}, pages 4819--4826. IEEE, 2015.

\bibitem{bohlin2000path}
R.~Bohlin and L.~E. Kavraki.
\newblock Path planning using lazy {PRM}.
\newblock In {\em Robotics and Automation (ICRA), 2000 IEEE International
  Conference on}, volume~1, pages 521--528. IEEE, 2000.

\bibitem{BLOY01}
M.~S. Branicky, S.~M. LaValle, K.~Olson, and L.~Yang.
\newblock Quasi-randomized path planning.
\newblock In {\em {IEEE} International Conference on Robotics and Automation},
  pages 1481--1487, 2001.

\bibitem{brin1995near}
S.~Brin.
\newblock Near neighbor search in large metric spaces.
\newblock In {\em Proceedings of the 21th International Conference on Very
  Large Data Bases}, pages 574--584. Morgan Kaufmann Publishers Inc., 1995.

\bibitem{burns2003information}
B.~Burns and O.~Brock.
\newblock Information theoretic construction of probabilistic roadmaps.
\newblock In {\em Intelligent Robots and Systems, 2003.(IROS 2003).
  Proceedings. 2003 IEEE/RSJ International Conference on}, volume~1, pages
  650--655. IEEE, 2003.

\bibitem{burns2005sampling}
B.~Burns and O.~Brock.
\newblock Sampling-based motion planning using predictive models.
\newblock In {\em Robotics and Automation (ICRA), 2005 IEEE International
  Conference on}, pages 3120--3125. IEEE, 2005.

\bibitem{choudhury2016pareto}
S.~Choudhury, C.~M. Dellin, and S.~S. Srinivasa.
\newblock Pareto-optimal search over configuration space beliefs for anytime
  motion planning.
\newblock In {\em Intelligent Robots and Systems (IROS), 2016 IEEE/RSJ
  International Conference on}, pages 3742--3749. IEEE, 2016.

\bibitem{choudhury2016regionally}
S.~Choudhury, J.~D. Gammell, T.~D. Barfoot, S.~S. Srinivasa, and S.~Scherer.
\newblock Regionally accelerated batch informed trees (rabit*): A framework to
  integrate local information into optimal path planning.
\newblock In {\em Robotics and Automation (ICRA), 2016 IEEE International
  Conference on}, pages 4207--4214. IEEE, 2016.

\bibitem{choudhury2017densification}
S.~Choudhury, O.~Salzman, S.~Choudhury, and S.~S. Srinivasa.
\newblock Densification strategies for anytime motion planning over large dense
  roadmaps.
\newblock In {\em Robotics and Automation (ICRA), 2017 IEEE International
  Conference on}, pages 3770--3777. IEEE, 2017.

\bibitem{CPL14}
B.~J. Cohen, M.~Phillips, and M.~Likhachev.
\newblock Planning single-arm manipulations with n-arm robots.
\newblock In {\em RSS}, 2014.

\bibitem{deheuvels1983strong}
P.~Deheuvels.
\newblock Strong bounds for multidimensional spacings.
\newblock {\em Probability Theory and Related Fields}, 64(4):411--424, 1983.

\bibitem{dellin2016unifying}
C.~M. Dellin and S.~S. Srinivasa.
\newblock A unifying formalism for shortest path problems with expensive edge
  evaluations via lazy best-first search over paths with edge selectors.
\newblock In {\em International Conference on Automated Planning and
  Scheduling}, pages 459--467, 2016.

\bibitem{dijkstra1959note}
E.~W. Dijkstra.
\newblock A note on two problems in connexion with graphs.
\newblock {\em Numerische mathematik}, 1(1):269--271, 1959.

\bibitem{dobson2014sparse}
A.~Dobson and K.~E. Bekris.
\newblock Sparse roadmap spanners for asymptotically near-optimal motion
  planning.
\newblock {\em The International Journal of Robotics Research}, 33(1):18--47,
  2014.

\bibitem{DMB15}
A.~Dobson, G.~V. Moustakides, and K.~E. Bekris.
\newblock Geometric probability results for bounding path quality in
  sampling-based roadmaps after finite computation.
\newblock In {\em IEEE International Conference on Robotics and Automation},
  2015.

\bibitem{gammell2014batch}
J.~D. Gammell, S.~S. Srinivasa, and T.~D. Barfoot.
\newblock Batch informed trees (bit*): Sampling-based optimal planning via the
  heuristically guided search of implicit random geometric graphs.
\newblock {\em arXiv preprint arXiv:1405.5848}, 2014.

\bibitem{GSB14}
J.~D. Gammell, S.~S. Srinivasa, and T.~D. Barfoot.
\newblock Informed {RRT}*: Optimal sampling-based path planning focused via
  direct sampling of an admissible ellipsoidal heuristic.
\newblock In {\em {IEEE/RSJ} International Conference on Intelligent Robots and
  Systems}, pages 2997--3004, 2014.

\bibitem{H60}
J.~H. Halton.
\newblock On the efficiency of certain quasi-random sequences of points in
  evaluating multi-dimensional integrals.
\newblock {\em Numer. Math.}, 2(1):84--90, 1960.

\bibitem{hart1968formal}
P.~E. Hart, N.~J. Nilsson, and B.~Raphael.
\newblock A formal basis for the heuristic determination of minimum cost paths.
\newblock {\em Systems Science and Cybernetics, IEEE Transactions on},
  4(2):100--107, 1968.

\bibitem{huh2016learning}
J.~Huh and D.~D. Lee.
\newblock Learning high-dimensional mixture models for fast collision detection
  in rapidly-exploring random trees.
\newblock In {\em Robotics and Automation (ICRA), 2016 IEEE International
  Conference on}, pages 63--69. IEEE, 2016.

\bibitem{JIP15}
L.~Janson, B.~Ichter, and M.~Pavone.
\newblock Deterministic sampling-based motion planning: Optimality, complexity,
  and performance.
\newblock {\em CoRR}, abs/1505.00023, 2015.

\bibitem{janson2015fast}
L.~Janson, E.~Schmerling, A.~Clark, and M.~Pavone.
\newblock Fast marching tree: A fast marching sampling-based method for optimal
  motion planning in many dimensions.
\newblock {\em I. J. Robotics Res.}, pages 883--921, 2015.

\bibitem{karaman2010incremental}
S.~Karaman and E.~Frazzoli.
\newblock Incremental sampling-based algorithms for optimal motion planning.
\newblock {\em RSS VI}, 104, 2010.

\bibitem{KF11}
S.~Karaman and E.~Frazzoli.
\newblock Sampling-based algorithms for optimal motion planning.
\newblock {\em I. J. Robotics Res.}, 30(7):846--894, 2011.

\bibitem{KKL98}
L.~E. Kavraki, M.~N. Kolountzakis, and J.~Latombe.
\newblock Analysis of probabilistic roadmaps for path planning.
\newblock {\em {IEEE} Trans. Robotics and Automation}, 14(1):166--171, 1998.

\bibitem{kavraki1996probabilistic}
L.~E. Kavraki, P.~{\v{S}}vestka, J.-C. Latombe, and M.~H. Overmars.
\newblock Probabilistic roadmaps for path planning in high-dimensional
  configuration spaces.
\newblock {\em Robotics and Automation, IEEE Transactions on}, 12(4):566--580,
  1996.

\bibitem{kleinbort2016collision}
M.~Kleinbort, O.~Salzman, and D.~Halperin.
\newblock Collision detection or nearest-neighbor search? on the computational
  bottleneck in sampling-based motion planning.
\newblock {\em arXiv preprint arXiv:1607.04800}, 2016.

\bibitem{knepper2012real}
R.~A. Knepper and M.~T. Mason.
\newblock Real-time informed path sampling for motion planning search.
\newblock {\em The International Journal of Robotics Research}, page
  0278364912456444, 2012.

\bibitem{koenig2004lifelong}
S.~Koenig, M.~Likhachev, and D.~Furcy.
\newblock Lifelong planning {A}*.
\newblock {\em Artificial Intelligence}, 155(1):93--146, 2004.

\bibitem{K85}
R.~E. Korf.
\newblock Iterative-deepening-{A}*: An optimal admissible tree search.
\newblock In {\em Joint Conference on Artificial Intelligence}, pages
  1034--1036, 1985.

\bibitem{kuffner2000rrt}
J.~J. Kuffner and S.~M. LaValle.
\newblock {RRT}-connect: An efficient approach to single-query path planning.
\newblock In {\em Robotics and Automation (ICRA), 2000 IEEE International
  Conference on}, volume~2, pages 995--1001. IEEE, 2000.

\bibitem{LK99}
S.~M. LaValle and J.~J. Kuffner.
\newblock Randomized kinodynamic planning.
\newblock In {\em {IEEE} International Conference on Robotics and Automation},
  pages 473--479, 1999.

\bibitem{likhachev2005anytime}
M.~Likhachev, D.~Ferguson, G.~Gordon, A.~Stentz, and S.~Thrun.
\newblock Anytime dynamic {A*}: an anytime, replanning algorithm.
\newblock In {\em Proceedings of the Fifteenth International Conference on
  International Conference on Automated Planning and Scheduling}, pages
  262--271. AAAI Press, 2005.

\bibitem{likhachev2004ara}
M.~Likhachev, G.~J. Gordon, and S.~Thrun.
\newblock {ARA*}: Anytime {A*} with provable bounds on sub-optimality.
\newblock In {\em Advances in Neural Information Processing Systems}, pages
  767--774, 2004.

\bibitem{marble2013asymptotically}
J.~D. Marble and K.~E. Bekris.
\newblock Asymptotically near-optimal planning with probabilistic roadmap
  spanners.
\newblock {\em IEEE Transactions on Robotics}, 29(2):432--444, 2013.

\bibitem{missiuro2006adapting}
P.~E. Missiuro and N.~Roy.
\newblock Adapting probabilistic roadmaps to handle uncertain maps.
\newblock In {\em Robotics and Automation (ICRA), 2006 IEEE International
  Conference on}, pages 1261--1267. IEEE, 2006.

\bibitem{N92}
H.~Niederreiter.
\newblock {\em Random Number Generation and quasi-Monte Carlo Methods}.
\newblock Society for Industrial and Applied Mathematics, 1992.

\bibitem{nielsen2000two}
C.~L. Nielsen and L.~E. Kavraki.
\newblock A two level fuzzy {PRM} for manipulation planning.
\newblock In {\em Intelligent Robots and Systems, 2000.(IROS 2000).
  Proceedings. 2000 IEEE/RSJ International Conference on}, volume~3, pages
  1716--1721. IEEE, 2000.

\bibitem{pan2013faster}
J.~Pan, S.~Chitta, and D.~Manocha.
\newblock Faster sample-based motion planning using instance-based learning.
\newblock In {\em Algorithmic Foundations of Robotics X}, pages 381--396.
  Springer, 2013.

\bibitem{pokorny2016topological}
F.~T. Pokorny, M.~Hawasly, and S.~Ramamoorthy.
\newblock Topological trajectory classification with filtrations of simplicial
  complexes and persistent homology.
\newblock {\em The International Journal of Robotics Research},
  35(1-3):204--223, 2016.

\bibitem{ratliff2009chomp}
N.~Ratliff, M.~Zucker, J.~A. Bagnell, and S.~Srinivasa.
\newblock {CHOMP}: Gradient optimization techniques for efficient motion
  planning.
\newblock In {\em Robotics and Automation (ICRA), 2009 IEEE International
  Conference on}, pages 489--494. IEEE, 2009.

\bibitem{rickert2008balancing}
M.~Rickert, O.~Brock, and A.~Knoll.
\newblock Balancing exploration and exploitation in motion planning.
\newblock In {\em Robotics and Automation (ICRA), 2015 IEEE International
  Conference on}, pages 2812--2817. IEEE, 2008.

\bibitem{SH16}
O.~Salzman and D.~Halperin.
\newblock Asymptotically near-optimal {RRT} for fast, high-quality motion
  planning.
\newblock {\em {IEEE} Trans. Robotics}, 32(3):473--483, 2016.

\bibitem{SSH16b}
O.~Salzman, K.~Solovey, and D.~Halperin.
\newblock Motion planning for multi-link robots by implicit configuration-space
  tiling.
\newblock {\em IEEE Robotics and Automation Letters}, 2016.

\bibitem{sanchez2002delaying}
G.~S{\'a}nchez and J.-C. Latombe.
\newblock On delaying collision checking in {PRM} planning: Application to
  multi-robot coordination.
\newblock {\em The International Journal of Robotics Research}, 21(1):5--26,
  2002.

\bibitem{singh2003high}
A.~Singh, H.~Ferhatosmanoglu, and A.~{\c{S}}. Tosun.
\newblock High dimensional reverse nearest neighbor queries.
\newblock In {\em Proceedings of the twelfth international conference on
  Information and knowledge management}, pages 91--98. ACM, 2003.

\bibitem{SSH16}
K.~Solovey, O.~Salzman, and D.~Halperin.
\newblock Finding a needle in an exponential haystack: Discrete {RRT} for
  exploration of implicit roadmaps in multi-robot motion planning.
\newblock {\em I. J. Robotics Res.}, 35(5):501--513, 2016.

\bibitem{SSH16c}
K.~Solovey, O.~Salzman, and D.~Halperin.
\newblock New perspective on sampling-based motion planning via random
  geometric graphs.
\newblock In {\em RSS}, 2016.

\bibitem{srinivasa2010herb}
S.~S. Srinivasa, D.~Ferguson, C.~J. Helfrich, D.~Berenson, A.~Collet,
  R.~Diankov, G.~Gallagher, G.~Hollinger, J.~J. Kuffner, and M.~V. Weghe.
\newblock {HERB}: a home exploring robotic butler.
\newblock {\em Autonomous Robots}, 28(1):5--20, 2010.

\bibitem{srinivasa2016system}
S.~S. Srinivasa, A.~M. Johnson, G.~Lee, M.~C. Koval, S.~Choudhury, J.~E. King,
  C.~M. Dellin, M.~Harding, D.~T. Butterworth, P.~Velagapudi, and A.~Thackston.
\newblock A system for multi-step mobile manipulation: Architecture,
  algorithms, and experiments.
\newblock In {\em International Symposium on Experimental Robotics}, pages
  254--265. Springer, 2016.

\bibitem{SMK12}
I.~A. Sucan, M.~Moll, and L.~E. Kavraki.
\newblock The {O}pen {M}otion {P}lanning {L}ibrary.
\newblock {\em {IEEE} Robotics \& Automation Magazine}, 2012.

\bibitem{van2006anytime}
J.~Van Den~Berg, D.~Ferguson, and J.~Kuffner.
\newblock Anytime path planning and replanning in dynamic environments.
\newblock In {\em Robotics and Automation (ICRA), 2006 IEEE International
  Conference on}, pages 2366--2371. IEEE, 2006.

\bibitem{BSHG11}
J.~van~den Berg, R.~Shah, A.~Huang, and K.~Y. Goldberg.
\newblock Anytime nonparametric {A}.
\newblock In {\em Association for the Advancement of Artificial Intelligence},
  pages 105--111, 2011.

\bibitem{WR12}
C.~M. Wilt and W.~Ruml.
\newblock When does weighted {A}* fail?
\newblock In {\em Symposium on Combinatorial Search}, pages 137--144, 2012.

\end{thebibliography}
}

\end{document}